\documentclass{article}
\usepackage[final]{neurips_2021}
\usepackage[utf8]{inputenc} %
\usepackage[T1]{fontenc}    %
\usepackage{url}            %
\usepackage{booktabs}       %
\usepackage{nicefrac}       %
\usepackage{microtype}      %
\usepackage{algorithm,algorithmic}%
\usepackage{amsfonts,amsmath,amssymb,amsthm}
\usepackage[dvipsnames]{xcolor}
\usepackage{mathtools}
\usepackage{graphicx}
\usepackage{hyperref}
\usepackage{csquotes}
\usepackage{cleveref}
\hypersetup{
 colorlinks=True,
 linkcolor=blue,
 citecolor=blue,
 urlcolor=blue}
\usepackage{float}
\usepackage{lipsum} %
\usepackage{bm}
\usepackage{makecell}
\usepackage{enumitem}
\usepackage{caption}
\usepackage{adjustbox}
\usepackage{wrapfig, blindtext}

\usepackage{color-edits}
\usepackage{xspace}
\usepackage{subfigure}

\usepackage{mkolar_definitions,color,comment}

\bibpunct{[}{]}{;}{a}{,}{,}

\newcount\Comments  %
\definecolor{darkgreen}{rgb}{0,0.5,0}
\definecolor{darkred}{rgb}{0.7,0,0}
\definecolor{teal}{rgb}{0.3,0.8,0.8}

\newcommand{\IL}{\mathrm{IL}}
\newcommand{\RL}{\mathrm{RL}}

\newcommand{\pie}{\pi_e}

\setcitestyle{numbers}

\newcommand{\alg}{\texttt{MILO}}
\newcommand{\hopper}{\texttt{Hopper-v2}}
\newcommand{\walker}{\texttt{Walker2d-v2}}
\newcommand{\halfcheetah}{\texttt{HalfCheetah-v2}}
\newcommand{\ant}{\texttt{Ant-v2}}
\newcommand{\humanoid}{\texttt{Humanoid-v2}}

\title{Mitigating Covariate Shift in Imitation Learning \\ via Offline  Data Without Great Coverage}
\author{%
  Jonathan D. Chang\thanks{Equal contribution} \\
  Department of Computer Science\\
  Cornell University\\
  \texttt{jdc396@cornell.edu} \\
  \And
  $\text{Masatoshi Uehara}^*$ \\
  Department of Computer Science\\
  Cornell University\\
  \texttt{mu223@cornell.edu} \\
  \AND
  Dhruv Sreenivas \\
  Department of Computer Science\\
  Cornell University\\
  \texttt{ds844@cornell.edu} \\
  \And
  Rahul Kidambi\thanks{Work done outside Amazon} \\
  Amazon Search \& AI \\
  \texttt{rk773@cornell.edu} \\
  \And
  Wen Sun \\
  Department of Computer Science\\
  Cornell University\\
  \texttt{ws455@cornell.edu} \\
}

\begin{document}
\maketitle
\begin{abstract}
This paper studies offline Imitation Learning (IL) where an agent learns to imitate an expert demonstrator without additional online environment interactions. Instead, the learner is presented with a static offline dataset of state-action-next state transition triples from a potentially less proficient behavior policy. We introduce Model-based IL from Offline data (\alg): an algorithmic framework that utilizes the static dataset to solve the offline IL problem efficiently both in theory and in practice. In theory, even if the behavior policy is highly sub-optimal compared to the expert, we show that as long as the data from the behavior policy provides sufficient coverage on the expert state-action traces (and with no necessity for a global coverage over the entire state-action space), \alg~can provably combat the covariate shift issue in IL. Complementing our theory results, we also demonstrate that a practical implementation of our approach mitigates covariate shift on benchmark MuJoCo continuous control tasks. We demonstrate that with behavior policies whose performances are less than half of that of the expert, \alg~still successfully imitates with an extremely low number of expert state-action pairs while traditional offline IL methods such as behavior cloning (BC) fail completely. Source code is provided at \url{https://github.com/jdchang1/milo}.

\end{abstract}

\section{Introduction}\label{sec:intro}

\emph{Covariate shift} is a core issue in Imitation Learning (IL). Traditional IL methods like behavior cloning (BC)~\citep{pom89}, while simple, suffer from covariate shift, learning a policy that can make arbitrary mistakes in parts of the state space not covered by the expert dataset. This leads to compounding errors in the agent's performance \citep{ross2010efficient}, hurting the generalization capabilities in practice.

\begin{figure}[ht]
    \centering
    \includegraphics[width=0.96\textwidth]{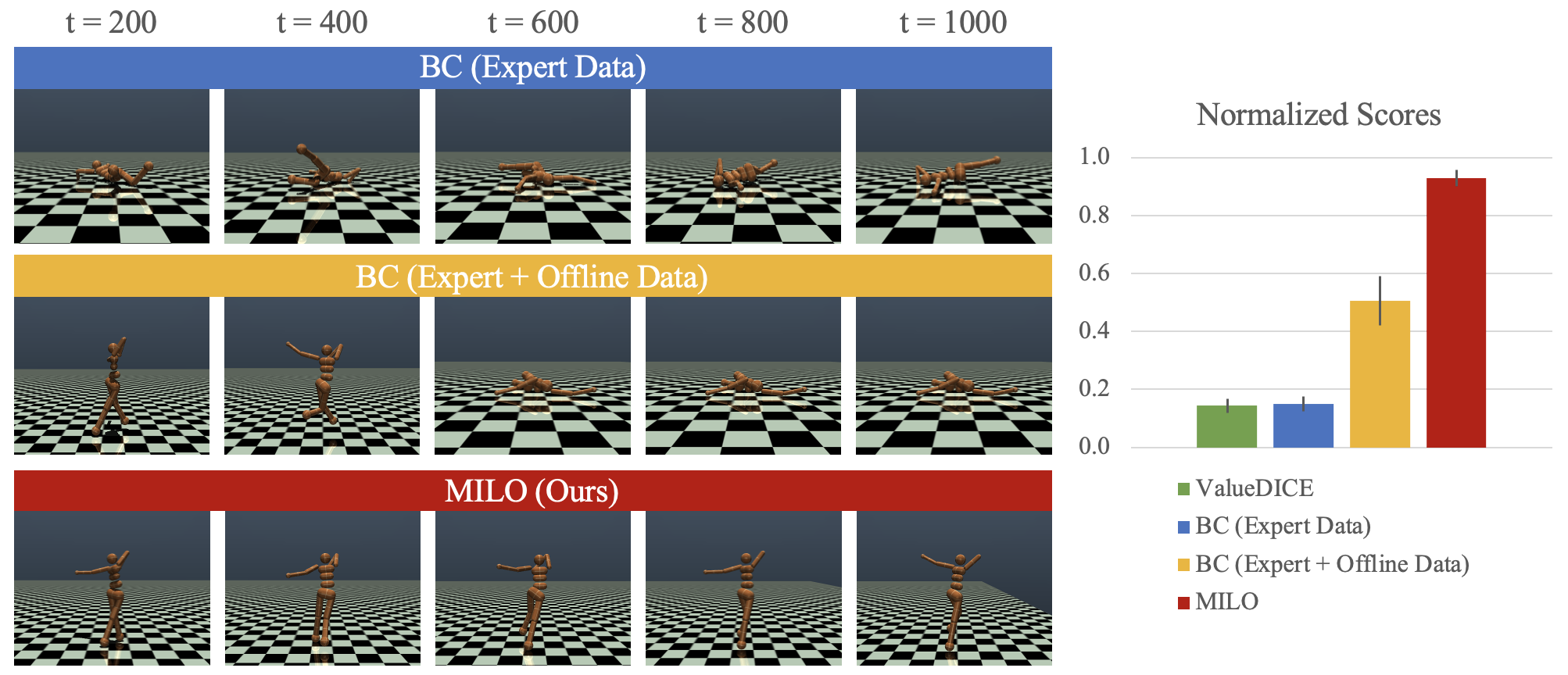}
\caption{(Left) Frames at timesteps 200, 400, 600, 800, and 1000 for \humanoid~from policies trained with BC on 100 state-action pairs from the expert (\textcolor{Blue}{blue}), BC on 1M offline samples plus 100 expert samples (\textcolor{Dandelion}{yellow}), and our algorithm \alg~(\textcolor{Mahogany}{red}). The expert has a performance of 3248 and the behavior policy used to collect the offline dataset has performance of $1505\pm 473$ ($\approx 46\%$ of the expert's). (Right) Expert performance normalized scores averaged across 5 seeds.}
    \label{fig:front_page}
    \vspace{-5mm}
\end{figure}

Prior works have presented several means to combat this phenomenon in IL.
One line of thought utilizes an {\em interactive expert}, i.e. an expert that can be queried at an arbitrary state encountered during the training procedure.
Interactive IL algorithms such as DAgger \citep{Ross2011_AISTATS}, LOLS \citep{chang2015learning_dependency}, DART \citep{lee2017dart}, and AggreVaTe(D) \citep{ross2014reinforcement,sun2017deeply} utilize a reduction to no-regret online learning and demonstrate that under certain conditions, they can successfully learn a policy that imitates the expert. These interactive IL algorithms, however, cannot provably avoid covariate shift if the expert is not \emph{recoverable}. That is, $A^{\pie}(s,a) = \Theta(H)$ where $\pie$ is the expert, $A^\pi$ is the usual (dis)advantage function,\footnote{in this work, we use cost instead of reward, thus we call $A^{\pi}$ the disadvantage function.} and $H$ is the planning horizon \citep[Chapter 15]{RajaramanNived2020TtFL,agarwal2019reinforcement}.
A second line of work that avoids covariate shift utilizes either a known transition dynamics model \citep{abbeel2004apprenticeship,ziebart2008maximum} or uses real world interactions \citep{ho2016generative,brantley2019disagreement,sun2019provably,KostrikovIlya2019ILvO,reddy2020sqil,kidambi2021optimism}. Prior works have shown that with known transition dynamics or real world interactions, agents can provably avoid covariate shift in both tabular and general MDPs \citep{agarwal2019reinforcement, RajaramanNived2020TtFL} even without a recoverable expert. While these results offer strong theoretical guarantees and empirical performance, online interactions are often costly and prohibitive for real world applications where active trial-and-error exploration in the environment could be unsafe or impossible. A third perspective towards addressing this issue is to assume that the expert visits \emph{the entire state space} \cite{SpencerJonathan2021FiIL}, 
where the expert effectively informs the learner what actions to take in every state. Unfortunately, such a full coverage expert distribution might be rare and holds only for special MDPs and expert policies (for e.g. an expert that induces ergodicity in the MDP).

In this work, we consider a new perspective towards handling the covariate shift issue in IL. In particular, we investigate a pure offline learning setting where the learner has access to neither the expert nor the environment for additional interactions. The learner, instead, has access to a small pre-collected dataset of state-action pairs sampled from the expert and a large batch offline dataset of state-action-next state transition triples sampled from a behavior policy that could be highly sub-optimal (see \pref{fig:front_page} where BC on the offline data results in a low-quality policy). Unlike prior works that require online interactions, our proposed method, \alg~performs high fidelity imitation in an offline, data-driven manner. 
Moreover, different from interactive IL, we do not require the expert to be present during learning, significantly relieving the expert's burden. Finally, in contrast to the prior work \citep{SpencerJonathan2021FiIL} that assumes the expert distribution covers the entire state-action space (i.e., $\max_{\pi}\max_{s,a} d^{\pi}(s,a) / d^{\pi^e}(s,a) < \infty$ where $d^{\pi}$ denotes the state-action distribution of policy $\pi$), we require the offline dataset to provide \emph{partial coverage}, i.e., it only needs to cover the expert's state-actions (i.e., $\max_{s,a} d^{\pi^e}(s,a) / \rho(s,a) < \infty$ where $\rho$ is the offline distribution of some behavior policy).\footnote{In our analysis, we refine the density ratio $d^{\pi^e}(s,a) / \rho(s,a)$ via the concept of relative conditional number which allows us to extend it to large MDPs where the ratio is infinite but the relative condition number is finite.}  


In summary, we list our main contributions below:
\begin{enumerate}[leftmargin=0.6cm]
    \item We propose Model based Imitation Learning from Offline data, \alg: a model-based framework that leverages offline batch data with only partial coverage (see \pref{sec:discrete} for definition) to overcome covariate shift in IL. 
    \item Our analysis is modular and covers common models such as discrete MDPs and linear models. Notably, our new result on non-parametric models (e.g. Gaussian Processes) with relative condition number is new even considering all existing results in offline RL (see Remark \ref{rem:implication_1},\ref{rem:offline_linear},\ref{rem:offline_nonlinear}). 
    \item The practical instantiation of our general framework leverages neural network model ensembles, and demonstrates its efficacy on benchmark MuJoCo continuous control problems. Specifically, even under low-quality behavior policies, our approach can successfully imitate using an extremely small number of expert 
    samples while algorithms like BC completely fail (\pref{fig:front_page}).
\end{enumerate}

\subsection{Related work}
\textbf{Imitation Learning} As summarized above, avoiding covariate shift in IL is an important topic. Another relevant line of research is IL algorithms that use \emph{offline} or \emph{off-policy} learning. ValueDICE \cite{KostrikovIlya2019ILvO} presents a principled way to leverage off-policy data for IL. In theory, the techniques from ValueDICE (and more broadly, DICE \cite{ChowYinlam2019DBEo,zhang2019gendice}) require the data provided to the agent to have global coverage. Moreover in practice, ValueDICE uses online interaction and maintains an increasing replay buffer which may eventually provide global coverage. Instead, we aim to study offline IL without any online interactions and are interested in the setting where offline data does not have global coverage. Another line of work \cite{jarrett2020sbil,chan2021sbirl} studies IL in an offline setting by only using the expert dataset. In contrast to these works, our goal is to study the use of an additional offline dataset collected from a behavior policy to mitigate covariate shift, as information theoretically any algorithm that relies solely on expert data will still suffer from covariate shift in the worst case \cite{RajaramanNived2020TtFL}. 

Similar in setting, Cascaded Supervised IRL (CSI) \cite{klein2013csi} performs imitation offline with both an expert dataset and a static offline dataset by first fitting a reward function that is then used in Least Squares Policy Iteration. CSI, however, requires the expert data to have global coverage and does not mitigate covariate shift with partial coverage like \alg~does. Finally, Variational Model-Based Adversarial IL (V-MAIL) \cite{rafailov2021vmail} learns a dynamics model from a static offline dataset and performs offline imitation within the model. V-MAIL, however, studies zero-shot IL where the static offline data is samples from a variety of source tasks, and the expert dataset is samples from the transfer target task. In contrast, \alg~investigates avoiding covariate shift in IL with an offline dataset collected by a potentially suboptimal policy from the same task.


\textbf{Offline RL} In offline RL, algorithms such as FQI
\citep{ernst2005tree} have finite-sample error guarantees under the global coverage \citep{munos2008finite,antos2008learning}. Recently, many algorithms to tackle this problem have been proposed from both model-free \citep{WuYifan2019BROR,TouatiAhmed2020SPOv,Liu2020,pmlr-v97-fujimoto19a,FakoorRasool2021CDCB,kumar2020conservative} and model-based perspectives \citep{Yu2020,Kidambi2020,MatsushimaTatsuya2020DRLv} with some pessimism ideas. The idea of pessimism features in offline RL with an eye to penalize the learner from visiting unknown regions of the state-action space~\citep{RashidinejadParia2021BORL,JinYing2020IPPE,YinMing2021NORL,BuckmanJacob2020TIoP}. We utilize pessimism within the IL context where, unlike RL, the learner does not have access to an underlying reward signal. 
Our work expands prior theoretical results  by (a) formalizing the partial coverage condition using a notion of a relative condition number, and (b) offering distribution-dependent results when working with non-parametric models including gaussian processes. 
See Appendix \ref{ape:model_error} for a detailed literature review. 



\section{Setting}
We consider an episodic finite-horizon Markov Decision Process (MDP), $\Mcal=\{\Scal,\Acal,P,H,c,d_0 \}$, where $\Scal$ is the state space, $\Acal$ is the action space, $P:\Scal \times \Acal \to \Delta(\Scal)$ is the MDP's transition, $H$ is the horizon, $d_0$ is an initial distribution, and $c: \Scal\times \Acal \to [0,1]$ is the cost function. A policy $\pi:\Scal \to \Delta(\Acal)$ maps from state to distribution over actions. 
We denote $d^{\pi}_P\in \Delta(\Scal\times \Acal)$ as the average state-action distribution of $\pi$ under transition kernel $P$, that is, $d^{\pi}_{P}=1/H\sum_{t=1}^H d^{\pi}_{P,t}$, where $ d^{\pi}_{P,t}\in \Delta(\Scal \times \Acal)$ is the distribution of $(s^{(t)},a^{(t)})$ under $\pi$ at $t$. Given a cost function $f:\Scal\times\Acal\mapsto [0,1]$, $V^{\pi}_{P,f}$ denotes the expected cumulative cost of $\pi$ under the transition kernel $P$ and cost function $f$. Following a standard IL setting, the ground truth cost function $c$ is unknown. Instead, we have the demonstrations by the expert specified by $\pie:\Scal \to \Delta(\Acal)$ (potentially stochastic and not necessarily optimal). Concretely, we have an expert dataset in the form of i.i.d tuples $\Dcal_e=\{s_i,a_i\}_{i=1}^{n_e}$ sampled from distribution $d^{\pie}_{P}$. 

In our setting, we also have an offline static dataset consisting of i.i.d tuples $\Dcal_o=\{s_i,a_i,s'_i\}_{i=1}^{n_o}$ s.t.  $(s,a) \sim \rho(s,a),s'\sim P(s,a)$, where \emph{$\rho \in \Delta(\Scal\times\Acal)$ is an offline distribution} resulting from some behavior policies. Note behavior policy could be a much worse policy than the expert $\pi^e$. 
Our goal is to only leverage $(\Dcal_e + \Dcal_o)$ to learn a policy $\pi$ that performs as well as $\pie$ with regard to optimizing the ground truth cost $c$. More specifically, our goal is to utilize the offline static data $\Dcal_o$ to combat covariate shift and learn a policy that can significantly outperform traditional offline IL methods such as Behavior cloning (BC), \emph{without any interaction with the real world or expert}. 

\textbf{Function classes} 
We introduce function approximation. Since we do not know the true cost function $c$ and transition kernel $P$, we introduce a cost function class $\Fcal\subset \Scal\times\Acal \to [0,1]$ and a transition model class $\Pcal:\Scal\times \Acal\to \Delta(\Scal)$. We also need a policy class $\Pi$. For the analysis, we assume realizability:
\begin{assum}\label{assum:realizability}
$c\in \Fcal,P\in \Pcal,\pie \in \Pi$. 
\end{assum}

We use Integral Probability Metric (IPM) as a distribution distance measure, i.e., given two distributions $\rho_1$ and $\rho_2$, IPM with  $\Fcal$ is defined as $\max_{f\in\Fcal} \left[ \mathbb{E}_{(s,a)\sim \rho_1}[f(s,a)] - \mathbb{E}_{(s,a)\sim \rho_2}[f(s,a)] \right]$. 

\section{Algorithm}
\begin{algorithm}[t]
\caption{Framework for model-based Imitation Learning with offline data (\alg)}
\begin{algorithmic}[1]
    \STATE \textbf{Require}: IPM class $\Fcal$, model class $\Pcal$, policy class $\Pi$, datasets $\Dcal_e = \{s, a\}$, $\Dcal_{o} := \{s,a,s'\}$
    \STATE \textcolor{blue}{Train Dynamics Model and Bonus:} $\widehat{P}:\Scal\times\Acal\rightarrow\Scal$ and $b:\Scal\times\Acal\rightarrow\RR^{+}$ on offline data $\Dcal_o$
        \STATE \textcolor{blue}{Pessimistic model-based min-max IL:} with $\widehat{P}$, $b$, $\Dcal_e$, obtain $\hat\pi_{\IL}$ by solving the following:  \label{line:offline-imitation}
        \begin{align}\ts 
        \label{eq:final_obj}
        \hat\pi_{\IL}=\argmin_{\pi\in \Pi} \max_{f\in\Fcal} \left[ \mathbb{E}_{(s,a)\sim d^{\pi}_{\hat{P}} } \left[f(s,a) + b(s,a) \right]  - \EE_{(s,a)\sim \Dcal_e}[f(s,a)]    \right]
        \end{align}
\end{algorithmic}
\label{alg:milo}
\vspace{-5pt}
\end{algorithm}
The core idea of \alg~is to imitate the expert by optimizing an IPM distance between the agent and the expert with a penalty term for pessimism over the policy class. \alg~ consists of three steps:
\begin{enumerate}[leftmargin=0.6cm]
    \item \textbf{Model learning:} fit a model $\hat P$ from the offline data $\Dcal_o$ to learn $P$,  
    \item \textbf{Pessimistic penalty design:} construct penalty function $b(s,a)$ such that there is a high penalty on state-action pairs that are not covered by the offline data distribution $\rho$.
    \item \textbf{Offline min-max model-based policy optimization:} optimize Eq.~\pref{eq:final_obj}
\end{enumerate} 
\pref{alg:milo} provides the details of \alg. We explain each component in detail as follows. 



\textbf{Model learning and Penalty:} Our framework assumes we can learn a calibrated model $(\hat{P}, \sigma)$ from the dataset $\Dcal_o$, in the sense that for any $s,a$, we have: $\ts \left\|  \hat{P}(\cdot | s,a) - P(\cdot | s,a)   \right\|_{1} \leq \min\{2, \sigma(s,a)\}$.
Such model training is possible in many settings including classic discrete MDPs, linear models (KNR \citep{Kakade2020}), and non-parametric models such as GP. In practice, it is also common to train a model ensemble based on the idea of bootstrapping and then use the model-disagreement to approximate $\sigma$. With such a calibrated model, the penalty will simply be $b(s,a) = O(H \sigma(s,a))$. We will formalize this model learning assumption in \pref{sec:analysis}. We give several examples below. 

\emph{For any discrete MDP}, we use the empirical distribution, i.e., $\hat P(s'|s,a)=N(s',s,a)/(N(s,a)+\lambda)$, where $N(s,a)$ is the number of $(s,a)$ in $\Dcal_o$, and $N(s',s,a)$ is the number of $(s,a,s')$ in $\Dcal_o$, and $\lambda\in\mathbb{R}^+$.  In this case, we can set $\sigma(s,a) = \widetilde{O}\left( \sqrt{|\Scal| / N(s,a)} \right)$. See \cref{exp:discrete_model} for more details. 

\emph{For continuous Kernelized Nolinear Regulator} (KNR \citep{Kakade2020}) model where the ground truth transition $P(s'|s,a)$ is defined as $s' = W^\star\phi(s,a) + \epsilon$, $\epsilon\sim \mathcal{N}(0,\Sigma)$, with $\phi$ being a (nonlinear) feature mapping, we can learn $\widehat{P}$ by classic Ridge regression on offline dataset $\Dcal_o$. 
Here we can set $\ts \sigma(s,a) = \widetilde{O}\left( \beta \sqrt{ \phi(s,a)^{\top} \Sigma^{-1}_{n_o} \phi(s,a) } \right)$ for some $\beta\in\mathbb{R}^+$, where $\Sigma_o$ is the data covariance matrix $\ts \Sigma_{n_o} := \sum_{i=1}^{n_o} \phi(s_i,a_i)\phi(s_i,a_i)^{\top} + \lambda \Ib$. See \cref{exp:knr_model} for more details. 

\emph{For non-parametric nonlinear model such as Gaussian Process (GP)}, under the assumption that $P$ is in the form of $s' = g^\star(s,a) + \epsilon$, $\epsilon\sim \mathcal{N}(0,\Sigma)$ (here $\Scal\subset \mathbb{R}^{d_{\Scal}}$), we can simply represent $\hat{P}$ using GP posteriors induced by $\Dcal_o$, i.e., letting GP posterior be $GP(\hat{g}, k_{n_o})$, we have $\hat{P}(s' | s,a)$ being represented as $s'= \hat{g}(s,a) + \epsilon$. Then, we can set $\sigma(s,a) = \widetilde{O}\left( \beta k_{n_o}\left((s,a),(s,a)\right)\right)$ with some parameter $\beta\in\mathbb{R}^+$ (see \cref{exp:gp_model} for more details). GP is a powerful model and has been being widely used in robotics problems, see \citep{ko2007gaussian,deisenroth2011pilco,bansal2017goal,umlauft2018uncertainty,fisac2018general} for examples. 

In practice, we can also use a model ensemble of neural networks with the maximum disagreement between models as $\sigma$. This has been widely used in practice (e.g., \cite{Osband2018,AzizzadenesheliKamyar2018EETB,Pathak2019}). We leave the details to \pref{sec:prac_alg} where we instantiate a practical version of \alg, and the experiment section.
As we can see from the examples mentioned above, in general, the penalty $b(s,a)=O( H\sigma(s,a))$ is designed such that it has a high value in state-action space that is not covered well by the offline data $\Dcal_o$, and has a low value in space that is covered by $\Dcal_o$. 
Adding such a penalty automatically forces our policy to stay away from these regions where $\hat{P}$ is not accurate. On the other hand, for regions where $\rho$ has good coverage (thus $\hat{P}$ is accurate), we force $\pi$ to stay close to $\pi^e$.
\textbf{Pessimistic model-based min-max IL:} 
Note Eq.~\ref{eq:final_obj} is purely computational, i.e., we do not need any real world samples. 
To solve such min-max objective, we can iteratively (1) perform the best response on the max player, i.e., compute the $\argmax$ discriminator $f$ given the current $\pi$, and (2) perform incremental update on the min player, e.g., use policy gradient (PG) methods (e.g. TRPO) inside the learned model $\hat{P}$ with cost function $f(s,a) + b(s,a)$.  We again leave the details to \pref{sec:prac_alg}. 

\subsection{Specialization to offline RL}
In RL, the cost function $c$ is given. The goal is to obtain $\pi^{*}=\argmax_{\pi \in \Pi}V^{\pi}_{P,c}$. The pessimistic policy optimization procedure~\citep{YuTianhe2021CCOM,JinYing2020IPPE} is $\ts
\hat \pi_{\RL}=\argmin_{\pi\in \Pi}\EE_{(s,a)\sim d^{\pi}_{\hat P}}[c(s,a)+b(s,a)]$. While this is not our main contribution, we will show a byproduct of our result is a novel non-parametric analysis for offline RL which does not assume  $\rho$ has global coverage (see Remarks \ref{rem:implication_1},\ref{rem:offline_linear},\ref{rem:offline_nonlinear}). 


\section{Analysis}
\label{sec:analysis}\label{sec:analysis}
Our algorithm depends on the model $\hat P$ estimated from the offline data. We provide a unified analysis assuming that $\hat P$ is calibrated in that its confidence interval is provided. Specifically, we assume: 
\begin{assum} \label{assum:calibration} With probability $1-\delta$, the estimate model $\hat P$ satisfies the following:
 $   \|\hat P(\cdot|s,a)-P(\cdot|s,a)\|_1\leq \min(\sigma(s,a),2)\quad \forall (s,a)\in \Scal\times \Acal. $
We set the penalty as $    b(s,a) = H\min(\sigma(s,a),2)$. 
\end{assum}
We give the following three examples. For details, refer to \cref{ape:calibrated_models}. 

\begin{example}[Discrete MDPs] \label{exp:discrete_model} Set uncertainty measure $\sigma(s,a)$= 
       $\sqrt{\frac{|\Scal|\log 2+\log (2|\Scal||\Acal|/\delta)}{2\{N(s,a)+\lambda\}} }+\frac{\lambda}{N(s,a)+\lambda}. $
\end{example}
\begin{example}[KNRs] \label{exp:knr_model}In KNRs,  the ground truth model is $s'=W^{*}\phi(s,a)+\epsilon,\,\epsilon\sim \Ncal(0,\zeta^2\Ib)$, where $s\in \RR^{d_\Scal}, a\in\RR^{d_{\Acal}}$, $\phi: \Scal\times\Acal\mapsto \mathbb{R}^{d}$ is some known state-action feature mapping. The estimator is 
    \begin{align*}\ts 
       \hat g(\cdot)=\hat W\phi(\cdot),\quad \hat W=\argmin_{W\in \mathbb{R}^{d_{\Scal}\times d_{\Acal}}}1/n_o \sum_{(s,a)\in \Dcal_o}[\|W\phi(s,a)-s'\|^2_2]+\lambda \|W\|^2_F, 
    \end{align*}
    where $\|\cdot\|_F$ is a frobenius norm. We set the uncertainty measure $\sigma(s,a)$: 
\begin{align*}\ts 
  \sigma(s,a)=(1/\zeta)\beta_{n_o} \sqrt{ \phi^{\top}(s,a)\Sigma^{-1}_{n_o}\phi(s,a)},\quad  \Sigma_{n_o}=\sum_{i=1}^{n_o}\phi(s_i,a_i)\phi^{\top}(s_i,a_i)+\lambda \Ib 
\end{align*}
with $\beta_{n_o}=\{2\lambda \|W^{*}\|^2_2+8\zeta^2[d_{\Scal}\log(5)+\log(1/\delta)+\bar \Ical_{n_o})]\}^{1/2}$, where $\bar \Ical_{n_o}=\log(\det(\Sigma_{n_o}/\lambda \Ib))$. 
\end{example}


\begin{example}[GPs] \label{exp:gp_model}
In GPs, the ground truth model is defined as $s'=g^{*}(s,a)+\epsilon,\,\epsilon\sim \Ncal(0,\zeta^2\Ib)$ where $g^\star$ belongs to an RKHS $\Hcal_k$ with a kernel $k(\cdot,\cdot)$. Denote $x:= (s,a)$, 
we have GP posterior as
    \begin{align*} 
    &\resizebox{\textwidth}{!}{$\hat g(\cdot)  \ts  =S(\Kb_{n_o}+\zeta^2 \Ib)^{-1}\bar k_{n_o}(\cdot) , \quad S=[s'_1,\cdots,s'_{n_o}]\in  \mathbb{R}^{d_{\Scal}\times n_o}, \quad \bar k_{n_o}(x) \ts =[k(x_1,x),\,\cdots,k(x_{n_o},x)]^{\top},$}\\
    &\resizebox{\textwidth}{!}{$\{\Kb_{n_o}\}_{i,j}=k(x_i,x_j)\,(1\leq i\leq n_o,1\leq j\leq n_o),  \quad k_{n_o}(x,x')  =k(x,x')-\bar k_{n_o}(x)^{\top}(\Kb_{n_o}+\zeta^2 \Ib)^{-1}\bar k_{n_o}(x'),$}  
\end{align*}
with $\sigma(\cdot)= \beta_{n_o} k_{n_o}(\cdot,\cdot) / \zeta$, 
$\beta_{n_o}=O((d_{\Scal}\log^3(d_{\Scal}n_o/\delta)\Ical_{n_o})^{1/2})$, $\Ical_{n_o}=\log(\det(\Ib+\zeta^{-2}\Kb_{n_o}))$. 
\end{example}

\paragraph{General results} We show our general error bound results. For the proof, refer to \cref{ape:main}. For analytical simplicity, we assume $|\Fcal|$ is finite (but the bound only depends on $\ln(|\Fcal|)$) \footnote{
When  $|\Fcal|$ is infinite, we can show that the resulting error bound scales w.r.t its metric entropy. 

}. 

\begin{theorem}[Bound of \alg]\label{thm:main}
Suppose assumptions \ref{assum:realizability},\ref{assum:calibration}. Then, with probability $1-2\delta$, 
\begin{align*}\ts 
  V^{\hat \pi_{\IL}}_{P,c}-   V^{\pie}_{P,c}   \leq  \mathrm{Err}_{o}+\mathrm{Err}_{e},\,\mathrm{Err}_{o}=8H^2\EE_{(s,a)\sim d^{\pi_e}_P}[\min(\sigma(s,a),1) ],\,\mathrm{Err}_{e}=2H\sqrt{\frac{\log(2|\Fcal|/\delta)}{2n_e}}. 
 \end{align*}
\end{theorem}

We will show through a set of examples where $\EE_{(s,a)\sim d^{\pi_e}_P}[\min(\sigma(s,a),1) ]$ shrinks to zero as $n_o\to\infty$ under the partial coverage, i.e., when $\rho$ covers $d^{\pi_e}_P$ . Asymptotically, $\mathrm{Err}_e$ will dominate the bound. 
Note that $\mathrm{Err}_e$ has two components, a linear $H$ and a term that corresponds to the statistical error related to expert samples and function class complexity. 
Comparing to BC, which has a rate $O(H^2\sqrt{\log (|\Pi|)/n_e})$ \citep[Chapter 14]{agarwal2019reinforcement}, we see that the horizon dependence is improved. 


Before going to each analysis of $\mathrm{Err}_{o}$, we highlight two important points in our analysis. First, our bound requires only the partial coverage, i,e., it depends on $\pie$-concentrability coefficient which measures the discrepancy between the offline data and expert data. 
This is the first work deriving the bound with $\pie$-concentrability coefficient in IL with offline data. Second, our analysis covers non-parametric models. This is a significant contribution as previous pessimistic offline RL finite-sample error results have been limited to the finite-dimensional linear models or discrete MDPs \citep{JinYing2020IPPE,RashidinejadParia2021BORL}.

\begin{remark}[Implications on offline RL]\label{rem:implication_1}
As in \cref{thm:main}, we have $V^{\hat \pi_{\RL}}_{P,c}-   V^{\pi^{*}}_{P,c}=O(H^2\EE_{(s,a)\sim d^{\pi^{*}}_P}[\sigma(s,a)])$ (\cref{ape:main}). Note similar results have been obtained in \citep{Yu2020,Kidambi2020}. Since this term is $\Error_o$ by just replacing $\pie$ with $\pi^{*}$, this offline RL result is a by-product of our analysis.
\end{remark}

\subsection{Analysis: Discrete MDPs} \label{sec:discrete}
We start from discrete MDP as a warm up. Denote  $\ts C^{\pie}=\max_{(s,a)} d^{\pie}_P(s,a)/\rho(s,a)$. 

\begin{theorem}\label{cor:tabular_bound}
Suppose $\lambda=\Omega(1)$ and the partial coverage $C^{\pie}<\infty$. With probability $1-\delta$,
\begin{align*} \ts
   \Error_o\leq c_1 H^2\prns{ \sqrt{\frac{C^{\pie} |\Scal|^2|\Acal|}{n_o}}+\frac{C^{\pie}|\Scal||\Acal|}{n_o} } \cdot \log(|\Scal||\Acal |c_2/\delta) ,
\end{align*}
where $c_1, c_2$ are universal constants. 
\end{theorem}

The error does not depend on $\sup_{\pi\in \Pi}C^{\pi}$ or $\bar C=\sup_{\pi \in\Pi}\max_{(s,a)}d^{\pi}_P(s,a)/d^{\pie}_P(s,a)$. 
We only require the \emph{partial coverage} $C^{\pie} < \infty$, which is much weaker than $\ts \sup_{\pi\in \Pi}C^{\pi} < \infty$ ($\rho$ has global coverage) and  $\bar C< \infty$ ($d^{\pie}_P$ has global coverage \citep{SpencerJonathan2021FiIL}). When $C^{\pie}$ is small and $n_o$ is large enough, $\ts \mathrm{Err}_e = O\left( H \sqrt{|\Scal||\Acal| / n_e} \right)$ dominates  $\mathrm{Err}_o$ in \cref{thm:main}. Then, the error is linear in horizon $H$.
\subsection{Analysis: KNRs and GPs for Continuous MDPs}
Now we move to continuous state-action MDPs. In continuous MDPs, assuming the boundedness of density ratio $C^{\pie}$ is still a strong assumption. As we dive into the KNR and the nonparametric GP model, we will replace the density ratio with a more refined concept \emph{relative condition number}. 

\paragraph{KNRs} 

Let $\Sigma_{\rho}=\EE_{(s,a)\sim \rho}[\phi(s,a)\phi(s,a)^{\top}]$ and $\Sigma_{\pie}=\EE_{(s,a)\sim d^{\pie}_P}[\phi(s,a)\phi(s,a)^{\top}]$. We define the \emph{relative condition number} as $\ts C^{\pie} = \sup_{x \in \mathbb{R}^d}\left(\frac{x^{\top}\Sigma_{\pie} x}{x^{\top}\Sigma_{\rho}x}\right)$.  
Even when density ratio is infinite, this number could still be finite as it concerns subspaces on $\phi(s,a)$ rather than the whole $\Scal\times \Acal$. 

To further gain its intuition, we can consider discrete MDPs and the feature mapping $\phi(s,a)\in \mathbb{R}^{|\Scal | |\Acal|}$ which is a one-hot encoding vector that has zero everywhere except one at the entry corresponding to the pair $(s,a)$. In this case, the relative condition number is reduced to $\max_{(s,a)} d^{\pie}_P(s,a) / \rho(s,a)$.

\begin{theorem}[Error for KNRs]  \label{thm:linear_error}Suppose $\sup_{s,a}\|\phi(s,a)\|\leq 1$, $\lambda=\Omega(1),\zeta^2=\Omega(1),\|W^{*}\|_2=\Omega(1)$ and the partial coverage $C^{\pie}<\infty$. 
With probability $1-\delta$, 
\begin{align} \ts 
     \mathrm{Err}_o\leq  c_1H^2\left( \mathrm{rank}^2(\Sigma_{\rho})+ \mathrm{rank}(\Sigma_{\rho}) \log(\frac{c_2}{\delta})\right)  \sqrt{\frac{d_{\Scal}C^{\pie} }{n_o}} \cdot {\log^{1/2} (1+n_o) },
\end{align}
where $c_1$ and $c_2$ are some universal constants. 
\end{theorem}

Theorem \ref{thm:linear_error} suggests $ \mathrm{Err}_o$ is $\tilde O(H^2\rank[\Sigma_{\rho}]^2\sqrt{d_{|\Scal|}C^{\pie}/n_o})$. In other words, when $C^{\pie},\rank[\Sigma_{\rho}]$ are small and the offline sample size $n_o$ is large enough, $\mathrm{Err}_e$ dominates  $\mathrm{Err}_o$ in \cref{thm:main}. Again, in this case, $\mathrm{Err}_e = O\left( H \sqrt{ \ln(\Fcal)/n_e } \right)$, and we see that it grows linearly w.r.t horizon $H$.

Our result is distribution dependent and captures the possible low-rankness of the offline data, i.e., $\mathrm{rank}[\Sigma_{\rho}]$ depends on $\rho$ and could be much smaller than the ambient dimension of feature $\phi(s,a)$. The quantity $C^{\pie}$ corresponds to the discrepancy measured between the batch data and expert data. This is much smaller than the worst-case concentrability coefficient: $
   \tilde C=\sup_{\pi\in \Pi}C^{\pi}$.

\begin{remark}\label{rem:offline_linear}
In RL, a similar quantity has been analyzed in \citep{JinYing2020IPPE}, which studies the error bound of linear FQI with pessimism. Comparing to our result only requiring partial coverage, \citep[Corollary 4.5]{JinYing2020IPPE} assumes the global coverage, i.e., $\Sigma_{\rho}$ is full-rank, which is stronger than $   \tilde C <\infty$.  
\end{remark}


\paragraph{GPs} Now we specialize our main theorem to non-parametric GP models. 
For simplicity, following \citep{Srinivas2010}, we assume $\Scal\times \Acal$ is a compact space. We also suppose the following. Recall $x := (s,a)$. 
\begin{assum}\label{assm:kernel_mercer}
$k(x,x)\leq 1,\forall x\in \Scal\times \Acal$. $k(\cdot,\cdot)$ is a continuous and positive semidefinite kernel. 
\end{assum}
Under the \cref{assm:kernel_mercer}, we can use Mercer's theorem \citep{WainwrightMartinJ2019HS:A}, which shows that there exists a set of pairs of eigenvalues and eigenfunctions $\{ \mu_i, \psi_i \}_{i=1}^{\infty}$, where $\int \rho(x) \psi_i(x) \psi_i(x) dx  = 1$ for all $i$ and $\int \rho(x) \psi_i(x)\psi_j(x) dx = 0$ for $i\neq j$. Eigenfunctions and eigenvalues essentially defines an infinite-dimensional feature mapping $\phi(x) := [ \sqrt{\mu_1} \psi_1(x), \dots, \sqrt{\mu_\infty} \psi_\infty(x) ]^{\top}$. Here, $k(x,x) = \phi(x)^{\top} \phi(x)$, and any function $f\in \Hcal_k$ can be represented as $f(\cdot) = \alpha^{\top} \phi(\cdot)$. Note that the eigenvalues and eigenfunctions are defined w.r.t the offline data $\rho$, thus our result here is still distribution dependent rather than a worst case analysis which often appears in online RL/IL settings \citep{Srinivas2010,Kakade2020,2020Yang,Chowdhury2019}.


Assume eigenvalues $\{\mu_1,\dots, \mu_\infty\}$ are in non-increasing order, we define the effective dimension, 
\begin{definition}[Effective dimension]
     $d^{*}  =   \min\{j \in \mathbb{N}: j\geq B(j+1)n_o/\zeta^2\},\, B(j)=\sum_{k=j}^{\infty} \mu_k$. 
\end{definition}
The effective dimensions $d^{*}$ is widely used and calculated for many kernels \citep{ZhangTong2005LBfK,BachFrancis2017OtEb,Valko2013,Chiappa2020}. 
In finite-dimensional linear kernels $\{x\mapsto a^{\top}\phi(x);a\in \RR^d\}$ ($k(x,x)=\phi^{\top}(x)\phi(x)$), we have $d^{*}\leq \rank[\Sigma_{\rho}]$. Thus, $d^{*}$ is considered to be a natural extension of $\rank[\Sigma_{\rho}]$ to infinite-dimensional models. 
\begin{theorem}[Error for GPs]\label{thm:infinite}
Let $\ts \Sigma_{\pie}=\EE_{x \sim d^{\pie}_P}[\phi(x)\phi(x)^{\top}],\Sigma_{\rho}\ts =\EE_{x \sim \rho}[\phi(x)\phi(x)^{\top}]$. Suppose \cref{assm:kernel_mercer}, $\zeta^2=\Omega(1)$ and the partial coverage $\ts C^{\pie}= \sup_{\|x\|_2\leq 1}(x^{\top} \Sigma_{\pie} x/x^{\top} \Sigma_\rho x)<\infty$.
With probability $1-\delta$, 
    \begin{align}\ts 
  \mathrm{Err}_o\leq  c_1H^2\left( (d^{*})^2+ d^{*} \log(c_2/\delta)\right)\sqrt{\frac{d_\Scal C^{\pie} }{n_o} } \cdot  \sqrt{\log^3(c_2d_{\Scal}n_o/\delta)\log (1+n_o) }, 
    \end{align}
where $c_1,c_2$ are universal constants. 
\end{theorem}
The theorem suggests that $\mathrm{Err}_o$ is $\tilde O(H^2 {d^*}^2\sqrt{d_{\Scal}C^{\pie}/n_o})$. Thus, when $C^{\pie}$, $d^{*}$ are not so large and $n_o$ is large enough, $\mathrm{Err}_e$ asymptotically dominates $\mathrm{Err}_o$ in \cref{thm:main} (again $\mathrm{Err}_e$ is linear in $H$). 

While we defer the detailed proof of the above theorem to \pref{app:GP_proof}, we highlight some techniques we used here. The analysis is reduced to how to bound the information gain $\Ical_{n_o}$ and $\EE_{x\sim d^{\pie}_P}[k_{n_o}(x,x)]$. In both cases, we analyze them into two steps: transforming them into the variational representation and then bounding them via the uniform low with localization (\pref{lem:kernel2}). 


\begin{remark}[Implication to Offline RL]\label{rem:offline_nonlinear}
As related literature, in model-free offline RL, \citep{UeharaMasatoshi2021FSAo,DuanYaqi2021RBaR} obtain the finite-sample error bounds using nonparametric models. Though their bounds can be characterized by the effective dimension, their bounds assume full coverage, i.e., $\max_{(s,a)}1/\rho(s,a)<\infty$. 
\end{remark}

\section{Practical Algorithm}
\label{sec:prac_alg}

\begin{algorithm}[t]
\caption{A practical instantiation of \alg}
\begin{algorithmic}[1]
    \STATE \textbf{Require}: expert dataset $\Dcal_e = \{s, a\}$, offline dataset $\Dcal_{o} := \{s,a,s'\}$
    \STATE Train an ensemble of neural network models $\{ \hat{g}_1,\dots, \hat{g}_n\}$ where each $P_i$ starts with different random initialization;
    \STATE Set bonus $b(s,a) = \max_{i,j} \| g_i(s,a) - g_j(s,a) \|_2$ and initialize $\pi_{\theta_0}$. 
         \FOR{$ t = 0 \to T-1$}
        		\STATE  $ w_t = \arg\max_{\|w\|_2\leq 1} w^{\top}\left(\mathbb{E}_{(s,a)\sim d^{\pi}_{\hat{P}}}[\phi(s,a)] -  \mathbb{E}_{(s,a)\sim \Dcal_e}[ \phi(s,a)]\right) $,  $f_t(s,a) := w_t^{\top} \phi(s,a)$ \label{line:mmd}
		\STATE \resizebox{0.9\textwidth}{!}{$\ts \theta_{t+1} = \theta_t - \eta F_{\theta_t}^{-1}\left( \mathbb{E}_{(s,a)\sim d_{\hat{P}}^{\pi_{\theta_t}}}   \left[ \nabla \ln\pi_{\theta_t}(a|s) A^{\pi_{\theta_t}}_{\hat{P}, f_t+b} (s,a)\right]+ \lambda \mathbb{E}_{(s,a)\sim \Dcal_e} \left[ \nabla \ell(a,s,\pi_{\theta_t})\right] \right)    $} \label{line:npg}
        \ENDFOR
     \end{algorithmic}
\label{alg:milo_prac}
\end{algorithm}

In this section we instantiate a practical version of \alg\, using neural networks for the model class $\Pcal$ and policy class $\Pi$. We use the Maximum Mean Discrepancy (MMD) with a Radial Basis Function kernel as our discriminator class $\Fcal$. Note using MMD as our discrepancy measure allows us to compute the exact maximum discriminator $\argmax_{f\in\Fcal}$ in closed form (and is detailed in appendix). We use a KL-based trust-region formulation for incremental policy update inside the learned model $\hat{P}$. Based on Eq.~\pref{eq:final_obj}, we first formalize the following constrained optimization framework:
\vspace{-0.5mm}
\begin{align*}
\resizebox{\textwidth}{!}{%
$\min_{\pi\in \Pi} \max_{f\in\Fcal} \left( \mathbb{E}_{(s,a)\sim d^{\pi}_{\hat{P}} } \left[ f(s,a) + b(s,a) \right]  - \mathbb{E}_{ (s,a)\sim \Dcal_e}[f(s,a)]    \right)\, \text{s.t.}\, \mathbb{E}_{(s,a)\sim \Dcal_e}[\ell\left( a, s, \pi \right)] \leq \delta,$
}
\end{align*} 
where $\ell: \Acal\times\Scal\times \Pi \mapsto \mathbb{R}$ is a loss function (e.g., negative log-likelihood or any supervised learning loss one would use in BC). Essentially, since we have $\Dcal_e$ available, we use it together with any supervised learning loss to constrain the policy hypothesis space. Note for a deterministic expert $\pie$, the expert policy is always a feasible solution. Thus adding this constraint reduces the complexity of the policy class but does not eliminate the expert policy, and our analysis in \pref{sec:analysis} still applies. 

In our practical instantiation, we replace the hard constraint instead by a Lagrange multiplier, i.e. we use the behavior cloning objective as a regularization term when solving the min-max problem:
\begin{align*}\ts
\min_{\pi\in \Pi} \max_{f\in\Fcal} \left[ \mathbb{E}_{(s,a)\sim d^{\pi}_{\hat{P}} } \left( f(s,a) + b(s,a) \right)  - \mathbb{E}_{(s,a)\sim\Dcal_e}[f(s,a)]    \right] + \lambda \cdot \mathbb{E}_{(s,a)\sim \Dcal_e}[\ell\left( a, s, \pi \right)]. 
\end{align*} 
Given policy $\pi_{\theta_t}$ ($\theta_t$ denotes the parameters), we update the discriminator $f_t$.
Then, with a fixed $f_t$, in order to update policy $\pi$ we use NPG as in \pref{line:npg} in \pref{alg:milo_prac}, where $A^{\pi_{\theta}}_{\hat{P}, f+b}$ is the disadvantage function of $\pi_{\theta}$ 
and  $\ts F_{\theta_t} := \EE_{(s,a)\sim d_{\hat{P}}^{\pi_{\theta_t}}}[\nabla \ln\pi_{\theta_t}(a|s)  \nabla \ln\pi_{\theta_t}(a|s)^{\top}]$ is the fisher information matrix. 


\begin{figure}[b]
    \vspace{-8mm}
    \centering
    \includegraphics[width=\textwidth]{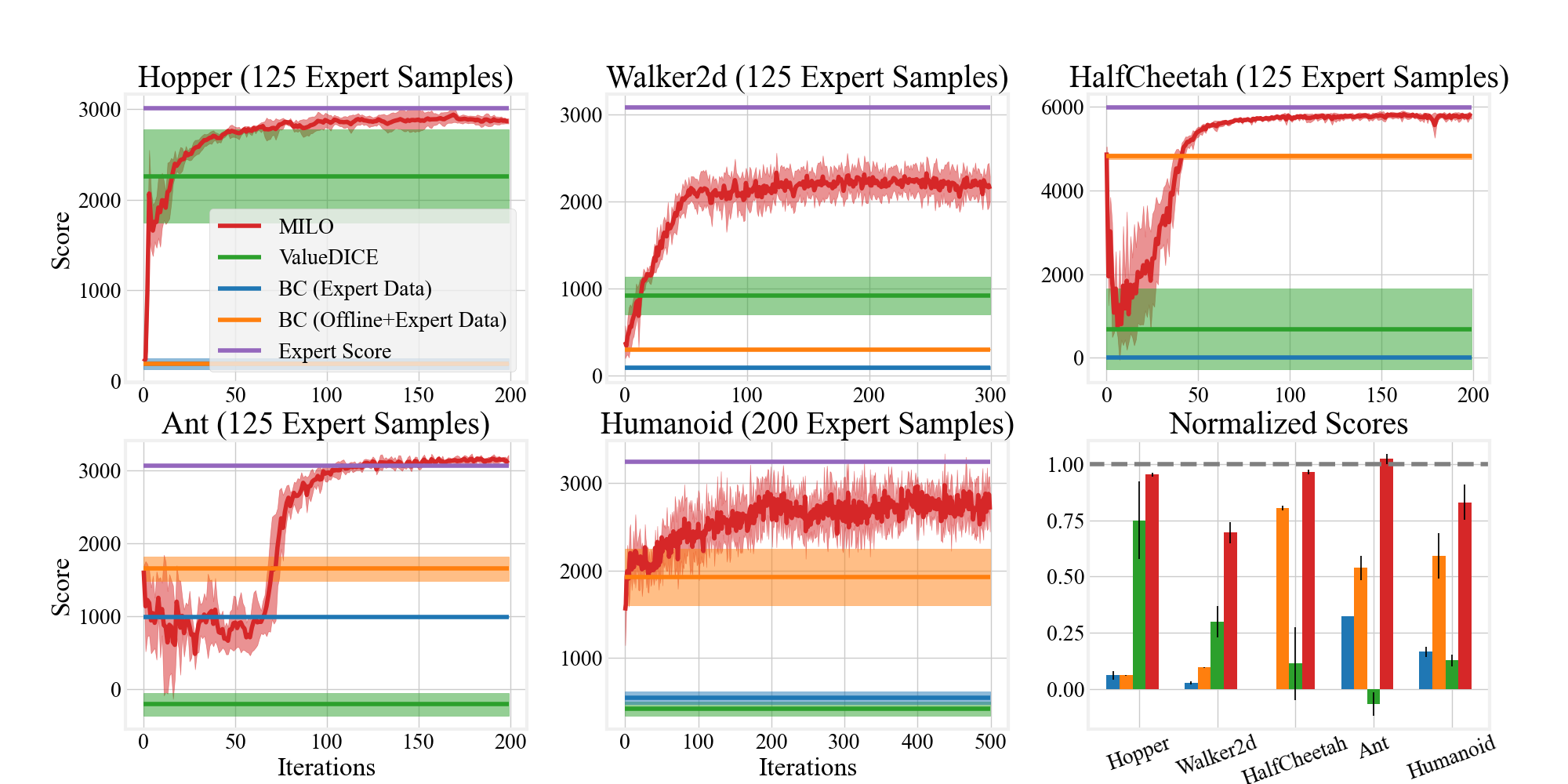}\vspace{-5pt}
    \caption{Learning curves across five seeds for \alg~plotted against the best performance of BC after 1000 epochs of training on the expert/offline+expert data and the best performance of ValueDICE after 10 thousand iterations. The bottom right bar graph shows the expert performance normalized scores where we plot the performance at the last iteration for \alg.}
    \label{fig:core_results}
    \vspace{-10pt}
\end{figure}
\section{Experiments}
\label{sec:experiments}
We aim to answer the following questions with our experiments:
(1) How does \alg~perform relative to other offline IL methods, 
(2) What is the impact of pessimism on \alg's performance?
(3) How does the behavior policy's coverage impact \alg's performance?
(4) How does \alg's result vary when we increase the number of samples drawn from the expert policy?

We evaluate \alg~on five environments from OpenAI Gym \citep{brockman2016gym} simulated with MuJoCo \citep{todorov2012mujoco}: \hopper, \walker, \halfcheetah, \ant, and \humanoid. We compare \alg~against the following baselines: (1) ValueDICE \citep{KostrikovIlya2019ILvO}, a state-of-the-art off-policy IL method modified for the offline IL setting; (2) BC on the expert dataset; and (3) BC on both the offline and expert dataset.

\begin{wraptable}[11]{r}[0cm]{0.5\textwidth}
    \centering
    \caption{Performance for expert and behavior policy used to collect expert and offline datasets respectively.}
    \vspace{-1mm}
    \resizebox{0.5\textwidth}{!}{
    \begin{tabular}{c|cc}
    \toprule
    Environment & \begin{tabular}{c}Expert\\ Performance\end{tabular} & \begin{tabular}{c}Behavior\\ Performance\end{tabular} \\
    \midrule
        \hopper      & 3012 & 752 (25\%)  \\
        \walker      & 3082 & 1383 (45\%) \\
        \halfcheetah & 5986 & 3972 (66\%) \\
        \ant         & 3072 & 1208 (40\%) \\
        \humanoid    & 3248 & 1505 (46\%) \\
    \bottomrule
    \end{tabular}}
    \label{tab:performance}
\end{wraptable}
For the expert dataset, we first train expert policies and then randomly sample $(s,a)$-pairs from a pool of 100 expert trajectories collected from these expert policies. We randomly sample to create very small expert $(s,a)$-pair datasets where BC struggles to learn. Note that BC is effective at imitating the expert for MuJoCo tasks even with a single trajectory; prior works \citep{ho2016generative, kostrikov2019dac, KostrikovIlya2019ILvO} have used similar sub-sampling strategies to create expert datasets to make it harder for BC to learn. 
The offline datasets are collected building on prior Offline RL works \citep{WuYifan2019BROR,Kidambi2020}; each dataset contains 1 million samples from the environment. We first train behavior policies with mean performances often less than half of the expert performance (Table \ref{tab:performance}, column 2).
All results are averaged over five random seeds. See appendix for details on hyperparameters, environments, and dataset composition.
\subsection{Evaluation on MuJoCo Continuous Control Tasks}
\label{subsec:experiments_mujoco}
Figure \ref{fig:core_results} presents results comparing \alg~against benchmarks.
\,\alg~is able to achieve close to expert level performance on three out of the five environments and outperforms both BC and ValueDICE on all five environments. Both \alg~and ValueDICE were warmstarted with one epoch of BC on the offline dataset. We significantly outperform BC's performance when trained on the expert dataset, suggesting \alg~indeed mitigates covariate shift through the use of a static offline dataset of $(s,a)$-pairs. BC on both the offline and expert dataset does improve the performance, but this still cannot successfully imitate the expert since BC has no way of differentiating random/sub-optimal trajectories from the expert samples. ValueDICE, on the other hand, does explicitly aim to imitate the expert samples; however, in theory, it would require either the offline data (i.e. the replay buffer) or the expert samples to have full coverage over the state-action space. Since our offline dataset is mainly collected from a sub-optimal behavior policy and our expert samples are from a high quality expert, neither our offline nor our expert dataset is likely to have full coverage globally; thus potentially hurting the performance of algorithms like ValueDICE. Note that \alg~is still able to perform reasonably well across environments even with these offline and expert datasets.

\begin{figure}[b]
    \centering
    \includegraphics[width=\textwidth]{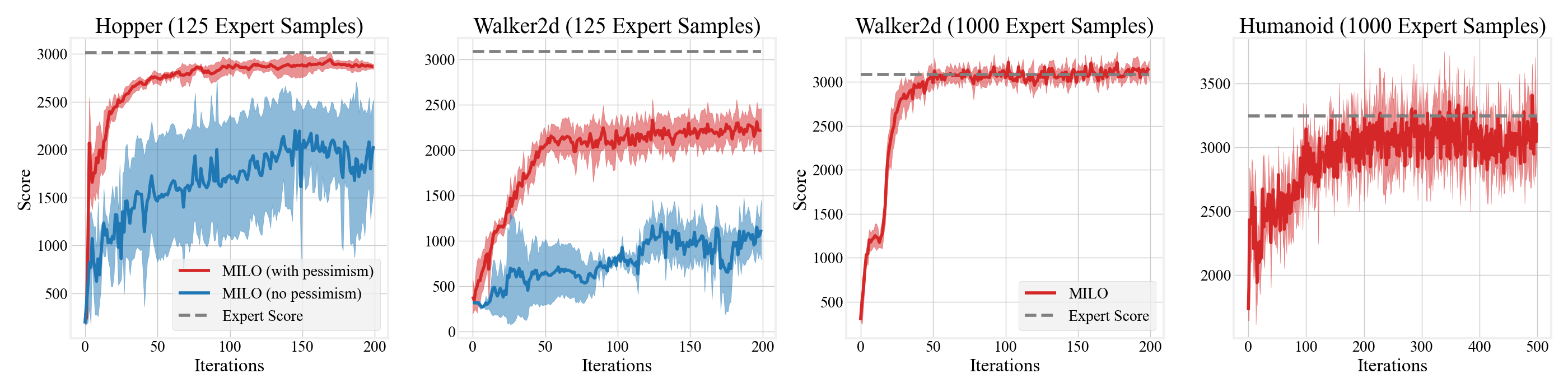}
    \caption{(Left 2) Learning curves for Hopper and Walker2d with (\textcolor{red}{red}) and without (\textcolor{blue}{blue}) pessimism. \alg\, generally performs worse without pessimism. (Right 2) Learning curve for Walker2d and Humanoid with more expert samples.} 
    \label{fig:ablation}
\end{figure}
\subsection{Ablation}
\paragraph{Impact of Pessimism} Figure \ref{fig:ablation} (Left 2) presents \alg's performance on two representative environments with and without pessimism (i.e., setting penalty to be zero) added to the imitation objective. Pessimism stabilizes and improves the final performance for \alg. In general, \alg~consistently outperforms benchmarks and/or achieves expert level performance for a given set of hyperparameters. See the Appendix for evaluation on other environments.

\paragraph{Behavior with more expert samples} We investigate whether \alg~is able to achieve expert performance with more expert samples in the two environments (walker and humanoid) that it did not solve with very small expert datasets in Figure~\ref{fig:core_results}. Figure \ref{fig:ablation} (Right 2) shows that with one trajectory worth of expert samples, \alg~is able to achieve expert performance on walker and humanoid.

\paragraph{Impact of Coverage} \begin{wraptable}[11]{r}[0cm]{0.5\textwidth}
    \vspace{-4mm}
    \centering
    \caption{Expert performance normalized scores on three different offline datasets collected from behavior policies with approximately $50\%$, $25\%$, and random performance relative to the expert.}
    \resizebox{0.5\textwidth}{!}{
    \begin{tabular}{c|ccc}
    \toprule
        Environment &  $\approx50\%$ & $\approx25\%$ & Random \\
    \midrule
        \hopper      & $0.95\pm0.01$ & $0.66\pm0.33$ & $0.42\pm0.36$\\
        \walker      & $0.72\pm0.02$ & $0.27\pm0.06$ & $0.23\pm0.12$\\
        \halfcheetah & $0.96\pm0.01$ & $0.01\pm0.02$ & $0.01\pm0.02$\\
        \ant         & $1.02\pm0.02$ & $0.99\pm0.01$ & $0.21\pm0.52$\\
        \humanoid    & $0.88\pm0.10$ & $0.72\pm0.03$ & $0.08\pm0.01$\\
    \bottomrule
    \end{tabular}
    }\label{tab:coverage}
\end{wraptable}As our analysis suggests, \alg's performance degrades as the offline data's coverage over the expert's state-action space decreases. We use the behavior policy's value as a surrogate for lower coverage, i.e. a lower value suggests lower coverage. 
We generate two additional offline datasets for each environment by lowering the performance of the behavior policy. The three datasets are: (1) the original offline datasets used in Table~\ref{tab:performance} ($\approx 25\%$ for \hopper\, and $\approx 50\%$ for others); (2) ones that have roughly half the performance of (1) ($12\%$ for \hopper\, and $\approx 25\%$ for others); and (3) ones collected from a random behavior policy (Random). Table \ref{tab:coverage} shows that \alg~performs reasonably on three environments even with a lower coverage dataset (second column) and achieves more than 20\% of the expert performance on three environment even with the Random dataset.

\section{Conclusion}
\label{sec:conclusion}
\alg~investigates how to mitigate covariate shift in IL using an offline dataset of environment interactions that has partial coverage of the expert's state-action space. We show the effectiveness of \alg~both in theory and in practice. In future works, we hope to scale to image-based control.

We want to highlight the potential negative societal/ethical impacts our work. An IL algorithm is only as good as the expert that it is imitating, not only in terms of performance but also with regards to the negative biases and intentions that the demonstrator has. When designing real-world experiments/applications for \alg~we believe the users should do their due diligence on removing any negative bias or malicious intent in the demonstrations that they provide.

\newpage
\bibliographystyle{abbrv}

\newpage

\tableofcontents
\appendix

\newpage
\section{Penalty Designs}

\label{ape:calibrated_models}

We show that the penalty design in \cref{sec:analysis} is valid, i.e, the model is well-calibrated for tabular MDPs, KNRs, and GPs. 

\subsection{Tabular models}
\begin{lemma}
 With probability $1-\delta$, 
\begin{align*}
        \|\hat P(\cdot|s,a)-P(\cdot|s,a)\|_1 \leq   \sqrt{\frac{|\Scal|\log 2+\log (2|\Scal||\Acal|/\delta)}{2\{N(s,a)+\lambda\}} }+\frac{\lambda}{N(s,a)+\lambda}\quad \forall (s,a)\in \Scal\times \Acal. 
\end{align*}
\end{lemma}
\begin{proof}
When $N(s,a)>0$, we use the concentration inequality of discrete distributions \citep{Jiang2020_note}. Then, with probability $1-\delta$,
\begin{align*}
        \left \|\frac{N(\cdot|s,a)}{N(s,a) }-P(\cdot|s,a)\right \|_1\leq \sqrt{\frac{|\Scal|\log 2+\log (2|\Scal||\Acal|/\delta)}{2N(s,a)} }\quad \forall(s,a)\in \{(s,a):N(s,a)>0\}. 
\end{align*}
Thus, noting $0<N(s,a)/(N(s,a)+\lambda)<1$, with probability $1-\delta$, we have $\forall(s,a)\in \{(s,a):N(s,a)>0\}$, 
\begin{align}\label{eq:tabular_con} 
       \left \|\frac{N(\cdot|s,a)}{(N(s,a)+\lambda)} -P(\cdot|s,a)\times \frac{N(s,a)}{N(s,a)+\lambda} \right \|_1\leq \sqrt{\frac{|\Scal|\log 2+\log (2|\Scal||\Acal|/\delta)}{2\{N(s,a)+\lambda\}} }. 
\end{align}
Besides, the above inequality is still well-defined and holds including the case $N(s,a)=0$. Thus, with probability $1-\delta$, we have $\forall(s,a)\in \Scal\times \Acal$, we have \cref{eq:tabular_con}. 

Recall the estimator $\hat P$ is $N(s'|s,a)/(N(s,a)+\lambda)$. Therefore,

\resizebox{\textwidth}{!}{%
\begin{math}
\begin{aligned}
    \|\hat P(\cdot|s,a)-P(\cdot|s,a)\|_1&\leq    \left\|\hat P(\cdot|s,a)-P(\cdot|s,a)\times \frac{N(s,a)}{N(s,a)+\lambda} \right\|_1+\left\|P(\cdot|s,a)-P(\cdot|s,a)\times \frac{N(s,a)}{N(s,a)+\lambda}\right \|_1 \\
    & \leq \sqrt{\frac{|\Scal|\log 2+\log (2|\Scal||\Acal|/\delta)}{2\{N(s,a)+\lambda\}} }+\frac{\lambda}{N(s,a)+\lambda}. 
\end{aligned}\end{math}}

This concludes the proof. 
\end{proof}

\subsection{KNRs}
In KNRs, the ground truth model is $s'=W^{*}\phi(s,a)+\epsilon,\epsilon\sim \Ncal(0,\zeta^2 \Ib)$, where $s\in \RR^{d_{\Scal}},a\in \RR^{d_{\Acal}},\phi:\Scal \times \Acal \to \RR^d$.  We define 
\begin{align*}
    \|\phi(s,a)\|_{_{\Sigma_{n_o}^{-1}}}\coloneqq \phi^{\top}(s,a)\Sigma_{n_o}^{-1}\phi(s,a).
\end{align*}

\begin{lemma}  With probability at least $1-\delta$, we have:
\begin{align*}
\left\| \widehat{P}(\cdot | s,a) - P(\cdot | s,a)  \right\|_{1} \leq \min\left\{ \frac{\beta_{n_o}}{\zeta} \left\| \phi(s,a) \right\|_{\Sigma_{n_o}^{-1}},2 \right\}\quad \forall (s,a)\in \Scal\times \Acal,
\end{align*} 
where 
\begin{align*}
\beta_{n_o} =  \sqrt{ 2\lambda \|W^\star\|^2_2  + 8 \zeta^2 \left(d_{\Scal} \ln(5) +  \ln(1/\delta) +  \bar \Ical_{n_o} \right) }, \quad   \bar \Ical_{n_o}=\ln\left( \det(\Sigma_{n_o}) / \det(\lambda \Ib) \right). 
\end{align*}
\label{prop:knr_bonus}

\end{lemma}

\begin{proof} The proof directly follows the confidence ball construction and proof from \cite{Kakade2020}. Specifically, from Lemma B.5 in \cite{Kakade2020}, we have that with probability at least $1-\delta$,
\begin{align*}
    \left\| \left(\widehat{W}  - W^\star\right) \left(\Sigma_{n_o}\right)^{1/2}  \right\|_2^2 \leq \beta^2_{n_o} . 
\end{align*}
Thus, with \pref{lem:gaussian_tv}, we have:
\begin{align*}
    \left\| \widehat{P}(\cdot | s,a) - P(\cdot | s,a)  \right\|_{1} &\leq \frac{1}{\zeta} \left\| (\widehat{W} - W^\star) \phi(s,a)   \right\|_2 \\
    &\leq \left\| (\widehat{W} - W^\star) (\Sigma_{n_o})^{1/2}  \right\|_2 \left\| \phi(s,a) \right\|_{\Sigma_{n_o}^{-1}}  / \zeta \leq \frac{\beta_{n_o}}{\zeta} \| \phi(s,a) \|_{\Sigma_{n_o}^{-1}}.
\end{align*} This concludes the proof. 
\end{proof}

\subsection{Gaussian processes}

Let $\Hcal_k$ be the RKHS with the kernel $k(\cdot,\cdot)$. We denote the associated norm and inner product by $\|\cdot\|_{k}$ and $\langle \cdot,\cdot \rangle_{k}$.  
In GPs, the ground truth model is defined as $s'=g^{*}(s,a)+\epsilon,\epsilon\sim \Ncal(0,\zeta^2 \Ib)$, where $g^{*}$ belongs to an RKHS $\Hcal_k$. 

\begin{lemma}
With probability $1-\delta$, 
\begin{align*}
        \|\hat P(\cdot|s,a)-P(\cdot|s,a)\|_1\leq \min \prns{\frac{\beta_{n_o}}{\zeta}\sqrt{k_{n_o}((s,a),(s,a))},2} \quad \forall (s,a)\in \Scal\times \Acal, 
\end{align*}
and 
\begin{align*}
    \beta_{n_o}=\sqrt{d_{\Scal}\{2+150 \log^3(d_{\Scal}n_o/\delta)\mathcal{I}_{n_o}\}} ,\quad \mathcal{I}_{n_o}=\log(\det(\Ib+\zeta^{-2}\Kb_{n_o})). 
\end{align*}
\end{lemma}
\begin{proof}
Let $\hat g_i$ and $g^{*}$ be $i$-th component of $\hat g$ and $g^{*}$. We have 
\begin{align*}
    \|\hat P(\cdot|s,a)-P(\cdot|s,a)\|_1& \leq \frac{1}{\zeta}\|\hat g(s,a)-g^{*}(s,a)\|_2  \tag{ \pref{lem:gaussian_tv}} \\
     &= \frac{1}{\zeta}\prns{\sum_{i=1}^{d_{\Scal}}\{\hat g_i(s,a)-g^{*}_i(s,a)\}^2  }^{1/2}\\ 
       &\leq \frac{1}{\zeta}\prns{\sum_{i=1}^{d_{\Scal}}k_{n_o}((s,a),(s,a))\|\hat g_i-g^{*}_i\|^2_{k_{n_o}}}^{1/2}\, \tag{CS inequality and $g=\langle g(\cdot),k((s,a),\cdot)\rangle_{k_{n_o}}$}. 
\end{align*}
By \cite[Theorem 6]{Srinivas2010}, with probability $1-\delta$, we have
\begin{align*}
\|\hat g_i(s,a)-g^{*}_i\|_{k_{n_o}}\leq \beta_{n_o}\quad \forall i\in [1,\cdots,d_{\Scal}]. 
\end{align*}
This concludes the statement. 

\end{proof}

\section{Proof of \cref{thm:main}}
\label{ape:main}
In this section, we prove \pref{thm:main}. We also prove the RL version of \pref{thm:main} when the cost $c$ is given and the goal is policy optimization.  Before that, we prepare several lemmas. 

\begin{lemma} \label{lem:Hoeffding} With probability $1-\delta$,  we have $\forall f\in \Fcal$, 
\begin{align*}
    |\EE_{(s,a)\sim d^{\pie}}[f(s,a)]- \EE_{D_e}[f(s,a)] |\leq \epsilon_{\mathrm{stat}},\quad \epsilon_{\mathrm{stat}}=\sqrt{\log(2|\Fcal|/\delta)/2n_e}. 
\end{align*}
\end{lemma}
\begin{proof}
From Hoeffding's inequality and a union bound over $\Fcal$.
\end{proof}
\begin{lemma}[Pessimistic Policy Evaluation 1 ]\label{lem:pessmistic}
Suppose \pref{assum:calibration} holds and $\max_{f\in\Fcal}\|f\|_{\infty}\leq 1$. With probability at least $1-\delta$, $\forall \pi \in \Pi,\forall f\in \Fcal$, 
\begin{align*}
   0 \leq V^{\pi}_{\hat P,f+b}-   V^{\pi}_{P,f}. 
\end{align*}
\end{lemma}

\begin{proof}[Proof of \pref{lem:pessmistic}]
We denote the expected total cost of $\pi$ under $\hat P$ and cost function $f$ by $V^{\pi}_{\hat P,f:h}(s,a)$. In this proof, we condition on the event 
\begin{align*}
    \|\hat P(\cdot|s,a)-P(\cdot|s,a)\|_1 \leq  \min(\sigma(s,a),2)\,\quad \forall (s,a)\in \Scal\times \Acal. 
\end{align*}

We use the inductive hypothesis argument. We start from $h=H+1$, where $  V^{\pi}_{\hat P,f+b:H+1}=  V^{\pi}_{P,f:H+1}=0$. Assume the inductive hypothesis holds at $h+1$, i.e, $$0\leq  V^{\pi}_{\hat P,f+b:h+1}(s)-   V^{\pi}_{P,f:h+1}(s),\quad\,\forall s\in \Scal,\forall \pi \in \Pi,\,\forall f\in \Fcal.$$ Then, $\forall \pi \in \Pi,\,\forall f\in \Fcal$, 
\begin{align*}
      &Q^{\pi}_{P,f:h}(s,a)- Q^{\pi}_{\hat P,f+b:h}(s,a)\\
      &=-b(s,a)+\EE_{s'\sim \hat P(\cdot|s,a)}[V^{\pi}_{P,f:h+1}(s') ]-\EE_{s'\sim P(\cdot|s,a)}[ V^{\pi}_{\hat P,f+b:h+1}(s')]\\ 
      &\leq -b(s,a)+ \EE_{s'\sim \hat P(\cdot|s,a)}[V^{\pi}_{P,f:h+1}(s') ]-\EE_{s'\sim P(\cdot|s,a)}[ V^{\pi}_{P,f:h+1}(s')] \tag{Inductive hypothesis assumption}\\ 
      &\leq -b(s,a)+H \|\hat P(\cdot|s,a)-P(\cdot|s,a)\|_1  \tag{$\|\Fcal\|_{\infty}\leq 1$}\\
      &\leq  -H\min(\sigma(s,a),2)+ H \min(\sigma(s,a),2)=0.  \tag{Bonus construction}
\end{align*}
Then, noting $ Q^{\pi}_{P,f:h}(s,\pi(s))- Q^{\pi}_{\hat P,f+b:h}(s,\pi(s))=   V^{\pi}_{P,f:h}(s)- V^{\pi}_{\hat P,f+b:h}(s)$, we have 
\begin{align*}
     V^{\pi}_{P,f:h}(s)- V^{\pi}_{\hat P,f+b:h}(s)\leq 0\quad \forall \pi \in \Pi,\,\forall f\in \Fcal. 
\end{align*}
This concludes the induction step. 

Then, we have 
\begin{align*}
      V^{\pi}_{P,f}- V^{\pi}_{\hat P,f+b}=    V^{\pi}_{P,f:1}- V^{\pi}_{\hat P,f+b:1} \leq 0\quad \forall \pi \in \Pi,\,\forall f\in \Fcal. 
\end{align*}
\end{proof}

\begin{lemma}[Pessimistic Policy Evaluation 2  ]\label{lem:pessmistic2}
Suppose \pref{assum:calibration} holds and $\|\Fcal\|_{\infty}\leq 1$. With probability at least $1-\delta$, $\forall \pi \in \Pi,\,\forall f\in \Fcal$, 
\begin{align*}
       V^{\pi}_{\hat P,f+b}-   V^{\pi}_{P,f}\leq \mathrm{Error},\quad \mathrm{Error}:= (3H^2+H)\EE_{(s,a)\sim d^{\pi}_P}[\min(\sigma(s,a),2) ]. 
\end{align*}
\end{lemma}

\begin{proof}[Proof of \pref{lem:pessmistic2}]
In this proof, we condition on the event 
\begin{align*}
    \|\hat P(\cdot|s,a)-P(\cdot|s,a)\|_1 \leq  \min(\sigma(s,a),2)\quad \forall(s,a)\in \Scal\times \Acal. 
\end{align*} 

We invoke simulation \pref{lem:simulation}. Then, we have $\forall \pi\in \Pi,\forall f\in \Fcal$  
\begin{align*}
    V^{\pi}_{\hat P,f+b}-   V^{\pi}_{P,f}&=\sum_{h=1}^{H}\EE_{(s,a)\sim d^{\pi}_P}[b(s,a)+\EE_{s'\sim \hat P(\cdot|s,a)}[V^{\pi}_{\hat P,f+b;h}(s')]- \EE_{s'\sim P(\cdot|s,a)}[V^{\pi}_{\hat P,f+b;h}(s')] ] \\ 
    &\leq \sum_{h=1}^{H}\EE_{(s,a)\sim d^{\pi}_P}[b(s,a) + \| V^{\pi}_{\hat P,f+b;h}\|_{\infty}\|  \hat P(\cdot|s,a)-P(\cdot|s,a)\|_1 ]\\
    &\leq H\EE_{(s,a)\sim d^{\pi}_P}[H\min(\sigma(s,a),2) + H(2H+1)\min(\sigma(s,a),2) ]  \tag{ $\| V^{\pi}_{\hat P,f+b;h}\|_{\infty}\leq H(2H+1)$}\\
    &=(3H^2+H)\EE_{(s,a)\sim d^{\pi}_P}[\min(\sigma(s,a),2) ]. 
\end{align*}
Here, we use  $\| V^{\pi}_{\hat P,f+b;h}\|_{\infty}\leq H(2H+1)$ which is derived by $0\leq f+b\leq 2H+1$. 
\end{proof}

By using the above lemmas, we prove our main result. 
\begin{proof}[Proof of \pref{thm:main}]
In this proof, we condition on the event 
\begin{align*}
    \|\hat P(\cdot|s,a)-P(\cdot|s,a)\|_1 \leq  \min(\sigma(s,a),2),  
\end{align*}
which holds with probability $1-\delta$,  and the event in \pref{lem:Hoeffding}, which holds with probability $1-\delta$. 

Then, with probability $1-2\delta$, we have 
\begin{align*}
        V^{\hat \pi_{\IL}}_{P,c}-   V^{\pie}_{P,c}&\leq    V^{\hat \pi_{\IL}}_{\hat P,c+b}-   V^{\pie}_{P,c} \tag{\pref{lem:pessmistic}}\\
        &\leq H\max_{f\in \Fcal}\{\EE_{(s,a)\sim d^{\hat \pi_{\IL}}_{\hat P}}[f(s,a)+b(s,a)]- \EE_{(s,a)\sim d^{\pie}_P}[f(s,a)]\} \tag{$c \in \Fcal$}\\
        &\leq H\max_{f\in \Fcal}\{\EE_{(s,a)\sim d^{\hat \pi_{\IL}}_{\hat P}}[f(s,a)+b(s,a)]- \EE_{\Dcal_e}[f(s,a)]\}+H\epsilon_{\mathrm{stats}} \tag{ \pref{lem:Hoeffding}}\\
        &\leq H\max_{f\in \Fcal}\{\EE_{(s,a)\sim d^{\pie}_{\hat P}}[f(s,a)+b(s,a)]- \EE_{\Dcal_e}[f(s,a)]\}+H\epsilon_{\mathrm{stats}} \tag{$\pie \in \Pi$ and the definition of $\hat \pi_{\IL}$}\\
        &\leq H\max_{f\in \Fcal}\{\EE_{(s,a)\sim d^{\pie}_{\hat P}}[f(s,a)+b(s,a)]- \EE_{(s,a)\sim d^{\pie}_P}[f(s,a)]\}+2H\epsilon_{\mathrm{stats}} \tag{ \pref{lem:Hoeffding}} \\
        &\leq  \max_{f\in \Fcal}\{V^{\pie}_{\hat P,f+b}-   V^{\pie}_{P,f}\}  +2H\epsilon_{\mathrm{stats}} \\ 
        &\leq (3H^2+H)\EE_{(s,a)\sim d^{\pie}_P}[\min(\sigma(s,a),2) ]+ 2H\epsilon_{\mathrm{stats}} \tag{\pref{lem:pessmistic2}}\\
         &\leq (6H^2+2H)\EE_{(s,a)\sim d^{\pie}_P}[\min(\sigma(s,a),1) ]+ 2H\epsilon_{\mathrm{stats}}. 
\end{align*}
This concludes the proof. 
\end{proof}

Finally, we prove the finite-sample error bounds for the RL case. Similar results are obtained in \citep{Kidambi2020,Yu2020}. We use this theorem in the next section. 

\begin{theorem}[Bounds for RL]\label{thm:main2}
Suppose $\pi^{*} \in \Pi, P\in \Pcal$ and Assumption \ref{assum:calibration}. With probability $1-2\delta$, we have
\begin{align}\ts 
     V^{\hat \pi_{\RL}}_{P,c}-   V^{\pi^{*}}_{P,c} \leq (6H^2+2H)\EE_{(s,a)\sim d^{\pi^{*}}_P}[\min(\sigma(s,a),1) ] \label{eq:offline_rl}.    
\end{align}
\end{theorem}

\begin{proof}[Proof of \pref{thm:main2} ]
\begin{align*}
     V^{\hat \pi_{\RL}}_{P,c}-   V^{\pi^{*}}_{P,c} &\leq      V^{\hat \pi_{\RL}}_{\hat P,c+b}-   V^{\pi^{*}}_{P,c} \tag{\pref{lem:pessmistic}}\\
     &= V^{\pi^{*}}_{\hat P,c+b}-   V^{\pi^{*}}_{P,c} \tag{$\pi^{*} \in \Pi$ and the definition of $\hat \pi_{\RL}$ }\\
     & =  (3H^2+H)\EE_{(s,a)\sim d^{\pi^{*}}_P}[\min(\sigma(s,a),2) ]  \tag{\pref{lem:pessmistic2}}\\
     &\leq  (6H^2+2H)\EE_{(s,a)\sim d^{\pi^{*}}_P}[\min(\sigma(s,a),1) ]. 
\end{align*}This concludes the proof.
\end{proof}

\section{Finite sample error bound for each model}
\label{ape:model_error}

In this section, we analyze the bound for the following models: (1) discrete MDPs, (2) KNRs, (3) GPs. All of the proofs are deferred to \pref{sec:error_proof}. We will also discuss the implication to the RL case using \cref{thm:main2}.

\subsection{Discrete MDPs} 

Recall $\pie$-concentratabiliy coefficient is defined by 
\begin{align*}
    C^{\pie}=\max_{(s,a)}\frac{d^{\pie}_P(s,a)}{\rho(s,a)}.    
\end{align*}
Then, the error is calculated as follows. 
\begin{theorem}[Error of \alg\,for discrete MDPs]\label{thm:tabular_bound} 
~\\
\begin{itemize}
    \item With probability $1-\delta$, when $\lambda=\Omega(1)$,
    
    \resizebox{0.9\textwidth}{!}{
    \begin{math}
    \begin{aligned}
            &V^{\hat \pi_{\IL}}_{P,c}-   V^{\pie}_{P,c}\leq \Error_o+\Error_e,\, \\
            & \Error_o=c_1 H^2\log(|\Scal||\Acal|c_2/\delta) \prns{ \sqrt{\frac{C^{\pie} |\Scal|^2|\Acal|}{n_o}}+\frac{C^{\pie}|\Scal||\Acal|}{n_o} },\,\Error_e= 2H\sqrt{\frac{\log(2|\Fcal|/\delta)}{2n_e}}, 
    \end{aligned}
    \end{math}}
where $c_1$ and $c_2$ are some universal constants. 
\item With probability $1-\delta$, when $\lambda=\Omega(1)$,

\resizebox{0.89\textwidth}{!}{
\begin{minipage}{\textwidth}
\begin{align} \label{eq:tabular_bound}
        V^{\hat \pi_{\RL}}_{P,c}-   V^{\pie}_{P,c}\leq c_1 H^2\log(|\Scal||\Acal|c_2/\delta) \prns{ \sqrt{\frac{C^{\pi^{*}} |\Scal|^2|\Acal|}{n_o}}+\frac{C^{\pi^{*}}|\Scal||\Acal|}{n_o} },\quad    C^{\pi^{*}}=\max_{(s,a)}\frac{d^{\pi^{*}}_P(s,a)}{\rho(s,a)}.  
\end{align}
\end{minipage}}

where $c_1$ and $c_2$ are some universal constants. 
    
\end{itemize}

\end{theorem}

The quantity $C^{\pie}$ measures the difference of distributions between the expert and the batch data. This is much smaller than the common concentratabiliy coefficients in offline RL: 
\begin{align*}
    \max_{\pi\in \Pi}\max_{(s,a)\in \Scal\times \Acal}\frac{d^{\pi}_P(s,a)}{\rho(s,a)},\quad   \frac{1}{\min_{(s,a)} \rho(s,a)}, 
\end{align*}
which measure the worst discrepancy between all policies in $\Pi$ and the batch data \cite{YinMing2020AEOE}. These assumptions imply $\rho$ has global coverage. We achieve this better bound via pessimism. In the RL case, the similar bound as \pref{eq:tabular_bound} has been obtained in offline policy optimization based on FQI \cite{RashidinejadParia2021BORL}. However, their work is limited to a tabular case. Hereafter, we will show our result is extended to more general continuous MDPs.

\subsection{KNRs}

As in \pref{prop:knr_bonus}, $\sigma(s,a)$ is given by $\beta_{n_o}/\zeta \|\phi(s,a)\|_{\Sigma^{-1}_{n_o}}$. Thus, from \pref{thm:main}, the final error bound of $\hat V^{\hat \pi_{\IL}}_{P,c}-V^{\pie}_{P,c}$ is 
\begin{align*}
   (6H^2+2H) \min(\EE_{(s,a)\sim d^{\pie}_P}[\beta_{n_o}/\zeta \|\phi(s,a)\|_{\Sigma^{-1}_{n_o}}],1)+2H\sqrt{\log(2|\Fcal|/\delta)/(2n_e)}. 
\end{align*}
Hereafter, we analyze $\beta_{n_o}$ and $\EE_{(s,a)\sim d^{\pie}_P}[\|\phi(s,a)\|_{\Sigma^{-1}_{n_o}}]$.

\paragraph{Analysis of information gain}

First, we analyze $\beta_{n_o}$. We need to upper-bound the information gain $\bar \Ical_{n_o}$ in $\beta_{n_o}$. Recall $\Sigma_{\rho}=\EE_{(s,a)\sim \rho}[\phi(s,a) \phi^{\top}(s,a)]$ and $\phi(s,a)\in \RR^d$. 

\begin{theorem}[Finite sample analysis of information gain in finite-dimensional linear models]\label{thm:info_gain_para}
~
\\
Assume $\|\phi(s,a)\|_2\leq 1\,\forall(s,a)\in \Scal\times \Acal$. Let $c_1,c_2$ be universal constants. 
\begin{enumerate}
\item When $\lambda=\Omega(1)$, with probability $1-\delta$, we have
\begin{align*}
     \bar \Ical_{n_o}=\log (\det(\Sigma_{n_o}/\lambda \Ib))& \leq  c_1\mathrm{rank}(\Sigma_{\rho})\{\mathrm{rank}(\Sigma_{\rho})+\log(c_2/\delta)\}\log (1+n_o). 
\end{align*}
\item When $\lambda=\Omega(1)$ and $\zeta^2=\Omega(1)$, With probability $1-\delta$, we have
\begin{align*}
      \beta_{n_o}\leq c_1\sqrt{\|W^{*}\|_2+d_{\Scal}\mathrm{rank}(\Sigma_{\rho})\{\mathrm{rank}(\Sigma_{\rho})+\log(c_2/\delta)\}\log (1+n_o) }. 
\end{align*}
\end{enumerate}

\end{theorem}

Theorem \ref{thm:info_gain_para} states $\bar \Ical_{n_o}=O(\rank[\Sigma_{\rho}]^2\log(n_o))$. We highlight the novelty of our analysis comparing to the other literature. \cite{seeger2008} analyze the expectation of the information gain in a fixed or random design setting. Following their discussion, we can prove $$\EE[\bar \Ical_{n_o}]\leq \rank(\Sigma_{\rho})\log(1+n_o)$$ as \pref{thm:seeger2008} by Jensen's inequality. Going beyond the expectation, we derive the finite-sample result by leveraging the variational representation and the uniform law with localization in \pref{lem:kernel2}. The finite-sample analysis is much harder than calculating the bound of the expectation. 

The worse case of $\bar \Ical_{n_o}$ referred to as the maximum information gain has been often used in online learning \cite{Srinivas2010,Abbasi-yadkori2011,Kakade2020}. From their discussion, we always have $\bar \Ical_{n_o}=O(d\log(n))$. Here, we show that the information gain can be upper-bounded more tightly when $\rank[\Sigma_{\rho}]^2\leq d$ in offline RL (a random design setting). Comparing to the analysis of maximum information gain, our analysis takes the low-rankness of the design matrix $\Sigma_{\rho}$ into consideration by fully utilizing the random design setting assumption. 

\paragraph{Analysis of $\EE_{(s,a)\sim d^{\pie}_P}[\|\phi(s,a)\|_{\Sigma^{-1}_{n_o}}]$ and the final bound } 
~ \\
Next, we analyze $\EE_{(s,a)\sim d^{\pie}_P}[\|\phi(s,a)\|_{\Sigma^{-1}_{n_o}}]$. 

\begin{theorem}\label{thm:finite} Suppose $\lambda=\Omega(1),\zeta^2=\Omega(1),\|W^{*}\|_2=\Omega(1)$.  Let $c_1,c_2$ be some universal constants. 
\begin{enumerate}
    \item With probability $1-\delta$,
    
    \resizebox{0.91\textwidth}{!}{
    \begin{math}
    \begin{aligned}
      \EE_{(s,a)\sim d^{\pie}_P}[\|\phi(s,a)\|_{\Sigma^{-1}_{n_o}}]&\leq c_1\sqrt{\frac{C^{\pie}\mathrm{rank}[\Sigma_{\rho}]\{\mathrm{rank}[\Sigma_{\rho}] +\log(c_2/\delta)\}}{n_o}},\,C^{\pie} = \sup_{x \in \mathbb{R}^d}\left(\frac{x^{\top}\Sigma_{\pie} x}{x^{\top}\Sigma_{\rho}x}\right), 
    \end{aligned}
    \end{math}}
    
     where $\Sigma_{\pie}= \EE_{(s,a)\sim d^{\pie}_P}[\phi(s,a)\phi(s,a)^{\top}]$. 
    \item With probability $1-\delta$,  
    \begin{align}\label{eq:final_linear}
            & V^{\hat \pi_{\IL}}_{P,c}-   V^{\pie}_{P,c}\leq \Error_{o} + \Error_{e}, \bar R=  \mathrm{rank}[\Sigma_{\rho}]\{\mathrm{rank}[\Sigma_{\rho}] +\log(c_2/\delta)\}, \\
            & \Error_{o}= c_1 H^2 \min(d^{1/2},  \bar R)\sqrt{ \bar R }   \sqrt{\frac{d_{\Scal}C^{\pie}\log (1+n_o) }{n_o}}, \nonumber \\
            & \Error_{e}=2H\sqrt{\log(2|\Fcal|/\delta)/(2n_e)}.  \nonumber 
    \end{align}
    \item With probability $1-\delta$, let $C^{\pi^{*}} = \sup_{x \in \mathbb{R}^d}\left(\frac{x^{\top}\Sigma_{\pi^{*}} x}{x^{\top}\Sigma_{\rho}x}\right),\,\Sigma_{\pi^{*}}=\EE_{(s,a)\sim d^{*}_P}[\phi(s,a)\phi^{\top}(s,a)]$. Then, we have 
    \begin{align*}
        V^{\hat \pi_{\RL}}_{P,c}-   V^{\pi^{*}}_{P,c}\leq c_1 H^2\{\mathrm{rank}(\Sigma_{\rho})+\log(c_2/\delta)\}\mathrm{rank}(\Sigma_{\rho})\sqrt{\frac{d_{\Scal}C^{\pi^{*}}\log (1+n_o) }{n_o}}. 
    \end{align*}
\end{enumerate}

\end{theorem}

The final bound  \pref{eq:final_linear} suggests $\Error_{o}$ is $\tilde O(H^2\rank[\Sigma_{\rho}]^2\sqrt{d_sC^{\pie}/n_o})$. We can also get $\tilde O(H^2\rank[\Sigma_{\rho}]d^{1/2}\sqrt{d_sC^{\pie}/n_o})$, which implies $\tilde O(H^2d^{3/2}\sqrt{d_sC^{\pie}/n_o})$. 
In other words, when $C^{\pie},\rank[\Sigma_{\rho}]$ are not so large and the offline sample size $n_o$ is large enough, $O(H\sqrt{\log(|\Fcal|)/n_e})$ is a dominating term and the covariate shift problem in BC can be avoided since the horizon dependence is just $H$. Recall the known BC error bound is $O(H^2\sqrt{\log |\Pi|/n_e})$ \citep[Chapter 14]{agarwal2019reinforcement}.

We see the implication of $\Error_{o}$ in more details, which also corresponds to the error of RL case. The rate regarding $n_o$ is $n^{-1/2}_o$, which is the standard rate in parametric regression. Besides, we can see the bound depends on  $\mathrm{rank}[\Sigma_{\rho}],C^{\pie}$. Importantly, since we always have $\mathrm{rank}[\Sigma_{\rho}]\leq d$, our final bound captures the possible low-rankness of the batch data. The quantity $C^{\pie}$ corresponds to $\pie$-concentrability coefficient ($*$-concentrability in the RL case). This is much smaller than the worst case concentrability coefficients: 
\begin{align*}
  \sup_{\pi\in \Pi} C^{\pi},\quad    \tilde C=\sup_{(s,a)}\|\phi(s,a)\|^2_2\|\Sigma^{-1}_{\rho}\|_2. 
\end{align*}


Finally, we note the technical novelty by comparing it to the techniques developed in the offline RL literature. A quantity that is similar to $\EE_{(s,a)\sim d^{\pie}_P}[\|\phi(s,a)\|_{\Sigma^{-1}_{n_o}}]$ has been analyzed in \cite{JinYing2020IPPE}\footnote{They analyze $\EE_{(s,a)\sim d^{\pi^{*}}_P}[\|\phi(s,a)\|_{\Sigma^{-1}_{n_o}}]$, which also appears in our RL result \cref{thm:finite}. }, which studies the error bound of FQI with pessimism in linear MDPs. \cite[Corollary 4.5]{JinYing2020IPPE} assumes that full coverage, i.e., $\Sigma_{\rho}$ is full-rank and has lower bounded eigenvalues. Also the number of offline samples $n_0$ depends on the smallest eigenvalue. Our analysis just uses partial coverage with the refined concept of relative condition number and thus does not require the full rank assumption on $\Sigma_\rho$. Moreover, our bound is distribution dependent, i.e., it depends on $\mathrm{rank}[\Sigma_{\rho}]$ rather than the ambient dimension of the feature vector $\phi$.  Thus the bound is much tighter for benign cases where the offline data from $\rho$ happens to concentrate on a low-dimensional subspace. Beyond model-based offline RL literature, one can potentially adapt the model-free offline policy evaluation results (e.g., \cite{DuanYaqi2020MOEw,Wang2020}) with linear function approximation to offline policy optimization (without pessimism). Such model-free results  will also incur $ \sup_{\pi\in \Pi} C^{\pi},\tilde C$ and the ambient dimension $d$, instead of much more refined quantities $C^{\pie}$ and $\rank[\Sigma_{\rho}]$.


\subsection{Gaussian processes}
\label{app:GP_proof}
In this section, we give details on GPs. Note that prior works on model-free and model-based offline IL do not have results for infinite-dimensional non-parametric models. Thus our techniques developed in this section are new and relevant even to the offline RL literature---a point that we will return to at the end of this section. 

From \pref{thm:main}, the final error is  
\begin{align*}
   (6H^2+2H) \min(\EE_{x\sim d^{\pie}_P}[\beta_{n_o}/\zeta \sqrt{k_{n_o}(x,x)} ],1)+2H\sqrt{\log(2|\Fcal|/\delta)/(2n_e)}. 
\end{align*}
where $x=(s,a)$. Hereafter, we analyze $\beta_{n_o}$ and $\EE_{x\sim d^{\pie}_P}[\sqrt{k_{n_o}(x,x)}]$. Before going into the details, we repeat several important notations below. 

In this section, following \cite{Srinivas2010}, for simplicity, we suppose the following: 
\begin{assum}\label{assum:kernel_mercer_ape}
$k(x,x)\leq 1,\forall x\in \Scal\times \Acal$. $k(\cdot,\cdot)$ is a continuous and positive semidefinite kernel. $\Scal\times \Acal$ is a compact space. 
\end{assum} 
Recall we denote $x := (s,a)$ and  we have orthonormal eigenfunctions and eigenvalues $\{\psi_i, \mu_i\}_{i=1}^{\infty}$ by Mercer's theorem. We denote the feature mapping $\phi(x) := [ \sqrt{\mu}_1 \psi_1(x) ,\dots, \sqrt{\mu_{\infty}} \psi_{\infty}(x) ]^{\top}$.

Assume eigenvalues $\{\mu_1,\dots, \mu_\infty\}$ is in non-increasing order, we recall the effective dimension:
     $$d^{*}  =   \min\{j \in \mathbb{N}: j\geq B(j+1)n_o/\zeta^2\},\, B(j)=\sum_{k=j}^{\infty} \mu_k.$$
     

We also introduce the empirical version of $d^\star$, where $\hat{\mu}_i$ are eigenvalues of the gram matrix $\mathbf{K}_{n_o}$.
\begin{definition}[Empirical effective dimension]
     $    \hat d  =   \min\{j \in \mathbb{N}: j\geq B(j+1)/\zeta^2,\,\hat B(j)=\sum_{k=j}^{n_o} \hat \mu_k.$
\end{definition}
Hereafter, for simplicity, we treat $\zeta^2=1$, that is, $\zeta^2=\Omega(1)$. Then, since $n_o\leq B(n_o+1)n_o/\zeta^2$, we have $d^{*} \leq n_o$. 

The effective dimensions $\hat{d}$ and  $d^{*}$ are widely used in machine learning literature. The first quantity $d^{*}$ is often referred to as the degree of freedom \cite{ZhangTong2005LBfK,BachFrancis2017OtEb}. In finite-dimensional linear kernels $\{x\mapsto a^{\top}\phi(x),a\in \mathbb{R}^d\}$ ($k(x,x)=\phi^{\top}(x)\phi(x)$), $d^{*}$ is $\rank[\EE_{x\sim \rho}[\phi(x)\phi^{\top}(x)]]$. Thus, $d^{*}$ is considered to be a natural extension of $\rank[\EE_{x\sim \rho}[\phi(x)\phi^{\top}(x)]]$ to infinite-dimensional models. The worst case of the second quantity:$$\max_{\{x_1 \in \Scal\times \Acal,\cdots,x_{n_o}\in \Scal\times \Acal \}}\hat d$$ is often used in online learning literature \cite{Valko2013,Chiappa2020}. Up to logarithmic factors, it is equal to the maximum information gain \cite{Srinivas2010}: 
\begin{align*}
    \max_{\{x_1 \in \Scal\times \Acal,\cdots,x_{n_o}\in \Scal\times \Acal \}}\log \det(\Ib+ \Kb_{n_o}). 
\end{align*}
as shown in \cite{Calandriello19,Valko2013}. Importantly, as we will see soon since our setting is offline (a random design setting), $\hat d$ can be upper-bounded much tightly than their analysis.  

\paragraph{Analysis of information gain}

With the above in mind, we first analyze $\beta_{n_o}$. To do that, we need to bound the information gain $\Ical_{n_o}$. From \cite[Leemma 1]{seeger2008}, we can easily prove 
\begin{align*}
    \EE[\Ical_{n_o} ]\leq \log(1+n_o)d^{*}. 
\end{align*}
as in \pref{thm:seeger2008_kernel}. Going beyond the expectation, we derive the finite-sample error bound.

\begin{theorem}[Finite sample analysis of information gain in infinite-dimensional models]\label{thm:info_gain_nonpara}
~  \\ 
Suppose \pref{assum:kernel_mercer_ape}. Let $c_1$ and $c_2$ be universal constants. 
\begin{enumerate}
    \item We have 
\begin{align}\label{eq:d_analysis}
    \Ical_{n_o}=\log (\det(\Ib+\zeta^{-2}\Kb_{n_o})) & \leq  2\hat d\{\log (1+n_o/\zeta^2)+1\}. 
\end{align}
\item When $\zeta^2=\Omega(1)$, with probability $1-\delta$,  
\begin{align*}
    \Ical_{n_o}=\log (\det(\Ib+\zeta^{-2}\Kb_{n_o})) & \leq  c_1\{ d^{*}+\log(c_2/\delta) \} d^{*}\log (1+n_o). 
\end{align*}
\item  When $\zeta^2=\Omega(1)$, with probability $1-\delta$,  
\begin{align*}
        \beta_{n_o} \leq c_1\sqrt{d_{\Scal}\log^3(c_2d_{\Scal}n_o/\delta) \{ d^{*}+\log(c_2/\delta) \} d^{*}\log (1+n_o) }. 
\end{align*}
\end{enumerate}
\end{theorem}

Theorem \ref{thm:info_gain_nonpara} states $\Ical_{n_o}=O((d^{*})^2\log(n_o))$. Our bound in the offline (a random design) setting can be much tighter compared to the online setting, that is, the known upper bound of maximum information gain in \cite{Srinivas2010} though we can always use this as the bound of $\Ical_{n_o}$ with probability $1$.  We can see this situation in linear kernels as we see in the previous section. In $d$-linear dimensional linear kernels, the maximum information gain is $d$. On the other hand, $\{d^{*}\}^2=\rank[\Sigma_{\rho}]^2$ can be much smaller than $d$. 


\paragraph{Analysis of learning curves and the final bound}

We bound $\EE_{x\sim d^{\pie}_P}[\sqrt{k_{n_o}(x,x)}]$, where 
\begin{align*}
 k_{n_o}(x,x') &=k(x,x')-\bar k_{n_o}(x)^{\top}(\Kb_{n_o}+\zeta^2 \Ib)^{-1}\bar k_{n_o}(x'),\{x_i\}_{i=1}^{n_o}\sim \rho(x). 
\end{align*}
where $x=(s,a)$.

Recall the definition of eigenvalues  $\{\mu_i\}$ and eigenfunctions $\{\psi_i\}$ (which are orthonormal), we define the feature mapping $\phi(x) = [ \sqrt{\mu_1} \psi_1(x), \dots, \sqrt{\mu_\infty} \psi_\infty(x) ]^{\top}$. Denote $\Phi\in\mathbb{R}^{n_o\times \infty}$ as a matrix where each row of $\Phi$ corresponds to $\phi(x_i)$. Since $k(x,x') = \phi(x)^{\top} \phi(x')$, we can rewrite the kernel $k_{n_o}(x,x)$ as follows:
\begin{align*}
k_{n_o}(x,x) & = \phi(x)^{\top}\phi(x) - \phi(x)^{\top} \Phi^{\top} \left( \Phi\Phi^{\top} +\zeta^2 \Ib  \right)^{-1} \Phi \phi(x) \\
& =  \phi(x)^{\top} \left[ \Ib - \Phi^{\top}\left( \Phi\Phi^{\top} + \zeta^2  \Ib \right)^{-1} \Phi \right] \phi(x) \\
& = \phi(x)^{\top} \left( \Ib + \zeta^{-2}  \Phi^{\top} \Phi  \right)^{-1} \phi(x) \\
& = \phi(x)^{\top} \Sigma^{-1}_{n_o} \phi(x),
\end{align*} where $\Sigma_{n_o} \coloneqq \Ib +  \zeta^{-2}  \sum_{i=1}^{n_o} \phi(x_i) \phi(x_i)^{\top}$, and we use matrix inverse lemma in the third equality. Note the infinite-dimensional inverse lemma is formalized in the proof.

Now we can use the relative condition number definition and \pref{lem:distribution_change} for a distribution change, i.e., 
\begin{align*}
&\EE_{x\sim d^{\pie}_P}[\sqrt{k_{n_o}(x,x)}]  \leq \sqrt{\EE_{x\sim d^{\pie}_P}[{k_{n_o}(x,x)}]} \\
&= \sqrt{ \tr\left( \mathbb{E}_{x\sim d^{\pie}_{P}} \phi(x)\phi(x)^{\top} \Sigma_{n_o}    \right)  }  \leq \sqrt{ C^{\pie} \tr\left( \EE_{x\sim \rho} \phi(x)\phi(x)^{\top} \Sigma_{n_o}  \right) } = \sqrt{ C^{\pie} \mathbb{E}_{x\sim \rho} k_{n_o}(x,x)}, 
\end{align*}
where $$C^{\pie}=\sup_{\|x\|_2\leq 1}\frac{x\Sigma_{{\pie}}x}{x\Sigma_{\rho}x},\quad \Sigma_{\pie}=\EE_{x\sim d^{\pie}_P}[\phi(x)\phi(x)^{\top}],\quad \Sigma_{\rho}=\EE_{x\sim \rho}[\phi(x)\phi(x)^{\top}].$$  
Now we only need to focus on analyzing $\EE_{x\sim \rho}[k_{n_o}(x,x)]$.

Before proceeding to the analysis, we introduce the critical radius \cite{bartlett2005}. Given some function class $\Fcal$, consider the localized population Rademacher complexity: 
\begin{align*}
    \Rcal_n(\delta;\Fcal)=\EE\bracks{\sup_{f\in \Fcal,\EE_{x\sim \rho}[f^2(x)]\leq \delta}\left|\frac{1}{n_o}\sum_{i=1}^{n_o} \epsilon_i f(x_i)\right| }
\end{align*}
where $\{x_i\}$ are i.i.d samples following $\rho(x)$ and $\{\epsilon_i\}$ are i.i.d Rademacher variables taking values in $\{-1,+1\}$ equiprobably, independent of the sequence  $\{x_i\}$. The critical radius is defined as the minimum solution to 
\begin{align*}
        \Rcal_n(\xi;\Fcal)\leq \xi^2/b
\end{align*}
w.r.t $\xi$ where $b$ is a value s.t. $\|\Fcal\|_{\infty}\leq b$.  

\begin{theorem}\label{thm:learning_curve}
Suppose \pref{assum:kernel_mercer_ape}. Let $c_1$ and $c_2$ be universal constants. 
\begin{enumerate}
    \item Let $\delta_{n_o}$ be the critical radius of the function class $\{f:f\in \Hcal_k,\|f\|_{k}\leq 1\}$. 
    With probability $1-\delta$, 
  \begin{align*}
             \EE_{x\sim d^{\pie}_P}[\sqrt{k_{n_o}(x,x)}]\leq c_1 \zeta \delta'_{n_o}\sqrt{C^{\pie}d^{*}},
  \end{align*}
  where $\delta'_{n_o}=\delta_{n_o}+\sqrt{\log(c_2/\delta)/n_o}$. 
  \item Assume $\zeta^2=\Omega(1)$. With probability $1-\delta$, 
  \begin{align*}
      \delta_{n_o}\leq c_1\sqrt{d^{*}/n_o},\quad \,                 \EE_{x\sim d^{\pie}_P}[\sqrt{k_{n_o}(x,x)}]\leq c_1\sqrt{\frac{C^{\pie}d^{*}\{d^{*}+\log(c_2/\delta)\}}{n_o}}. 
  \end{align*}
  \item Assume $\zeta^2=\Omega(1)$. With probability $1-\delta$, 
    \begin{align}\label{eq:final_nonparra}
                 V^{\hat \pi_{\IL}}_{P,c}-   V^{\pie}_{P,c}&\leq        \Error_{o}+ \Error_{e}\\ 
                \Error_{o}  &= c_1 H^2\{d^{*}+\log(c_2/\delta)\}d^{*}\sqrt{\frac{d_{\Scal}C^{\pie}\log^3(c_2d_{\Scal}n_o/\delta)\log (1+n_o) }{n_o}}  \nonumber\\
                 \Error_{e} &= 2H\sqrt{\log(2|\Fcal|/\delta)/(2n_e)}.  \nonumber
    \end{align}
 \item Assume $\zeta^2=\Omega(1)$. For offline RL, with probability $1-\delta$, 
 \begin{align*}
                 V^{\hat \pi_{\RL}}_{P,c}-   V^{\pi^{*}}_{P,c}&\leq     c_1 H^2\{d^{*}+\log(c_2/\delta)\}d^{*}\sqrt{\frac{d_{\Scal}C^{\pi^{*}}\log^3(c_2d_{\Scal}n_o/\delta)\log (1+n_o) }{n_o}},
    \end{align*}
    where $C^{\pi^{*}}=\sup_{\|x\|_2\leq 1}\frac{x\Sigma_{{\pi^{*}}}x}{x\Sigma_{\rho}x}$. 
\end{enumerate}

\end{theorem}

The final bound in \pref{eq:final_nonparra} suggests that $\Error_{o}$ is $\tilde O(H^2\{d^{*}\}^2\sqrt{d_{\Scal}C^{\pie}/n_o})$. In other words, when $C^{\pie}$, $d^{*}$ are not so large and the offline sample size is large enough, $\Error_{e}$ dominates  $\Error_{e}$ and the covariate shift problem in BC can be avoided since the horizon dependence is just $H$.  Our bound is the natural extension of \pref{thm:finite} to possibly infinite dimensional models. 

The first and second statements in \pref{thm:learning_curve} are mainly proved in two steps: formulating $k_{n_o}(x,x)$ into the variational representation and utilizing the uniform law with localization. Note the critical radius can be upper-bounded more tightly than $O(\sqrt{d^{*}/n_o})$ depending on the kernels. Besides, $C^{\pie}$ can be replaced with a tighter quantity: 
\begin{align*}
    \max_{i\in \NN}\EE_{(s,a)\sim d^{\pie}_P}[\psi^2_i(s,a)]. 
\end{align*}
Since $\EE_{(s,a)\sim \rho}[\psi^2_i(s,a)]=1$, this quantity also measure the difference of batch data and expert. This is less than $C^{\pie}$ noting that $\frac{x\Sigma_{{\pi^{*}}}x}{x\Sigma_{\rho}x}=\EE_{(s,a)\sim d^{\pie}_P}[\psi^2_i(s,a)]$ when $x$ is a vector s.t. only $i$-th element is $1$ and the other elements are $0$. The third statement in \pref{thm:learning_curve} is directly proved by combining the second statement in \pref{thm:learning_curve}  and \pref{thm:info_gain_nonpara}.  

\paragraph{Implication to offline RL}  The final statement in \pref{thm:learning_curve} is the bound for the RL case. This is the first result showing the error bound for pessimistic offline RL with nonparametric models. As related literature, in model-free offline RL, \cite{UeharaMasatoshi2021FSAo,DuanYaqi2021RBaR} obtained the finite-sample error bounds characterized by the critical radius for some minimax-type estimators called Modified RBM \cite{antos2008learning}. As we show in \pref{thm:learning_curve}, since the critical radius of an RKHS ball is upper-bounded by the effective dimension $d^{*}$, their bounds are also characterized by the effective dimension. On top of that, several papers derived the bounds under the general function approximation setting: FQI \cite{FanJianqing2019ATAo,DuanYaqi2021RBaR,munos2008finite,ChenJinglin2019ICiB}, marginal weighting based estimators \cite{UeharaMasatoshi2019MWaQ}, DICE methods \cite{zhang2019gendice,ChowYinlam2019DBEo}, policy based methods \citep{LiaoPeng2020BPLi,Liu2020} and MABO \citep{XieTengyang2020QASf}. Comparing to our result, all of their bounds depend on 
\begin{align*}
   \sup_{\pi\in \Pi}\sup_{(s,a)}\frac{d^{\pi}_P(s,a)}{\rho(s,a)}\quad \text{or}\quad \sup_{(s,a)}\frac{1}{\rho(s,a)}. 
\end{align*}
The pessimistic bonus allows us to obtain the bound only depending on $C^{\pi^{*}}$ but not the above constants.  
Besides, our $C^{\pi^{*}}$ in \pref{thm:learning_curve} is more refined quantity than the density ratios in the sense that it is defined in terms of the relative condition number. Note we can easily obtain the statements which replace $C^{\pi^{*}}$ in \pref{thm:learning_curve} with $\frac{d^{\pi^{*}}_P(s,a)}{\rho(s,a)}$.

\begin{remark}[Relation with more general offline RL literature]
Due to the lack of exploration, it is known how to deal with the lack of the coverage of the offline data is a challenging problem \citep{ZanetteAndrea2020ELBf,Wang2020}. We use the penalty terms based on model-based RL. In the above, we explain how the penalty term in \alg\,(and its RL counterpart) is transferred to the final sample-error bounds. The idea of penalization has been utilized in a variety of other ways in offline RL. The first other way is imposing constraints on the policy class or Q-function class so that estimated policies are not too much far away from behavior policies. For example, we can use KL divegences, MMD distance, Wasserstein distance to measure the distance from behavior policies \citep{WuYifan2019BROR,FakoorRasool2021CDCB,MatsushimaTatsuya2020DRLv,TouatiAhmed2020SPOv,pmlr-v97-fujimoto19a} and add $D(\pi,\pi_b)$ as penalty terms, where $\pi_b$ is a behavior policy.   
Another way is explicitly estimating the lower bound of q-functions \citep{kumar2020conservative,YuTianhe2021CCOM,Yu2020}. By doing so, we can avoid the overestimation of the q-functions in unknown (non-covered) regions.
\end{remark}

\begin{remark}[Relation with GP literature]
The quantity $\EE_{x\sim \rho(x)}[k_{n_o}(x,x)]$ is often referred to as the learning curve in GP literature \cite{WilliamsChristopherK.I2000UaLB,SollichPeter2002LCfG,Rasmussen2005}. Their analysis mainly focuses on the numerical viewpoints, that is, how to approximately calculate $\EE_{x\sim \rho(x)}[k_{n_o}(x,x)]$. Though \cite{LeGratietLoic2015Aaot} analyzes the convergence property, their analysis is limited to the expectation and the result is asymptotic. As far as we know, our result is the first result showing the finite-sample error rate.  
\end{remark}

\begin{remark}[Duality between KNRs and GPs]\label{rem:duality}
KNRs and GPs have a primal and dual relationship via Mercer's theorem. In fact, as we see, $k(\cdot,\cdot)=\langle \phi(\cdot),\phi(\cdot) \rangle$, we have $k_{n_o}(x,x)=\phi(x)^{\top} \Sigma^{-1}_{n_o} \phi(x)$. Thus, our result in GPs can be applied to the result for infinite-dimensional KNRs with $\phi:\Scal\times\Acal\mapsto \Hcal$ where $\Hcal$ is some RKHS.
\end{remark}

\begin{remark}[Online RL using RKHS] There are several online RL literature using RKHS such as the model-based way \cite{Calandriello19} like our work and the model-free way \cite{agarwal2020pc,2020Yang,du2021bilinear}. In both cases, the final-sample error bounds incur the maximum information gain, i.e., a worse case quantity which is distribution independent.  Comparing to that, our final bounds use distribution-dependent quantities $d^{*}$. 
\end{remark}

\subsection{Missing Proofs }\label{sec:error_proof}

Below, we provide missing proofs for tabular MDPs, KNRs, and non-parametric GP models. 

\subsubsection{Missing proofs for tabular result}

We start by providing proof of the tabular MDP result. 

\begin{proof}[Proof of \pref{thm:tabular_bound}]

We use \pref{thm:main}. Then, we have
\begin{align*}
         V^{\hat \pi_{\IL}}_{P,c}-   V^{\pie}_{P,c}\leq (6H^2+2H)\min(1,\EE_{(s,a)\sim d^{\pie}_P}[\sigma(s,a)])+H\epsilon_{\stat}. 
\end{align*}
 Hereafter, we show how to upper-bound $\EE_{(s,a)\sim d^{\pie}_P}[\sigma(s,a)]$. We use \pref{lem:tabular_lemma}. Then, by letting $\xi=c_1\log(|\Scal||\Acal|c_2/\delta)$, with probability $1-\delta$, we have 
 \begin{align*}
    \frac{1}{N(s,a)+\lambda}\leq \frac{\xi}{n_o\rho(s,a)+\lambda}\quad \forall (s,a)\in \Scal\times \Acal.  
 \end{align*}
 We condition on the above event. 
 Then, 
\begin{align*}
    \EE_{(s,a)\sim d^{\pie}_P}[\sigma(s,a)]&\leq     \EE_{(s,a)\sim d^{\pie}_P}\bracks{\sqrt{{\frac{|\Scal|\log 2+\log (2|\Scal||\Acal|/\delta)}{2\{N(s,a)+\lambda\}} }}+\frac{\lambda}{N(s,a)+\lambda} } \\
    &\leq  \sqrt{\EE_{(s,a)\sim d^{\pie}_P}\bracks{\frac{|\Scal|\log 2+\log (2|\Scal||\Acal|/\delta)}{2\{N(s,a)+\lambda\}}}}+\EE_{(s,a)\sim d^{\pie}_P}\bracks{\frac{\lambda}{N(s,a)+\lambda}}. 
\end{align*}  
From \pref{lem:tabular_lemma}, we have

\resizebox{\textwidth}{!}{
\begin{math}
\begin{aligned}
  \EE_{(s,a)\sim d^{\pie}_P}[\sigma(s,a)]  & \leq \sqrt{\xi\EE_{(s,a)\sim d^{\pie}_P}\bracks{\frac{|\Scal|\log 2+\log (2|\Scal||\Acal|/\delta)}{\{n_o\rho(s,a)+\lambda\}}}}+\EE_{(s,a)\sim d^{\pie}_P}\bracks{\frac{\lambda \xi}{n_o\rho(s,a)+\lambda}} \\ 
    &\leq  \sqrt{\xi C^{\pie} \EE_{(s,a)\sim \rho}\bracks{\frac{|\Scal|\log 2+\log (2|\Scal||\Acal|/\delta)}{\{n_o\rho(s,a)+\lambda\}}}}+C^{\pie} \EE_{(s,a)\sim \rho}\bracks{\frac{ \lambda \xi}{n_o\rho(s,a)+\lambda}} \\ 
    &\leq  \sqrt{\xi C^{\pie} \sum_{s,a}\bracks{\frac{\{|\Scal|\log 2+\log (2|\Scal||\Acal|/\delta)\}\rho(s,a)}{\{n_o\rho(s,a)+\lambda\}}}}+C^{\pie} \sum_{s,a} \bracks{\frac{ \rho(s,a)\lambda \xi}{n_o\rho(s,a)+\lambda}} \\ 
    &\leq \sqrt{\xi C^{\pie}\{|\Scal|\log 2+\log (2|\Scal||\Acal|/\delta) \}|S||A|/n_o}+\lambda C^{\pie}\xi|S||A|/n_o.
\end{aligned}
\end{math}}
where  again
\begin{align*}
    C^{\pie}=\max_{(s,a)}\frac{d^{\pie}_P(s,a)}{\rho(s,a)}. 
\end{align*}
This concludes the proof. \end{proof}

\subsubsection{Missing proofs for KNR results}

Next we move to provide proofs for the KNR results. 

\begin{proof}[Proof of \pref{thm:info_gain_para}]
In the proof, we use two statements, \cref{eq:linear_con} and  \cref{eq:linear_expectation}, in the proof of \cref{thm:finite}. We recommend readers to read the proof of \cref{thm:finite} first. 

We denote the eigenvalues of $\sum_{i=1}^{n_o} \phi(s_i,a_i)\phi^{\top}(s_i,a_i)$ by $\{\hat \mu_i\}_{i=1}^d$ s.t. $\hat \mu_1\geq \hat \mu_2\geq \cdots$. Since we assume $\|\phi(s,a)\|_2\leq 1$, we have $\hat \mu_1\leq n_o$. 

\paragraph{First step}

We first show  
\begin{align*}
   \log (\det(\Sigma_{n_o})/\det(\lambda \Ib)) & \leq \tr\bracks{\Sigma^{-1}_{n_o}\sum_{i=1}^{n_o}\phi(s_i,a_i)\phi^{\top}(s_i,a_i)}\{\log (1+n_o/\lambda)+1\}.
\end{align*}
Note this directly shows $\log (\det(\Sigma_{n_o})/\det(\lambda \Ib))\leq d\log (1+n_o/\lambda),\phi(s,a)\in \RR^d$. The above is proved as follows: 
\begin{align*}
    \log (\det(\Sigma_{n_o})/\det(\lambda \Ib))&=\sum_{i=1}^d \log \left(1+\frac{\hat \mu_i}{\lambda}\right)=\sum_{i=1}^d \log \left(1+\frac{\hat \mu_i}{\lambda}\right)\frac{\hat \mu_i/\lambda+1}{\hat \mu_i/\lambda+1}\\
        &=\sum_{i=1}^d \log \left(1+\frac{\hat \mu_i}{\lambda}\right)\frac{\hat \mu_i/\lambda}{\hat \mu_i/\lambda+1}+\log \left(1+\frac{\hat \mu_i}{\lambda}\right)\frac{1}{\hat \mu_i/\lambda+1}\\
      &\leq \log \left(1+\frac{\hat \mu_1}{\lambda}\right)\sum_{i=1}^d  \frac{\hat \mu_i/\lambda}{\hat \mu_i/\lambda+1}+\sum_{i=1}^d  \frac{\hat \mu_i/\lambda}{\hat \mu_i/\lambda+1} \tag{$\log(1+x)<x$}\\
      &\leq \{\log(1+n_o/\lambda)+1\} \sum_{i=1}^d \frac{\hat \mu_i/\lambda}{\hat \mu_i/\lambda+1} \tag{$\hat \mu_1\leq n_o$} \\
      &= \{\log(1+n_o/\lambda)+1\}\tr\bracks{\Sigma^{-1}_{n_o}\sum_{i=1}^{n_o} \phi(s_i,a_i)\phi^{\top}(s_i,a_i) }. 
\end{align*}
In the last line, letting $UVU^{\top}$ be the eigendecomopsition of $\sum_{i=1}^{n_o} \phi(s_i,a_i)\phi^{\top}(s_i,a_i)$, we use 
\begin{align*}
   &\tr\bracks{\Sigma^{-1}_{n_o}\sum_{i=1}^{n_o} \phi(s_i,a_i)\phi^{\top}(s_i,a_i) }=\tr\bracks{\{V+\lambda \Ib\}^{-1}V }=  \sum_{i=1}^d  \frac{\hat \mu_i/\lambda}{\hat \mu_i/\lambda+1}. 
\end{align*}
Then, the first statement is proved.

\paragraph{Second step}

Next, we prove the second statement. We have 
\begin{align*}
    \tr\bracks{\Sigma^{-1}_{n_o}\sum_{i=1}^{n_o} \phi(s_i,a_i)\phi^{\top}(s_i,a_i) }= \sum_{i=1}^{n_o} \phi^{\top}(s_i,a_i)\Sigma^{-1}_{n_o}\phi(s_i,a_i). 
\end{align*}
Then, from \pref{eq:linear_con}, with probability $1-\delta$, 
\begin{align}\label{eq:info_con}
    \sum_{i=1}^{n_o} \phi^{\top}(s_i,a_i) \Sigma^{-1}_{n_o} \phi(s_i,a_i)\lesssim   c_1\{\rank[\Sigma_{\rho}]+\log(c_2/\delta) \}\sum_{i=1}^{n_o} \phi^{\top}(s_i,a_i)\{n_o\Sigma_{\rho}+\lambda \Ib \}^{-1}\phi(s_i,a_i). 
\end{align}
Hereafter, we condition on the above event. To upper-bound $\sum_{i=1}^{n_o}\|\phi(s_i,a_i)\|^2_{\{n_o\Sigma_{\rho}+\lambda \Ib \}^{-1} }$, we use Bernstein's inequality:  
\begin{align*}
   &\left| \sum_{i=1}^{n_o}  \phi^{\top}(s_i,a_i)\{n_o\Sigma_{\rho}+\lambda \Ib \}^{-1}\phi(s_i,a_i)-n_o\EE_{(s,a)\sim \rho}[ \phi^{\top}(s,a)\{n_o\Sigma_{\rho}+\lambda \Ib \}^{-1}\phi(s,a) ]\right|\\
   &\lesssim \sqrt{n_o\Var_{(s,a)\sim \rho}[ \phi^{\top}(s,a)\{n_o\Sigma_{\rho}+\lambda \Ib \}^{-1}\phi(s,a)]}+1/\lambda. 
\end{align*}
since $\|\phi(s,a)\|^2_{\{n_o\Sigma_{\rho}+\lambda \Ib \}^{-1} }\leq 1/\lambda\,\forall (s,a)\in \Scal\times \Acal$. Here, from \pref{eq:linear_expectation}, 
$$n_o\EE_{(s,a)\sim \rho}[ \phi^{\top}(s,a)\{n_o\Sigma_{\rho}+\lambda \Ib \}^{-1}\phi(s,a) ]\leq \mathrm{rank}[\Sigma_{\rho}].$$ Besides, 
\begin{align*}
    \Var_{(s,a)\sim \rho}[ \phi^{\top}(s,a)\{n_o\Sigma_{\rho}+\lambda \Ib \}^{-1}\phi(s,a)]&\leq \EE_{(s,a)\sim \rho}[\{\phi^{\top}(s,a)\{n_o\Sigma_{\rho}+\lambda \Ib \}^{-1}\phi(s,a)\}^2]\\
    & \leq 1/\lambda\EE_{(s,a)\sim \rho}[\phi^{\top}(s,a)\{n_o\Sigma_{\rho}+\lambda \Ib \}^{-1}\phi(s,a)]\\
    &\leq \mathrm{rank}[\Sigma_{\rho}]/(n_o\lambda).  \tag{from \pref{eq:linear_expectation} }
\end{align*}
Thus, 
\begin{align*}
    \sum_{i=1}^{n_o} \phi^{\top}(s_i,a_i)\{n_o\Sigma_{\rho}+\lambda \Ib \}^{-1}\phi(s_i,a_i)\lesssim \mathrm{rank}[\Sigma_{\rho}]+\sqrt{\rank[\Sigma_{\rho}]/\lambda}+1/\lambda. 
\end{align*}

By combining \pref{eq:info_con} with the above, we have  
\begin{align*}
     \log (\det(\Sigma_{n_o})/ \det(\lambda \Ib))& \leq  c_1\mathrm{rank}(\Sigma_{\rho})\{\mathrm{rank}(\Sigma_{\rho})+\log(c_2/\delta)\}\log (1+n_o c_3). 
\end{align*}
from $\lambda=\Omega(1)$. 



\end{proof}

Before proving \pref{thm:finite}, we first present some lemmas. 
\begin{lemma}[Distribution change]\label{lem:distribution_change} Consider two distributions $\rho_1\in\Delta(\Scal\times\Acal)$ and $\rho_2 \in \Delta(\Scal\times\Acal)$, and a feature mapping $\phi:\Scal\times\Acal\mapsto \Hcal$ where $\Hcal$ is some Hilbert space (e.g., finite dimensional Euclidean space). Denote $C := \sup_{x\in\Hcal} \frac{ x^{\top} \mathbb{E}_{s,a\sim \rho_1} \phi(s,a)\phi(s,a)^{\top}  x }{ x^{\top}\mathbb{E}_{s,a\sim \rho_2} \phi(s,a)\phi(s,a)^{\top}  x  }$. Then for any positive definition linear matrix ( operator $\Lambda$), we have:
\begin{align*}
\mathbb{E}_{s,a\sim \rho_1} \phi(s,a)^{\top}  \Lambda \phi(s,a) \leq C \mathbb{E}_{s,a\sim \rho_2} \phi(s,a)^{\top}  \Lambda \phi(s,a).
\end{align*}\
\end{lemma}
\begin{proof}
Denote the eigendecomposition of $\Lambda = U \Sigma U^{\top}$ where $\{\sigma_i, u_i\}$ as the eigenvalue-eigenvector pairs. We have:
\begin{align*}
\mathbb{E}_{s,a\sim \rho_1} \phi(s,a)^{\top}  \Lambda \phi(s,a)&  = \sum_{i=0}^{\infty} \sigma_i  u_i^{\top} \mathbb{E}_{s,a\sim \rho_1} \phi(s,a)\phi(s,a)^{\top} u_i \\
& \leq  \sum_{i=0}^{\infty} \sigma_i C u_i^{\top} \mathbb{E}_{s,a\sim \rho_2} \phi(s,a)\phi(s,a)^{\top} u_i \\
& = C \mathbb{E}_{s,a\sim \rho_2} \phi(s,a)^{\top} \Lambda \phi(s,a),
\end{align*} which concludes the proof.
\end{proof}

\begin{proof}[Proof of \pref{thm:finite}]
~
Here, we prove the first statement. We need to upper-bound 
\begin{align*}
  \EE_{(s,a)\sim d^{\pie}_P}\bracks{\sqrt{\phi^{\top}(s,a)\Sigma_{n_o}^{-1}\phi(s,a)}}. 
\end{align*}
As the first step, we use Jensen's inequality: 
\begin{align*}
    \EE_{(s,a)\sim d^{\pie}_P}\bracks{\sqrt{\phi^{\top}(s,a)\Sigma_{n_o}^{-1}\phi(s,a)}}\leq \sqrt{\EE_{(s,a)\sim d^{\pie}_P}\bracks{\phi^{\top}(s,a)\Sigma_{n_o}^{-1}\phi(s,a)}}. 
\end{align*}
Hereafter, we analyze $\EE_{(s,a)\sim d^{\pie}_P}\bracks{\phi^{\top}(s,a)\Sigma_{n_o}^{-1}\phi(s,a)}$. 

We first use the definition of the relative condition number $C^{\pie}$ and \pref{lem:distribution_change} to change distribution from $d^{\pie}_{P}$ to $\rho$, i.e., via \pref{lem:distribution_change}, we have:
\begin{align*}
\mathbb{E}_{s,a\sim d^{\pie}_{P}} \phi(s,a)^{\top} \Sigma_{n_o}^{-1} \phi(s,a) \leq C^{\pie} \mathbb{E}_{s,a\sim \rho} \phi(s,a)^{\top} \Sigma_{n_o}^{-1} \phi(s,a).
\end{align*}
Thus, below we just need to bound $\mathbb{E}_{s,a\sim \rho} \phi(s,a)^{\top} \Sigma_{n_o}^{-1} \phi(s,a)$.

\paragraph{Concentration argument}

In this step, we consider how to bound $\EE_{(s,a)\sim \rho}[\phi^{\top}(s,a)\Sigma^{-1}_{n_o}\phi(s,a)]$. To do that, we show with probability $1-\delta$,

\resizebox{0.97\textwidth}{!}{
\begin{minipage}{\textwidth}
\begin{align}\label{eq:linear_con}
   \phi^{\top}(s,a)\Sigma^{-1}_{n_o}\phi(s,a)\leq  c_1\{\rank(\Sigma_{\rho})+\log(c_2/\delta) \} \phi^{\top}(s,a)\{n_o\Sigma_{\rho}+\lambda \Ib\}^{-1} \phi(s,a)\quad \forall (s,a)\in \Scal\times \Acal. 
\end{align}
\end{minipage}}

We use the variational representation: 
\begin{align*}
    \phi^{\top}(s,a)\Sigma^{-1}_{n_o}\phi(s,a)&=\sup_{\{a\in \RR^d: a^{\top}\Sigma_{n_o}  a\leq 1\}}\{a^{\top}\phi(s,a)\}^2\\
    &=\sup_{\{a\in \RR^d:a^{\top}\Sigma_{n_o}  a\leq 1,\|a\|^2_2\leq (1+\lambda)/\lambda,   \| a^{\top}\phi\|_{\infty}\leq 1/\lambda\}}\{a^{\top}\phi(s,a)\}^2.  
\end{align*}
Note that in the first line, we use 
\begin{align*}
    \sup_{\{ a\in \RR^d:a^{\top}\Sigma_{n_o}  a\leq 1\}}a^{\top}\phi(s,a)= \sup_{\{b\in \RR^d: b^{\top}b\leq 1\}}b^{\top}\Sigma^{-1/2}_{n_o}\phi(s,a)=\|\phi(s,a)\|_{\Sigma^{-1}_{n_o}}. 
\end{align*}
From the first line to the second line, we use the fact that the maximization regarding $a$ is taken when $\tilde a=\Sigma^{-1}_{n_o}\phi(s,a)/\|\phi(s,a)\|_{\Sigma^{-1}_{n_o}}$ and 
\begin{align*}
    \|\tilde a\|^2_2 &=\phi^{\top}(s,a)\Sigma^{-2}_{n_o}\phi(s,a)/\{ \phi^{\top}(s,a)\Sigma^{-1}_{n_o}\phi(s,a)\}=(n_o+\lambda)/\lambda^2,\,\\
    |\tilde a^{\top}\phi| & \leq  \|\phi(s,a)\|_{\Sigma^{-1}_{n_o}}\leq 1/\lambda\quad  \forall (s,a)\in \Scal\times \Acal, 
\end{align*}
noting $\|\phi(s,a)\|_2\leq 1$. By defining $\bar c=(n_o+\lambda)/\lambda^2$, we have $\forall(s,a)\in \Scal\times \Acal$, 
\begin{align*}
       \phi^{\top}(s,a)\Sigma^{-1}_{n_o}\phi(s,a)= &\sup_{\{a\in \RR^d:a^{\top}\Sigma_{n_o}  a\leq 1,\|a\|^2_2\leq \bar c,\|a^{\top}\phi\|_{\infty}\leq 1/\lambda\}}\{a^{\top}\phi(s,a)\}^2 \\
    &=\sup_{\{a\in \RR^d:a^{\top}\lambda\Ib a+\sum_{i=1}^{n_o} \{a^{\top}\phi_i\}^2 \leq 1,\|a\|^2_2\leq \bar c,\|a^{\top}\phi\|_{\infty}\leq 1/\lambda\}}\{a^{\top}\phi(s,a)\}^2. 
\end{align*}  

Next, we use \pref{lem:kernel3}, that is, with probability $1-\delta$, 
\begin{align*}
    \frac{1}{n_o}\sum_{i=1}^{n_o}  f^2(s_i,a_i) \geq 0.5 \EE_{(s,a)\sim \rho}[f^2(s,a)]-0.5\{\delta'_{n_o}\}^2\,\forall f\in \Fcal,
\end{align*}
where
\begin{align*}
         \Fcal=\{(s,a)\mapsto a^{\top}\phi(s,a):a^{\top}\Sigma_{n_o}  a\leq 1,\|a\|^2_2\leq \bar c,\|a^{\top}\phi\|_{\infty}\leq 1/\lambda,a\in \RR^d\}. 
\end{align*}
Here, $\delta'_{n_o}=\delta_{n_o}+\sqrt{\log(c_2/\delta)/n_o}$, where $\delta_{n_o}$ is the critical radius of the function class $\Fcal$. Noting $\lambda=\Omega(1)$, from \pref{lem:critical_para}, $\delta'_{n_o}=c_1\sqrt{\mathrm{rank}[\Sigma_{\rho}]/n_o} +\sqrt{\log(c_2/\delta)/n_o}$. By conditioning on the above event, $\forall (s,a)\in \Scal \times \Acal$, we have 
\begin{align*}
 \|\phi(s,a)\|^2_{\Sigma^{-1}_{n_o} }   & \leq \sup_{\{a\in \RR^d:a^{\top}\lambda\Ib a+0.5n_o\EE_{(s,a)\sim \rho}[\{a^{\top}\phi\}^2] \leq 1+0.5 n_o\delta'^2_{n_o},\|a\|^2_2\leq \bar c, \|a^{\top}\phi\|_{\infty}\leq 1/\lambda \}}\{a^{\top}\phi(s,a)\}^2 \\ 
   &\leq \sup_{\{a\in \RR^d:a^{\top}\{n_o\Sigma_{\rho}+\lambda \Ib\} a\leq 2+n_o\delta'^2_{n_o},\|a\|^2_2\leq \bar c,\|a^{\top}\phi\|_{\infty}<1/\lambda\}}\{a^{\top}\phi(s,a)\}^2\\ 
    &\leq \sup_{\{a\in \RR^d: a^{\top}\{n_o\Sigma_{\rho}+\lambda \Ib\} a\leq 2+n_o\delta'^2_{n_o}\}}\{a^{\top}\phi(s,a)\}^2 \\ 
    &=  (2+n_o\delta'^2_{n_o}) \phi^{\top}(s,a)\{n_o\Sigma_{\rho}+\lambda \Ib\}^{-1} \phi(s,a)\\
    &\leq c_1\{\mathrm{rank}[\Sigma_{\rho}] + \log(c_2/\delta)\} \phi^{\top}(s,a)\{n_o\Sigma_{\rho}+\lambda \Ib\}^{-1} \phi(s,a). 
\end{align*}

\paragraph{Last step}

Then, the final bound is

\resizebox{\textwidth}{!}{
\begin{math}
\begin{aligned}
    \EE_{(s,a)\sim d^{\pie}_P}[\|\phi(s,a)\|_{\Sigma^{-1}_{n_o}}]& = \sqrt{C^{\pie}\EE_{(s,a)\sim \rho}[\phi^{\top}(s,a)\Sigma^{-1}_{n_o}\phi(s,a)]}\\
    &\leq c_1 \sqrt{C^{\pie}n_o\{\mathrm{rank}[\Sigma_{\rho}] + \log(c_2/\delta)\} \EE_{(s,a)\sim \rho}[\phi^{\top}(s,a)\{\Sigma_{\rho}+\lambda \Ib\}^{-1} \phi(s,a)]}. 
\end{aligned}\end{math}}

Let $UVU^{\top}$ be the eigenvalue decomoposition of $\Sigma_{\rho}$ s.t. $V_{i,i}=\mu_i$. We have 
\begin{align}
    \EE_{(s,a)\sim \rho}\bracks{\phi^{\top}(s,a)\{n_o\Sigma_{\rho}+\lambda \Ib\}^{-1}\phi(s,a)}&= \mathrm{Tr}[ \{n_o\Sigma_{\rho}+\lambda \Ib\}^{-1}\{ \Sigma_{\rho} \} ]= \mathrm{Tr}[ \{n_oV+ \lambda \Ib\}^{-1}V ]  \nonumber \\ 
   &=\frac{1}{n_o}\sum_{i=1}^{n_o}\frac{\mu_i}{\mu_i+\lambda/n_o}\leq \frac{\mathrm{rank}[\Sigma_{\rho}]}{n_o}.  \label{eq:linear_expectation}
\end{align}

By combining all things together, with probability $1-\delta$, 
\begin{align*}
      \EE_{(s,a)\sim d^{\pie}_P}[\|\phi(s,a)\|_{\Sigma^{-1}_{n_o}}]&\leq c_1\sqrt{\frac{C^{\pie}\mathrm{rank}[\Sigma_{\rho}]\{\mathrm{rank}[\Sigma_{\rho}] +\log(c_2/\delta)\}}{n_o}}. 
\end{align*}

\end{proof}

\subsubsection{Missing proofs of non-parametric model}

Finally, we provide missing proofs for the non-parametric GP model.

\begin{proof}[Proof of \pref{thm:info_gain_nonpara}]
In the proof, we use two statements, \pref{eq:critical_1} and  \pref{eq:critical_2}, in the proof of \pref{thm:learning_curve}. We recommend readers to read the proof of \pref{thm:learning_curve} first.

We denote the eigenvalues of $\Kb_{n_o}$ by $\{\hat \mu_i\}_{i=1}^{n_o}$ s.t. $\hat \mu_1\geq \hat \mu_2\geq\cdots$. From \pref{assum:kernel_mercer_ape}, we have 
\begin{align*}
    n_o=\tr(\Kb_{n_o})=\sum_{i=1}^{n_o} \hat \mu_i. 
\end{align*}
Thus implies $\hat \mu_1\leq n_o$. Then,

\resizebox{\textwidth}{!}{
\begin{minipage}{\textwidth}
\begin{align*}
 \log (\det(\Ib+\zeta^{-2}\Kb_{n_o}))&=\sum_{i=1}^{n_o} \log \left(1+\frac{\hat \mu_i}{\zeta^2}\right)=\sum_{i=1}^{n_o}  \log \left(1+\frac{\hat \mu_i}{\zeta^2}\right)\frac{\hat \mu_i/\zeta^2+1}{\hat \mu_i/\zeta^2+1}\\
        &=\sum_{i=1}^{n_o}  \log \left(1+\frac{\hat \mu_i}{\zeta^2}\right)\frac{\hat \mu_i/\zeta^2}{\hat \mu_i/\zeta^2+1}+\log \left(1+\frac{\hat \mu_i}{\zeta^2}\right)\frac{1}{\hat \mu_i/\zeta^2+1}\\
        &=\sum_{i=1}^{n_o}  \log \left(1+\frac{\hat \mu_i}{\zeta^2}\right)\frac{\hat \mu_i/\zeta^2}{\hat \mu_i/\zeta^2+1}+\log \left(1+\frac{\hat \mu_i}{\zeta^2}\right)\frac{1}{\hat \mu_i/\zeta^2+1}\\
      &\leq \log \left(1+\frac{\hat \mu_1}{\zeta^2}\right)\sum_{i=1}^{n_o}  \frac{\hat \mu_i/\zeta^2}{\hat \mu_i/\zeta^2+1}+\sum_{i=1}^{n_o}  \frac{\hat \mu_i/\zeta^2}{\hat \mu_i/\zeta^2+1} \tag{$\log(1+x)\leq x$}\\
      &\leq \{\log(1+n_o/\zeta^2)+1\} \sum_{i=1}^{n_o}  \frac{\hat \mu_i/\zeta^2}{\hat \mu_i/\zeta^2+1} \tag{$\hat \mu_1\leq n_o$}\\ 
     &\leq \{\log(1+n_o/\zeta^2)+1\}\min_{j}\{j+ \hat B(j+1)/\zeta^2\}\leq 2\{\log(1+n_o/\zeta^2)+1\}\hat d,
\end{align*}\end{minipage}}

where the last second inequality uses the fact that $\sum_{i=1}^{n_o} \frac{\hat\mu_i / \xi^2}{\hat\mu_i/\xi^2 + 1} \leq j + \sum_{i=j+1}^{n_o} \hat\mu_i / \xi^2$. Then, the first statement is proved. 

Next, we prove the second statement. We use 
\begin{align*}
    \sum_{i=1}^{n_o}  \frac{\hat \mu_i/\zeta^2}{\hat \mu_i/\zeta^2+1}=\frac{1}{\zeta^2} \sum_{i=1}^{n_o}k_{n_o}(x_i,x_i). 
\end{align*}
proved in \pref{lem:kernel_lemma}. Then, from  \pref{eq:critical_1}, with probability $1-\delta$, 
\begin{align*}
   \frac{1}{\zeta^2}  \sum_{i=1}^{n_o} k_{n_o}(x_i,x_i)\lesssim   \delta'^2_{n_o} \sum_{i=1}^{n_o}  \sup_{\{f:\zeta^2/n_o\|f\|^2_k+\EE_{x\sim \rho}[f^2(x)]\leq 1,f\in \Hcal_k\}}f^2(x_i), 
\end{align*}
where $\delta'_n=\delta_n+\sqrt{\log(c_2/\delta)/n_o}$ and $\delta_n$ is the critical radius of $\{f\in \Hcal_k:\|f\|_k\leq 1\}$. Hereafter, we condition on the above event. 

Then, from Bernstein's inequality,

\resizebox{\textwidth}{!}{
\begin{minipage}{\textwidth}
\begin{align*}
   &\left| \braces{\sum_{i=1}^{n_o} \sup_{\{f:\zeta^2/n_o\|f\|^2_k+\EE_{x\sim \rho}[f^2(x)]\leq 1,f\in \Hcal_k\}}f^2(x_i)}
   -n_o\EE_{x\sim \rho}\bracks{\sup_{\{f:\zeta^2/n_o\|f\|^2_k+\EE_{x\sim \rho}[f^2(x)]\leq 1,f\in \Hcal_k\}}f^2(x)}
   \right|\\
   &\lesssim \sqrt{n_o\Var_{x\sim \rho}[\sup_{\{f:\zeta^2/n_o\|f\|^2_k+\EE_{x\sim \rho}[f^2(x)]\leq 1,f\in \Hcal_k\}}f^2(x)]}+n_o. 
\end{align*}\end{minipage}}

We use for $f$ in $\Hcal_k$ s.t. $\|f\|_{k}\leq 1$ 
\begin{align*}
     |f(x)|=|\langle f(\cdot),k(x,\cdot )\rangle_k | \leq \|f\|_k \|k(x,\cdot)\|_k \leq 1. 
\end{align*}
from \cref{assum:kernel_mercer_ape}. Here, from \pref{eq:critical_2}, the expectation is upper-bounded by $$\EE_{x\sim \rho}\bracks{ \sup_{\{f:\zeta^2/n_o\|f\|^2_k+\EE_{x\sim \rho}[f^2(x)]\leq 1,f\in \Hcal_k\}}f^2(x)}\leq d^{*}.$$ Besides, the variance is also upper-bounded by 
\begin{align*}
 &\Var_{x\sim \rho}[\sup_{\{f:\zeta^2/n_o\|f\|^2_k+\EE_{x\sim \rho}[f^2(x)]\leq 1,f\in \Hcal_k\}}f^2(x)]\\
 &\leq \EE_{x\sim \rho}[\sup_{\{f:\zeta^2/n_o\|f\|^2_k+\EE_{x\sim \rho}[f^2(x)]\leq 1,f\in \Hcal_k\}}f^4(x)]\\ 
    & \leq  \EE_{x\sim \rho}[\sup_{\{f:\zeta^2/n_o\|f\|^2_k+\EE_{x\sim \rho}[f^2(x)]\leq 1,f\in \Hcal_k\}}f^2(x)] \tag{$f^2(x)\leq 1\,\forall x\in \Scal \times \Acal$ from \pref{assum:kernel_mercer_ape}}\\
    &= d^{*}. \tag{From \pref{eq:critical_2}}
\end{align*}
Thus, with probability $1-\delta$, 
\begin{align*}
    \sum_{i=1}^{n_o} k_{n_o}(x_i,x_i)  & \lesssim  \{\delta'_{n_o}\}^2 n_o(d^{*}+ \sqrt{d^{*}}+1)\\ 
    &\lesssim   c_1\{d^{*}+\log(c_2/\delta)\}d^{*}. 
\end{align*}
noting $\delta'_{n_o}=\sqrt{d^{*}/n_o}+\sqrt{\log(c_2/\delta)/n_o}$ from \pref{thm:learning_curve}.  

By combining all things together,  with probability $1-\delta$, 
\begin{align*}
 \log (\det(\Ib +\zeta^{-2}\Kb_{n_o}))&\leq \{\log(1+n_o/\zeta^2)+1\} \sum_{i=1}^{n_o}  \frac{\hat \mu_i/\zeta^2}{\hat \mu_i/\zeta^2+1}\\
 &=  \{\log(1+n_o/\zeta^2)+1\}  \frac{1}{\zeta^2}\sum_{i=1}^{n_o} k_{n_o}(x_i,x_i)\\
 &\lesssim  \{\log(1+c_3 n_o)\}  \{d^{*}+\log(c_2/\delta)\}d^{*}. 
\end{align*}
This concludes the proof.


\end{proof}

\begin{proof}[Proof of \pref{thm:learning_curve}]
~

\paragraph{First Statement}
From Jensen's inequality, we have 
\begin{align*}
       \EE_{x\sim d^{\pie}_P}[\sqrt{k_{n_o}(x,x)}]\leq \sqrt{\EE_{x\sim d^{\pie}_P}[k_{n_o}(x,x)]}. 
\end{align*}
Thus, we focus how to bound $\EE_{x\sim d^{\pie}_P}[k_{n_o}(x,x)]$. Before that, we show the following statement. With probability $1-\delta$,  we have for $\forall x \in \Scal \times \Acal$: 
\begin{align}\label{eq:critical_1}
    k_{n_o}(x,x)\leq   c_1 \zeta^2\delta'^2_{n_o}\times \sup_{\{f\in \Hcal_k: \zeta^2/n_o\|f\|^2_k+\EE_{x\sim \rho}[f^2(x)]\leq 1\}}f^2(x), 
\end{align}
where $\delta'_{n_o}=\delta_{n_o}+\sqrt{\log(c_2/\delta)/n_o}$ and $\delta_{n_o} $ is the critical radius of $\{f\in \Hcal_k:\|f\|_k\leq 1\}$. 

As the first step, we use \pref{lem:kernel} and \pref{lem:kernel2}. 
\begin{align} 
k_{n_o}(x,x)&=\sup_{\{f\in \Hcal_{k_{n_o}}\mid  \|f\|^2_{k_{n_o}}\leq 1\}}f^2(x) \tag{From  \pref{lem:kernel}}  \nonumber \\
    &= \sup_{\{f\in \Hcal_k\mid  \|f\|^2_k+\zeta^{-2}\sum_{i=1}^{n_o}f(x_i)^2\leq 1\}}f^2(x) \tag{From  \pref{lem:kernel2}}. 
\end{align} 

Next invoke  \pref{lem:kernel3}, that is, with probability $1-\delta$, 
\begin{align*}
    \frac{1}{n_o}\sum_{i=1}^{n_o}  f^2(x_i) \geq 0.5 \EE_{(s,a)\sim \rho}[f^2(x)]-0.5\{\delta'_{n_o}\}^2\,\forall f\in \Fcal
\end{align*}
where
\begin{align*}
       \Fcal= \{f: f\in \Hcal_k,\|f\|^2_k=1\}.  
\end{align*}
Here, $\delta'_{n_o}=\delta_{n_o}+\sqrt{\log(c_2/\delta)/n_o}$, where $\delta_{n_o}$ is the critical radius of the function class $\Fcal$. Hereafter, we condition on the above event. Note the uniform boundedness assumption of $\Fcal$ for \pref{lem:kernel3} is satisfied noting 
\begin{align*}
    |f(x)|=|\langle f(\cdot),k(\cdot,x) \rangle_k| \leq \|f\|_{k}\|k(\cdot,x)\|_k\leq 1. 
\end{align*}
noting \cref{assum:kernel_mercer_ape}. 
Then, we have 
\begin{align*} 
 k_{n_o}(x,x)   \leq \sup_{\{f\in \Hcal_k\mid \|f\|^2_k+\zeta^{-2}n_o/2\EE_{x\sim \rho}[f^2(x)]\leq 1+n_o\delta'^2_{n_o}/2\}}f^2(x).
\end{align*}
$k_{n_o}(x,x)$ is further upper-bounded by 
\begin{align*}
    k_{n_o}(x,x) &\leq \sup_{\{f\in \Hcal_k: 2\zeta^2/n_o\|f\|^2_k+\EE_{x\sim \rho}[f^2(x)]\leq 2\zeta^2/n_o+\zeta^2\delta'^2_{n_o}\}}f^2(x) \tag{Multiply $2\zeta^2/n_o$}\\
     &\leq (2\zeta^2/n_o+\zeta^2\delta'^2_{n_o})\times \sup_{\{f\in \Hcal_k: \zeta^2/n_o\|f\|^2_k+\EE_{x\sim \rho}[f^2(x)]\leq 1\}}f^2(x)\\
     &\leq c_1 \zeta^2\delta'^2_{n_o}\times \sup_{\{f\in \Hcal_k: \zeta^2/n_o\|f\|^2_k+\EE_{x\sim \rho}[f^2(x)]\leq 1\}}f^2(x).   
\end{align*}
This concludes \pref{eq:critical_1}. 

Next, we show 
\begin{align*}
     \EE_{x\sim  d^{\pie}_P}\bracks{ \sup_{\{f\in \Hcal_k: \zeta^2/n_o\|f\|^2_k+\EE_{x\sim \rho}[f^2(x)]\leq 1\}}f^2(x)}\leq  2d^{*}\times \sup_{\|x\|_2\leq 1}\frac{x^{\top}\Sigma_{{\pie}}x}{x^{\top}\Sigma_{\rho}x} . 
\end{align*}
For $f(\cdot)=a^{\top}\phi(\cdot)$ (recall $\phi(\cdot)$ is the feature mapping defined by the eigenvalues $\mu_i$ and eigenfunctions $\phi$, s.t. $\phi=(\phi_1,\cdots,\phi_{\infty})$), we have 
\begin{align*}
    \|f\|^2_k=a^{\top}a,\quad \EE_{x\sim \rho}[f^2(x)]=a^{\top }Ma.
\end{align*}
where $M$ is a diagonal matrix  in $\RR^{\infty\times \infty}$ s.t. $M_{i,i}=\mu_i$. 
Thus,

\resizebox{\textwidth}{!}{
\begin{minipage}{\textwidth}
\begin{align*}
   \EE_{x\sim d^{\pie}_P}\bracks{\sup_{\{f:\zeta^2/n_o\|f\|^2_k+\EE_{x\sim \rho}\bracks{f^2(x)}\leq 1,f\in \Hcal_k\}}f^2(x)}= \EE_{x\sim  d^{\pie}_P}\bracks{\sup_{\{a\in \RR^{\infty}:a^{\top}(\zeta^2/n_o \Ib + M\}a\leq 1\}}\{a^{\top}\phi(x)\}^2}. 
\end{align*}\end{minipage}}

Then, by letting $\Sigma_{\rho}$ and $\Sigma_{\pie}$ be $\EE_{(s,a)\sim \rho}[\phi(s,a)\phi^{\top}(s,a)]$ and $\EE_{(s,a)\sim d^{\pie}_P}[\phi(s,a)\phi^{\top}(s,a)]$,

\resizebox{\textwidth}{!}{
\begin{minipage}{\textwidth}
\begin{align*}
    \EE_{x\sim  d^{\pie}_P}\bracks{\sup_{\{a\in \RR^d:a^{\top}(\zeta^2/n_o \Ib+ M\}a\leq 1 }\{a^{\top}\phi(x)\}^2}&=\EE_{x\sim  d^{\pie}_P}[\phi(x)\{\zeta^2/n_o \Ib + M \}^{-1}\phi(x)] \\ 
    &=\tr[ \EE_{x\sim  d^{\pie}_P}[\phi(x)\phi(x)^{\top}] \{\zeta^2/n_o \Ib + M \}^{-1} ]\\ 
        &=\tr[ \EE_{x\sim  \rho}[\phi(x)\phi(x)^{\top}] \{\zeta^2/n_o \Ib + M \}^{-1} ]\times \sup_{\|x\|_2\leq 1}\frac{x^{\top}\Sigma_{\pie} x}{x^{\top}\Sigma_{\rho}x} \\ 
    &=\sum_{i=1}^{\infty}\frac{\mu_i}{\zeta^2/n_o+\mu_i}\times \sup_{\|x\|_2\leq 1}\frac{x^{\top}\Sigma_{\pie} x}{x^{\top}\Sigma_{\rho}x}. 
      \end{align*}\end{minipage}}

Then, by defining $C^{\pie}= \sup_{\|x\|_2\leq 1}\frac{x^{\top}\Sigma_{\pie} x}{x^{\top}\Sigma_{\rho}x}$, we have 
\begin{align*}      
     \EE_{x\sim  d^{\pie}_P}\bracks{\sup_{\{a\in \RR^d:a^{\top}(\zeta^2/n_o+ M\}a\leq 1 }\{a^{\top}\phi(x)\}^2} & \leq \min_{j}\{j+n_o/\zeta^2\sum_{i=j+1}^{\infty}\mu_i\}\times C^{\pie}\\
     &\leq \min_{j}\{j+n_o/\zeta^2\sum_{i=j+1}^{\infty}\mu_i\}\times C^{\pie}\\
    &\leq 2d^{*}\times C^{\pie}. 
\end{align*}

By combining all things together (\pref{eq:critical_1} and \pref{eq:critical_2}), the statement is concluded, that is, with probability $1-\delta$: 
\begin{align*}
         \EE_{d^{\pie}_P}[\sqrt{k_{n_o}(x,x)}]&\leq \sqrt{\zeta^2\delta'^2_{n_o}\times     \EE_{x\sim  d^{\pi}_P}[\sup_{\{f\in \Hcal_k: \zeta^2/n_o\|f\|^2_k+\EE_{x\sim \rho}[f^2(x)]\leq 1\}}f^2(x)}]\\
         &\leq \zeta \delta'_{n_o}\sqrt{C^{\pie}d^{*}}. 
\end{align*}
where $C^{\pie}=\sup_{\|x\|_2\leq 1} \frac{x^{\top}\Sigma_{\pie} x}{x^{\top}\Sigma_{\rho}x}$. 
\begin{remark}
Like the above, We can also prove  
\begin{align}\label{eq:critical_2}
     \EE_{x\sim  \rho}[ \sup_{\{f\in \Hcal_k: \zeta^2/n_o\|f\|^2_k+\EE_{x\sim \rho}[f^2(x)]\leq 1\}}f^2(x)]\leq \sum_{i=1}^{\infty}\frac{\mu_i}{\zeta^2/n_o+\mu_i} \leq 2d^{*}.   
\end{align}
This is used in the proof of \cref{thm:info_gain_nonpara}. 
\end{remark}

\begin{remark}
We can also use 
\begin{align*}
    \EE_{x\sim  d^{\pie}_P}\bracks{\sup_{\{a\in \RR^d:a^{\top}(\zeta^2/n_o \Ib+ M\}a\leq 1 }\{a^{\top}\phi(x)\}^2}&=\EE_{x\sim  d^{\pie}_P}[\phi(x)\{\zeta^2/n_o \Ib + M \}^{-1}\phi(x)] \\ 
    &=\tr[ \EE_{x\sim  d^{\pie}_P}[\phi(x)\phi(x)^{\top}] \{\zeta^2/n_o \Ib + M \}^{-1} ]\\ 
    &=\sum_{i=1}^{\infty}\frac{\EE_{x\sim  d^{\pie}_P}[\phi_i(x)\phi_i(x)^{\top}]\}}{\zeta^2/n_o+\mu_i}\\
    &=\sum_{i=1}^{\infty}\frac{\mu_i}{\zeta^2/n_o+\mu_i}\times \braces{\frac{\EE_{x\sim  d^{\pie}_P}[\phi_i(x)\phi_i(x)^{\top}]}{\mu_i}}\\
       &=\sum_{j=1}^{\infty}\frac{\mu_j}{\zeta^2/n_o+\mu_j}\times \max_{i}(\EE_{x\sim  d^{\pie}_P}[\psi_i(x)\psi_i(x)^{\top}]).         
      \end{align*}  
Then, $C^{\pie}$ is replaced with $\max_{i}(\EE_{x\sim  d^{\pie}_P}[\psi_i(x)\psi_i(x)^{\top}])$. 
\end{remark}

\paragraph{Second statement}

We use \pref{lem:radema_rkhs} to calculate the critical radius of the RKHS ball. The critical inequality is 
\begin{align*}
    \sqrt{1/n_o}\sqrt{\sum_{i=1}^{n_o}\min(y^2,\mu_j)}\leq y^2.
\end{align*}
We show $y=\sqrt{d^{*}/n_o}$ satisfies the above. This is proved by 
\begin{align*}
     \sqrt{1/n_o}\sqrt{\sum_{i=1}^{n_o}\min(y^2,\mu_j)}& \leq    \min_{1\leq k\leq n_o}\{\sqrt{1/n}\sqrt{ky^2+B(k+1)}\}\\
     &\leq \sqrt{1/n_o}\sqrt{d^{*}y^2+B(d^{*}+1)} \tag{$d^{*}\leq n_o$} \\
     &\leq  \sqrt{1/n_o}\sqrt{d^{*}y^2+d^{*}/n_o} \tag{$B(d^{*}+1)\leq d^{*}/n_o$}\\
     &\leq \sqrt{d^{*}y^2/n_o}\leq y^2. 
\end{align*}

\end{proof}

\section{Auxiliary Lemmas}
\label{sec:auxiliary_lemmas}

\begin{lemma}[Simulation Lemma]  Consider any two functions $f: \Scal\times\Acal \mapsto [0,1]$ and $\widehat{f}:\Scal\times\Acal\mapsto[0,1]$, any two transitions $P$ and $\widehat{P}$, and any policy $\pi:\Scal\mapsto \Delta(\Acal)$. We have:

\resizebox{\textwidth}{!}{
\begin{minipage}{\textwidth}
\begin{align*}
    V^{\pi}_{P; f} - V^{\pi}_{\widehat{P}, \widehat{f}} & = \sum_{h=0}^{H} \EE_{s,a\sim d^{\pi}_{P}} \left[ f(s,a) - \widehat{f}(s,a)  + \EE_{s'\sim P(\cdot|s,a)} [V^{\pi}_{\widehat{P},\widehat{f}; h}(s')] - \EE_{s'\sim \widehat{P}(\cdot|s,a)} [V^{\pi}_{\widehat{P},\widehat{f}; h}(s')]  \right] \\
    & \leq \sum_{h=0}^{H} \EE_{s,a\sim d^{\pi}_{P}} \left[ f(s,a) - \widehat{f}(s,a)  + \|V^{\pi}_{\widehat{P},\widehat{f};h} \|_{\infty} \| P(\cdot | s,a) - \widehat{P}(\cdot | s,a)  \|_1   \right] .
\end{align*}\end{minipage}}

where $V^{\pi}_{P,f;h}$ denotes the value function at time step $h$, under $\pi, P, f$.
\label{lem:simulation}
\end{lemma}
Such simulation lemma is standard in model-based RL literature and the derivation can be found, for instance, in the proof of Lemma 10 from \cite{Sun2019_model}.

\begin{lemma}[$\ell_1$ Distance between two Gaussians] Consider two Gaussian distributions $P_1 := \Ncal(\mu_1, \zeta^2 I)$ and $P_2 := \Ncal(\mu_2, \zeta^2 I)$. We have:
\begin{align*}
\left\| P_1 - P_2 \right\|_{1} \leq \frac{1}{\zeta} \left\| \mu_1 - \mu_2 \right\|_2.
\end{align*}\label{lem:gaussian_tv}
\end{lemma}
This lemma is proved by Pinsker's inequality and the closed-form of the KL divergence between $P_1$ and $P_2$.

\begin{lemma}[Concentration on the inverse of state-action visitation]\label{lem:tabular_lemma}
We set $\lambda=\Omega(1)$. Then, with probability $1-\delta$, 
\begin{align*}
    \frac{1}{N(s,a)+\lambda}\leq \frac{c_1\log(|\Scal||\Acal|c_2/\delta)}{n_o\rho(s,a)+\lambda}\quad \forall (s,a)\in \Scal\times \Acal.  
\end{align*}
\end{lemma}
The extension of this lemma to the linear models is stated in \cref{eq:linear_con}. 

\begin{proof}
We set $\xi=c_1\log(|\Scal||\Acal|/\delta)+1\,(c_1> 4/3+3)$. 
First, we have 
\begin{align*}
    \frac{1}{N(s,a)+\lambda}\leq \frac{\xi}{ N(s,a)+\xi \lambda}. 
\end{align*}
from $\xi \geq 1$. Here, by Bernsteins's inequality, with probability $1-\delta$, 
\begin{align*}
    N(s,a)\geq n_o\rho(s,a)-2\sqrt{2n_o\rho(s,a)(1-\rho(s,a))\log(|\Scal||\Acal|/\delta)}-4\log(|\Scal||\Acal|/\delta)/3,\,\forall (s,a).  
\end{align*}
Thus, $\forall (s,a)\in \bar V$, we have

\resizebox{\textwidth}{!}{
\begin{minipage}{\textwidth}
\begin{align*}
    N(s,a)+\xi \lambda &\geq n_o\rho(s,a)-2\sqrt{2n_o\rho(s,a)(1-\rho(s,a))\log(|\Scal||\Acal|/\delta)}-4\log(|\Scal||\Acal|/\delta)/3+\xi \lambda\\
    &\geq n_o\rho(s,a)-2\sqrt{2n_o\rho(s,a)(1-\rho(s,a))\log(|\Scal||\Acal|/\delta)}+(c_1 -4/3)\log(|\Scal||\Acal|/\delta)+ \lambda \\
        &\geq n_o\rho(s,a)-2\sqrt{2n_o\rho(s,a)\log(|\Scal||\Acal|/\delta)}+(c_1 -4/3)\log(|\Scal||\Acal|/\delta)+ \lambda  \\
    &\geq  0.5 n_o\rho(s,a)+(\sqrt{0.5 n_o\rho(s,a)}- \sqrt{4\log(|\Scal||\Acal|/\delta)})^2+(c_1 -4/3-4)\log(|\Scal||\Acal|/\delta)+\lambda  \\
    &\geq  0.5 n_o\rho(s,a)+0.5\lambda. 
\end{align*}\end{minipage}}

This implies with $1-\delta$, 
\begin{align*}
    \frac{1}{N(s,a)+\lambda}\leq \frac{2\xi}{ n_0 \rho(s,a)+\lambda}\,\forall (s,a). 
\end{align*} 
Then, noting $c_1\log(|\Scal||\Acal|/\delta)+1\leq c_1\log(|\Scal||\Acal|c_2/\delta)$ for some $c_2$, the proof is concluded.  
\end{proof}

\begin{lemma}[A uniform law with localization: Theorem 14.1 in \citep{WainwrightMartinJ2019HS:A}]

\label{lem:kernel3}
Assume  $\|\Fcal \|_{\infty}\leq b$. Denote the critical radius of a function class $\Fcal$ by $\delta_n$. 
The critical radius $\delta_n$ is defined as a solution to 
\begin{align*}
    \Rcal_{n_o}(y;\Fcal)\leq y^2/b. 
\end{align*}
w.r.t $y$. Then, with probability $1-\delta$
\begin{align*}
    \frac{1}{n_o}\sum_{i=1}^{n_o}f(x_i)^2\geq 1/2\EE_{x\sim \rho}[f^2(x)]-(\delta'_n)^2/2\quad \forall f\in \Fcal,
\end{align*}
where $\delta'_n=\delta_n+c_1\sqrt{\log(c_2/\delta)/n_o}$. 
\end{lemma}

\begin{lemma}[Critical radius of linear models]\label{lem:critical_para}
Assume $\|\phi(s,a)\|_2\leq 1$ for any $(s,a)\in \Scal\times \Acal$. Then, the critical radius of function class $\Fcal=\{(s,a)\mapsto  a^{\top}\phi(s,a):\|a\|^2_2\leq \alpha,a^{\top}\phi\leq \beta,a\in \RR^d \}$ is upper-bounded by 
\begin{align*}
    c\sqrt{\beta\mathrm{rank}(\Sigma_{\rho})/n_o}.
\end{align*}
where $c$ is a universal constant. 
\end{lemma}

We follow the proof of \cite[Chapter 14]{WainwrightMartinJ2019HS:A}. Their argument depends on the assumption $\Sigma_{\rho}$ is full rank. We need to change the proof so that the full-rank assumption is removed and the rank $\rank[\Sigma_{\rho}]$ would appear in the final bound instead of $d$. Note that the final bound does not include $\alpha$. 

\begin{proof}
Unless otherwise noted, in this proof, $\EE[\cdot]$ is taken w.r.t. 
\begin{align*}
    x_i=(s_i,a_i) \sim \rho(s,a),\quad \epsilon_i \sim 2\{\mathrm{Ber}(0.5)-1\}. 
\end{align*}
Note that $x_i$ and $\epsilon_i$ are independent. 

Noting  $\EE_{\rho\sim (s,a)}[(a^{\top}\phi(s,a))^2]=a^{\top}\Sigma_{\rho}a$, the localized Rademacher complexity of $\Fcal$, $\Rcal_{n_o}(\xi;\Fcal)$, is
\begin{align*}
    \EE\bracks{\sup_{\{b\in \RR^d: \|b\|^2_2\leq \alpha,\|b\|_{\Sigma_{\rho}}\leq \xi,b^{\top}\phi\leq \beta\} }\left|\frac{1}{n_o}\sum_{i=1}^{n_o} \epsilon_i\{b^{\top}\phi(s_i,a_i)\}\right| }
\end{align*}    
 where $\{\epsilon_i\}_{i=1}^{n_o}$ is a set of independent Rademacher variables. This is upper-bounded by 
 \begin{align*}  
    &\EE\bracks{\sup_{\{b\in \RR^d: \|b\|^2_2\leq \alpha,\|b\|_{\Sigma_{\rho}}\leq \xi \} }\left|\frac{1}{n_o}\epsilon^{\top}\Phi b\right|  }
\end{align*}
where $\Phi$ is a $n_o\times d$ design matrix s.t. the $i$-th row is $\phi^{\top}(s_i,a_i)$ and $\epsilon=(\epsilon_1,\cdots,\epsilon_{n_o})^{\top}$. 

Here, we have $\EE[\Phi^{\top}\Phi]=n_o\Sigma_{\rho}$. Let $UVU^{\top}$ be the SVD of $\Sigma_{\rho}$,  where $U$ is a $n\times \mathrm{rank}[\Sigma_{\rho} ]$ matrix and $V$ is a $\mathrm{rank}[\Sigma_{\rho} ]\times \mathrm{rank}[\Sigma_{\rho} ]$ diagonal matrix.  Noting $b=UU^{\top}b+(\Ib-UU^{\top})b$, we have 
\begin{align*}
    & \EE\bracks{\sup_{\{b\in \RR^d: \|b\|^2_2\leq \alpha,\|b\|_{\Sigma_{\rho}}\leq \xi\} }|\frac{1}{n_o}\epsilon^{\top}\Phi \{UU^{\top}b+(\Ib-UU^{\top})b\}|  }\\
    & \leq \EE\bracks{\sup_{ \{b\in \RR^d: \|b\|^2_2\leq \alpha \}}|\frac{1}{n_o}\epsilon^{\top}\Phi (\Ib-UU^{\top})b\}|}+ \EE\bracks{\sup_{\|b\|_{\Sigma_{\rho}}\leq \xi }|\frac{1}{n_o}\epsilon^{\top}\Phi UU^{\top}b\}|}\\
      & \leq \EE\bracks{\sup_{ \{b\in \RR^d: \|b\|^2_2\leq \alpha \} }|\frac{1}{n_o}\epsilon^{\top}\Phi (\Ib-UU^{\top})b\}|}+ \EE\bracks{\sup_{\|c\|_{V}\leq \xi }|\frac{1}{n_o}\epsilon^{\top}\Phi Uc\}|} \tag{$U^{\top}c=b$}\\
  & \leq \EE\bracks{\frac{\alpha}{n_o}\|\epsilon^{\top}\Phi (\Ib-UU^{\top})\}\|_2}+\frac{\zeta}{n_o} \EE\bracks{\|\epsilon^{\top}\Phi U\}\|_{V^{-1}}} \tag{CS inequality}\\
   & \leq \frac{\alpha}{n_o}\sqrt{\EE\bracks{\|\epsilon^{\top}\Phi (\Ib-UU^{\top})\}\|^2_2}}+  \frac{\zeta}{n_o}\sqrt{\EE\bracks{\|\epsilon^{\top}\Phi U\|^2_{V^{-1}}}} \tag{Jensen's inequality}. 
\end{align*}
We analyze the second term and first term respectively. 

Regarding the second term, we have 
\begin{align*}
   \EE_{\epsilon}[\|\epsilon^{\top}\Phi U\|^2_{V^{-1}}]=\EE_{\epsilon}[\epsilon^{\top}\Phi UV^{-1}U^{\top}\Phi^{\top}\epsilon]=\tr(\Phi UV^{-1}U^{\top}\Phi^{\top}), 
\end{align*}
where $\EE_{\epsilon}[\cdot]$ is an expectation only regarding $\epsilon$. Then, by the law of total expectation, 
\begin{align*}
   \EE[\|\epsilon^{\top}\Phi U\|^2_{V^{-1}}]&=\EE[\tr(\Phi UV^{-1}U^{\top}\Phi^{\top}) ]\\
   &= \EE[\tr(\Phi^{\top}\Phi UV^{-1}U^{\top}) ]=\tr(n_o\Sigma_{\rho}UV^{-1}U^{\top})\\
   &=n_o\tr(UV U^{\top}UV^{-1}U^{\top})\\
   &= n_o\tr(UU^{\top})= n_o\tr(U^{\top}U)=n_o\mathrm{rank}(\Sigma_{\rho}). 
\end{align*}
Similarly, 
\begin{align*}
    \EE_{\epsilon}\bracks{\|\epsilon^{\top}\Phi (\Ib-UU^{\top})\}\|^2_2}=\tr(\Phi (\Ib-UU^{\top}) (\Ib-UU^{\top})\Phi^{\top} )=\tr(\Phi^{\top}\Phi (\Ib-UU^{\top}) ). 
\end{align*}
Then, by the law of total expectation, 
\begin{align*}
     \EE\bracks{\|\epsilon^{\top}\Phi (\Ib-UU^{\top})\}\|^2_2}&=\EE[\tr(\Phi^{\top}\Phi (\Ib-UU^{\top})) ]\\
     &=n_o\tr(\Sigma_{\rho} (\Ib-UU^{\top}))=n_o\tr(UV U^{\top}(\Ib-UU^{\top}))=0. 
\end{align*}

Combining all things together, 
\begin{align*}
    \Rcal_n(\xi;\Fcal)\leq \xi\sqrt{\rank[\Sigma_{\rho}]/n_o}. 
\end{align*}
Then, the critical inequality becomes 
\begin{align*}
   y \sqrt{\mathrm{rank}(\Sigma_{\rho})/n_o}\leq y^2/\beta. 
\end{align*}
Thus, the critical radius of $\Fcal$ is 
\begin{align*}
    \sqrt{\beta\mathrm{rank}(\Sigma_{\rho})/n_o}. 
\end{align*}
\end{proof}

\begin{lemma}[Variatioanl representation of kernels]\label{lem:kernel}
We denote the RKHS associated with a kernel $k(\cdot,\cdot)$ by $\Hcal_k$. Then, 
$$k(x,x)=\sup_{\{f:\|f\|_{k}\leq 1,f\in \Hcal_k \}}f^2(x).$$ 
\end{lemma}
\begin{proof}
We have 
\begin{align*}
    \sup_{\{f:\|f\|_{k}\leq 1,f\in \Hcal_k \}}f^2(x)&= \sup_{\{f:\|f\|_{k}\leq 1,f\in \Hcal_k \}}\langle f,k(x,\cdot)\rangle^2_{k}\\
    &\leq \sup_{\{f:\|f\|_{k}\leq 1,f\in \Hcal_k \}}\|f\|^2_{k}k(x,x) \tag{CS inequality}\\
    &=k(x,x). 
\end{align*}
Besides, the equality is satisfied when $f(\cdot)=k(x,\cdot)/\sqrt{k(x,x)}$ noting 
\begin{align*}
    f^2(x)=k^2(x,x)/k(x,x)=k(x,x),\quad \|f(\cdot)\|_k=\|k(x,\cdot)\|_k/k(x,x)=1.  
\end{align*}
Thus, 
$$k(x,x)=\sup_{\{f:\|f\|_{k}\leq 1,f\in \Hcal_k \}}f^2(x).$$ 
\end{proof}

\begin{lemma}[Relation between $\Hcal_{k_{n_o}}$ and $\Hcal_{k}$ ] \label{lem:kernel2}
We denoting the RKHS associated with a kernel $k(\cdot,\cdot)$ by $\Hcal_k$ and the RKHS with a kernel $k_{n_o}(\cdot,\cdot)$ by $\Hcal_{k_{n_o}}$. 
Then, we have $\Hcal_k=\Hcal_{k_{n_o}}$. Besides, for $f\in \Hcal_k$, we have 
$$\|f\|^2_{k_{n_o}}=\|f\|^2_k+\zeta^{-2}\sum_{i=1}^{n_o}f(x_i)^2.$$
\end{lemma}
This is stated in \cite[Appendix B]{Srinivas2010} without the proof. For completeness, we provide the proof. 

\begin{proof}
We use Mercer's theorem \citep[Theorem 12.20]{WainwrightMartinJ2019HS:A}. Then, any element in the RKHS associated with the kernel $k(x,x)$ is represented by 
\begin{align*}
    f(x)=\sum_{i=1}^{\infty} f_i\psi_i(x). 
\end{align*}
where $\{\psi_i\}_{i=1}^{\infty}$ is an orthonormal basis for $L^2(\rho)$: $\EE_{x\sim \rho}[\psi_i(x)\psi_j(x)]=I(i=j)$. 
Here, we have 
\begin{align*}
    k(x,x)=\psi^{\top}(x) \Lambda \psi(x)=\phi^{\top}(x)\phi(x),\quad \|f\|_k=\tilde f^{\top}\Lambda^{-1}\tilde f,
\end{align*}
where $\phi_i(x)=\sqrt{\mu_i}\psi_i(x)$ and $\tilde f=\{f_i\}_{i=1}^{\infty}\in \RR^{\infty}$. 

Then, by letting $\Phi$ be a $n\times d$ matrix s.t. the $i$-th row is $\phi^{\top}(s_i,a_i)$, 
\begin{align*}
k_{n_o}(x,x) &=\phi^{\top}(x)\phi(x)-\phi^{\top}(x)\Phi^{\top}(\Phi\Phi^{\top}+\zeta^2 \Ib)^{-1}\Phi\phi(x)\\
     &= \phi^{\top}(x)\{\Ib-\Phi^{\top}(\Phi\Phi^{\top}+\zeta^2 \Ib)^{-1} \Phi \} \phi(x)\\
     &= \phi^{\top}(x)\{\Ib+\Phi^{\top}\Phi/\zeta^2\}^{-1}\phi(x)  \tag{Woodbury matrix identity}\\
    &=\phi^{\top}(x)(\Ib+\sum_{i=1}^{n_o} \phi(x_i)\phi(x_i)^{\top}/\zeta^2)^{-1}\phi (x). 
\end{align*} 
Here, let  $UVU^{\top}$ be the eigenvalue decomposition of $\{\Lambda^{-1} +\sum_{i=1}^{n_o} \psi(x_i)\psi(x_i)^{\top}/\zeta^2\}^{-1}=UVU^{\top}$. Then, 
\begin{align*}    
 k_{n_o}(x,x)   &=\psi^{\top}(x)(\Lambda^{-1} +\sum_{i=1}^{n_o} \psi(x_i)\psi(x_i)^{\top}/\zeta^2)^{-1}\psi(x)\\
    &=\psi^{\top}(x) UVU^{\top}\phi(x)  \\
    &=\psi'^{\top}(x)V\psi'(x) \tag{$U^{\top}\psi=\psi'$ }. 
\end{align*}

Then, any element $f(\cdot)$ in the RKHS associated with the kernel $k_{n_o}(x,x)$ is represented as 
$$f(\cdot)=\tilde g^{\top}\psi'(\cdot),\,\,\tilde g\in \mathbb{R}^{\infty},$$
and the associated norm is $\|f\|_{k_{n_o}}=\tilde g^{\top}V^{-1}\tilde g$ since $\psi'(\cdot)$ is still an orthnormal basis for $L^2(\rho)$, i.e., $\EE_{x\sim \rho}[\phi'_i(x)\phi_j(x)]=I(i=j)$.  This immediately implies $\Hcal_k=\Hcal_{k_{n_o}}$. 

Finally, we check the relation of the norm: 
\begin{align*}
    \|f\|^2_{k_{n_o}}&= \|\sum_{i=1}^{n_o} f_i\psi_i\|_{k_{n_o}}=\|\tilde f^{\top}\psi \|_{k_{n_o}} \tag{$\tilde f=\{f_1,f_2\cdots\}^{\top}$}\\
    &=\|\{ U^{\top}\tilde f\}^{\top}U^{\top}\psi \|_{k_{n_o}}\\
    &= \|\{ U^{\top}\tilde f\}^{\top}\psi' \|_{k_{n_o}}  \\
    &=\{U^{\top} \tilde  f\}^{\top}V^{-1}U\tilde f  \\
    &=\tilde f^{\top}(\Lambda^{-1} +\sum_{i=1}^{n_o} \phi(x_i)\phi(x_i)^{\top}/\zeta^2)\tilde f\\
    &=\|f\|_k+1/\zeta^2 \sum_{i=1}^{n_o} \{\tilde f^{\top}\phi(x_i)\}^2=\|f\|_k+\zeta^{-2}\sum_{i=1}^{n_o} f^2(x_i). 
\end{align*}
\end{proof}

\begin{lemma}\label{lem:kernel_lemma}
Let $\{\hat \mu_i\}_{i=1}^{n_o}$ be the eigenvalues of $\Kb_{n_o}$. Then, 
\begin{align*}
\sum_{i=1}^{n_o}  \frac{\hat \mu_i/\zeta^2}{\hat \mu_i/\zeta^2+1}=\frac{1}{\zeta^2} \sum_{i=1}^{n_o}k_{n_o}(x_i,x_i). 
\end{align*}
\end{lemma}
\begin{proof}

\begin{align*}
     \sum_{i=1}^{n_o}k_{n_o}(x_i,x_i)&= \sum_{i=1}^{n_o}k(x_i,x_i)-\bar k^{\top}_{n_o}(x_i)\{\Kb_{n_o}+\zeta^2 \Ib\}^{-1}\bar k_{n_o}(x_i)\\ 
     &= \tr\prns{\sum_{i=1}^{n_o}k(x_i,x_i)-\bar k^{\top}_{n_o}(x_i)\{\Kb_{n_o}+\zeta^2 \Ib\}^{-1}\bar k_{n_o}(x_i)}\\
         &= \tr\prns{\Kb_{n_o}}-\tr\prns{\sum_{i=1}^{n_o}\bar k_{n_o}(x_i)\bar k^{\top}_{n_o}(x_i)\{\Kb_{n_o}+\zeta^2 \Ib\}^{-1}}\\
    &=\tr\prns{\Kb_{n_o}-\Kb^2_{n_o}\{\Kb_{n_o}+\zeta^2 \Ib\}^{-1} }\\
       &=\tr\prns{\{\Kb^2_{n_o}+\zeta^2 \Kb_{n_o}-\Kb^2_{n_o}\}\{\Kb_{n_o}+\zeta^2 \Ib\}^{-1} }\\ 
             &=\tr\prns{\zeta^2 \Kb_{n_o}\{\Kb_{n_o}+\zeta^2 \Ib\}^{-1} }=\sum_{i=1}^{n_o}  \frac{\hat \mu_i}{\hat \mu_i/\zeta^2+1}. 
\end{align*}

\end{proof}

\begin{lemma}[Calculation of localized Rademacher complexity of RKHS balls: Corollary 14.5 in \citep{WainwrightMartinJ2019HS:A}]\label{lem:radema_rkhs}
Let $\Fcal=\{f\in \Hcal_k:\|f\|_{k}\leq 1\}$ be the unit ball of an RKHS with eigenvalues $\{\mu_j\}_{j=1}^{\infty}$. Then, the localized population Rademacher complexity is upper-bounded by 
\begin{align*}
    \Rcal_n(\delta;\Fcal)\leq \sqrt{\frac{2}{n}}\sqrt{\sum_{j=1}^{\infty}\min(\mu_j,\delta^2) }. 
\end{align*}

\end{lemma}

\begin{lemma}[Upper-bound of expectation of information gains: finite-dimensional models ]\label{thm:seeger2008}
$$\EE[\bar \Ical_{n_o}] \leq \rank(\Sigma_{\rho})\{\log(1+n_o/\lambda)+1\}.$$ 
\end{lemma}
\begin{proof}
\begin{align*}
    \EE[\bar \Ical_{n_o}] &=\EE[\log(\det(\Sigma_{n_o}/\lambda))]\\ 
  & \leq \log \det(\EE[\Sigma_{n_o}/\lambda  ])=\log\det(\Ib+n_o/\lambda \Sigma_{\rho}) \tag{Jensen's inequality}\\
  &\leq \rank(\Sigma_{\rho})\{\log(1+n_o/\lambda)+1\}. 
\end{align*}
The final line is proved as in the proof of \cref{thm:info_gain_para}. 
\end{proof}

\begin{lemma}[Upper-bound of expectation of information gains: RKHS]\label{thm:seeger2008_kernel}
$$\EE[\Ical_{n_o}] \leq  2d^{*}\{\log(1+n_o/\zeta^2)+1\}.$$ 
\end{lemma}
\begin{proof}
\begin{align*}
    \EE[\Ical_{n_o}] &=\EE[\log(\det(I+\zeta^{-2} \Kb_{n_o}))]\\ 
  & \leq  \sum_{s=1}^{\infty}\log(1+\zeta^{-2} \mu_s n_o) \tag{Refer to \cite[Lemma 1]{seeger2008}  }\\
  &\leq  \{\log(1+n_o/\zeta^2)+1\}2d^{*}.
\end{align*}
From the second line to the third line, we follow in the proof of \cref{thm:info_gain_nonpara}. 
\end{proof}

\section{Implementation Details}
Here we detail all environment details and hyperparameters used for the experiments in the main text.

\subsection{Environment Details}
All environments have a maximum horizon length of 500 timesteps. We achieve this by reducing the data collection frequency of the base 1000 horizon environments. We also remove all contact information from the observation and the reward. Finally, to be able to compute the ground truth reward from the state, we add the velocity of the center of mass into the state.

\begin{table}[h]
    \centering
    \caption{Observation and action space dimensions for each of the environments}
    \begin{tabular}{c|cc}
    \toprule
    Environment & Observation Space Dimension & Action Space Dimension\\
    \midrule
    Hopper      & 12 & 3 \\
    Walker2d    & 18 & 6 \\
    HalfCheetah & 18 & 6 \\
    Ant         & 29 & 8 \\
    Humanoid    & 47 & 17 \\
    \bottomrule
    \end{tabular}
\end{table}
\begin{table}[h]
    \centering
    \caption{Ground truth environment reward function used to train the expert and behavior policies as well as evaluate the performance in the learning curves. At time $t$, $\dot{x}_t$ is the velocity of the center of mass in the $x$-axis, $\mathbf{a_t}$ is the action vector, and $z_t$ is the position of the center of mass in the $z$-axis.}
    \begin{tabular}{c|c}
    \toprule
        Environment & Ground Truth Reward Function \\
    \midrule
        Hopper      & $\dot{x}_t - 0.1\lVert\mathbf{a_t}\rVert_2^2 - 3.0\times(z_t - 1.3)^2$ \\
        Walker2d    & $\dot{x}_t - 0.1\lVert\mathbf{a_t}\rVert_2^2 - 3.0\times(z_t - 0.57)^2$\\
        HalfCheetah & $\dot{x}_t - 0.1\lVert\mathbf{a_t}\rVert_2^2$ \\
        Ant         & $\dot{x}_t - 0.1\lVert\mathbf{a_t}\rVert_2^2 - 3.0\times(z_t - 1.3)^2$\\
        Humanoid    & $1.25\times\dot{x}_t - 0.1\lVert\mathbf{a_t}\rVert_2^2 + 5\times bool(1.0\le z_t\le 2.0)$\\
    \bottomrule
    \end{tabular}
\end{table}

\subsection{Dynamics Ensemble Architecture and Model Learning}
For all of our experiments we use an ensemble of four dynamics models with each model parameterized by a feed-forward neural network with two hidden layers containing 1024 units. The learned model does not predict next state, but instead predicts the normalized difference between the next state and the current state, $s_{t+1}- s_t$. The activation function used at each layer is ReLU. We train all of our ensembles using Adam with learning rate $5\times10^{-5}$ and otherwise default hyperparameters. We train each dynamics model for 300 epochs on just the offline dataset for all of our experiments. Please see Table \ref{tab:dynamics_hparam} for all values.
\begin{table}[h]
    \centering
    \caption{All hyperparameters used for dynamics model learning}
    \begin{tabular}{c|c}
    \toprule
    Hyperparameter & Value \\
    \midrule
    Hidden Layers & $(1024, 1024)$\\
    Activation    & ReLU \\
    Optimizer     & Adam \\
    Learning Rate & $5\times10^{-5}$ \\
    Batch Size    & 256 \\
    Epochs        & 300 \\
    \bottomrule
    \end{tabular}
    \label{tab:dynamics_hparam}
\end{table}

\subsection{Policy Architecture and TRPO Details}
We use the open source NPG/TRPO implementation, MJRL \citep{Rajeswaran-RSS-18}. The policy network and the value network are feedforward neural networks with two hidden layers containing 32 and 128 hidden units respectively. Both networks use a \texttt{tanh} activation function with the policy network outputting a Gaussian distribution $\mathcal{N}(\mu(s), \sigma^2)$ where $\sigma$ is a trainable parameter. We use Generalized Advantage Estimator (GAE) to estimate the advantages. Please see Table \ref{tab:trpo} for all values.
\begin{table}[h]
    \centering
    \caption{TRPO/NPG hyperparameter values used in experiments.}
    \begin{tabular}{c|c}
    \toprule
    Hyperparameter & Value \\
    \midrule
    Policy Hidden Layers     & $(32, 32)$ \\
    Critic Hidden Layers     & $(128, 128)$ \\
    Batch Size               & 40000 \\
    Max KL Divergence        & 0.01 \\
    Discount $\gamma$        & $0.995$ \\
    CG Iterations            & 25 \\
    CG Damping               & $1\times 10^{-5}$ \\
    GAE $\lambda$            & $0.97$ \\
    Critic Update Epochs     & 2 \\
    Critic Optimizer         & Adam \\
    Critic Learning Rate     & $1\times 10^{-4}$ \\
    Critic L2 Regularization & $1\times 10^{-4}$ \\
    Policy Init Log Std.     & -0.25 \\
    Policy Min Log Std.      & -2.0 \\
    BC Regularization $\lambda_{\text{BC}}$ & 0.1 \\
    \bottomrule
    \end{tabular}
    \label{tab:trpo}
\end{table}

\subsection{Hyperparameter Selection}
For our core results, we tuned our hyperparameters on a randomly selected seed for \hopper~and then applied it for all environments. For TRPO, we tuned the conjugate gradient iterations from values 10, 25, and 50; and the conjugate gradient damping coefficients from values 1e-2, 1e-3, 1e-4, and 1e-5. All other hyperparameters were default ones in the MJRL repository \citep{Rajeswaran-RSS-18}. For the BC regularization coefficient we tested values of 0.1, 1e-2, 1e-3, 1e-4, and 1e-5. For the dynamics model architecture we tested 3 different hidden layer sizes: 256, 512, and 1024. Beyond this we used the exact same Adam optimizer and training procedure as \citep{Kidambi2020}.

\subsection{Discriminator Update and Cost Function Details}
We parameterize our discriminator as a linear function $f(s,a) = w^{\top}\phi(s,a)$, where $\phi(s,a)$ are Random Fourier Features \citep{rff} and $w$ is the vector of parameters for the discriminator. Recall our objective,
\begin{equation*}
    \min\limits_{\pi\in\Pi}\max\limits_{f\in\Fcal} \left[\EE_{(s,a)\sim d^{\pi}_{\hat{P}}}(f(s,a) + b(s,a)) - \EE_{(s,a)\sim\Dcal_{e}}[f(s,a)]\right] + \lambda_{\text{BC}}\cdot\EE_{(s,a)\sim\Dcal_{e}}[\ell(a,s,\pi)].
\end{equation*}
Now given a policy $\pi$, we can compute a closed form update for the discriminator parameters $w$ like so
\begin{align*}
    \max\limits_{w:\lVert w\rVert_2^2\le\eta} L(w; \pi,\widehat{P}, b, \Dcal_{e}) &:= \EE_{(s,a)\sim d^{\pi}_{\hat{P}}}(f(s,a) + b(s,a)) - \EE_{(s,a)\sim\Dcal_{e}}[f(s,a)] \\
    \equiv \max\limits_{w} L_{\eta}(w; \pi,\widehat{P}, b, \Dcal_{e}) &= \EE_{(s,a)\sim d^{\pi}_{\hat{P}}}(f(s,a) + b(s,a)) - \EE_{(s,a)\sim\Dcal_{e}}[f(s,a)] - \frac{1}{2}\cdot(\lVert w\rVert_2^2 - \eta)\\
    \Rightarrow \partial_{w}L_{\eta}(w; \pi,\widehat{P}, b, \Dcal_{e}) &= \EE_{(s,a)\sim d^{\pi}_{\hat{P}}}[\phi(s,a)] - \EE_{(s,a)\sim\Dcal_{e}}[\phi(s,a)] - w \\
\end{align*}
where $\partial_{w}L_{\eta}(w; \pi,\widehat{P}, b, \Dcal_{e})$ denotes the partial derivative of $L_{\eta}(\cdot)$ wrt to $w$. Setting the above expression to 0 and solving for $w$ gives us the closed form solution. Note that even with the BC regularization constraint added into the objective, the solution will still hold.

Now for a given updated $w_t$, we have our cost function $c(s,a) = w_t^{\top}\phi(s,a) + b(s,a)$ where our penalty, $b(s,a)$, is the maximum discrepancy of our model ensemble predictions. To balance our penalty term with our cost term, we introduce a parameter $\lambda_{\text{penalty}}$ to get the cost 
\begin{equation*}
    c(s,a) = (1-\lambda_{\text{penalty}})\cdot w_t^{\top}\phi(s,a) + \lambda_{\text{penalty}}\cdot b(s,a).
\end{equation*}
In our experiments, $\lambda_{\text{penalty}}$ was the only parameter we varied across environments.
\begin{table}[h]
    \centering
    \caption{$\lambda_{\text{penalty}}$ values used for each environment.}
    \begin{tabular}{c|c}
    \toprule
    Environment & $\lambda_{\text{penalty}}$ \\
    \midrule
    Hopper      &  $2.5\times 10^{-4}$\\
    Walker2d    &  $1.0\times 10^{-7}$\\
    HalfCheetah &  $1.0\times 10^{-4}$\\
    Ant         &  $1.0\times 10^{-4}$\\
    Humanoid    &  $5.0\times 10^{-4}$\\
    \bottomrule
    \end{tabular}
    \label{tab:cost_param}
\end{table}

\section{Additional Experiments}
Recall that in our main experiments, we create an extremely small expert dataset containing expert $(s,a)$ pairs by randomly sampling state-action pairs from an expert dataset consisting of state-action pairs from many expert trajectories, and we did that for the purpose of creating an expert dataset where BC almost fails completely.  One may wonder what \alg~would do if we feed \alg~ a complete single expert trajectory. We conduct such experiments in this section.
Figure \ref{fig:traj_level} shows the performance of \alg\, with one expert trajectory using the \emph{same hyperparameters} as before. All plots are shown averaged across five seeds. Note that \alg\, is still performs well with one expert trajectory---matching or nearly matching the expert performance across all 5 continuous control tasks.
\begin{figure}[h]
    \centering
    \includegraphics[width=\textwidth]{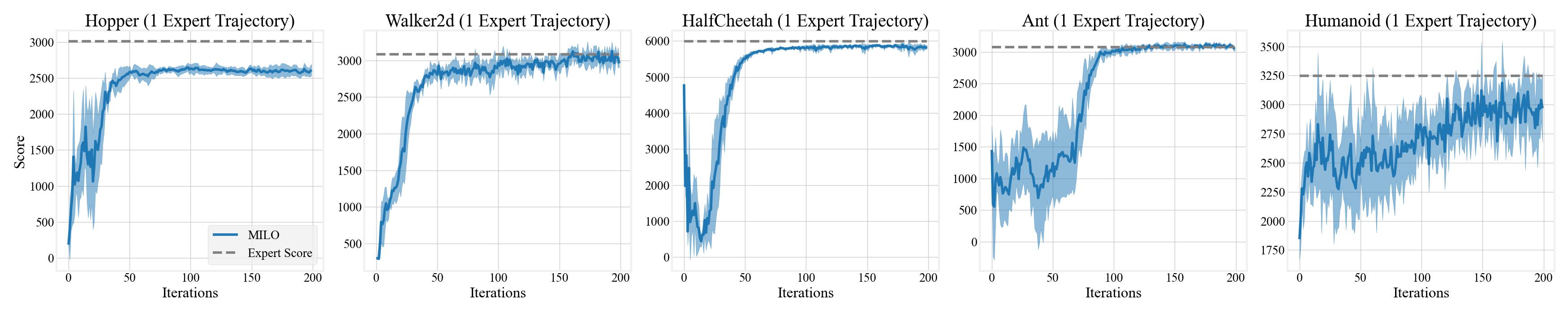}
    \caption{Performance of \alg\, with one expert trajectory. Note \alg\, performance just as well with trajectory inputs as with state-action pair sample inputs.}
    \label{fig:traj_level}
\end{figure}


\begin{thebibliography}{10}

\bibitem{Abbasi-yadkori2011}
Y.~Abbasi-yadkori, D.~P\'{a}l, and C.~Szepesv\'{a}ri.
\newblock Improved algorithms for linear stochastic bandits.
\newblock In {\em Advances in Neural Information Processing Systems},
  volume~24. Curran Associates, Inc., 2011.

\bibitem{abbeel2004apprenticeship}
P.~Abbeel and A.~Y. Ng.
\newblock Apprenticeship learning via inverse reinforcement learning.
\newblock In {\em ICML}, page~1. ACM, 2004.

\bibitem{agarwal2020pc}
A.~Agarwal, M.~Henaff, S.~Kakade, and W.~Sun.
\newblock Pc-pg: Policy cover directed exploration for provable policy gradient
  learning.
\newblock {\em NeurIPS}, 2020.

\bibitem{agarwal2019reinforcement}
A.~Agarwal, N.~Jiang, S.~M. Kakade, and W.~Sun.
\newblock Reinforcement learning: Theory and algorithms.
\newblock {\em CS Dept., UW Seattle, Seattle, WA, USA, Tech. Rep}, 2019.

\bibitem{antos2008learning}
A.~Antos, C.~Szepesv{\'a}ri, and R.~Munos.
\newblock Learning near-optimal policies with bellman-residual minimization
  based fitted policy iteration and a single sample path.
\newblock {\em Machine Learning}, 71:89--129, 2008.

\bibitem{AzizzadenesheliKamyar2018EETB}
K.~Azizzadenesheli, E.~Brunskill, and A.~Anandkumar.
\newblock Efficient exploration through bayesian deep q-networks.
\newblock In {\em 2018 Information Theory and Applications Workshop (ITA)},
  pages 1--9. IEEE, 2018.

\bibitem{BachFrancis2017OtEb}
F.~Bach.
\newblock On the equivalence between kernel quadrature rules and random feature
  expansions.
\newblock {\em Journal of machine learning research}, 18(21):1--38, 2017.

\bibitem{bansal2017goal}
S.~Bansal, R.~Calandra, T.~Xiao, S.~Levine, and C.~J. Tomiin.
\newblock Goal-driven dynamics learning via bayesian optimization.
\newblock In {\em 2017 IEEE 56th Annual Conference on Decision and Control
  (CDC)}, pages 5168--5173. IEEE, 2017.

\bibitem{bartlett2005}
P.~L. Bartlett, O.~Bousquet, and S.~Mendelson.
\newblock Local rademacher complexities.
\newblock {\em Annals of Statistics}, 33(4):1497--1537, 08 2005.

\bibitem{brantley2019disagreement}
K.~Brantley, W.~Sun, and M.~Henaff.
\newblock Disagreement-regularized imitation learning.
\newblock In {\em International Conference on Learning Representations}, 2019.

\bibitem{brockman2016gym}
G.~Brockman, V.~Cheung, L.~Pettersson, J.~Schneider, J.~Schulman, J.~Tang, and
  W.~Zaremba.
\newblock Openai gym, 2016.

\bibitem{BuckmanJacob2020TIoP}
J.~Buckman, C.~Gelada, and M.~G. Bellemare.
\newblock The importance of pessimism in fixed-dataset policy optimization.
\newblock {\em arXiv preprint arXiv:2009.06799}, 2020.

\bibitem{Calandriello19}
D.~Calandriello, L.~Carratino, A.~Lazaric, M.~Valko, and L.~Rosasco.
\newblock Gaussian process optimization with adaptive sketching: Scalable and
  no regret.
\newblock In {\em Proceedings of the Thirty-Second Conference on Learning
  Theory}, volume~99 of {\em Proceedings of Machine Learning Research}, pages
  533--557, 2019.

\bibitem{chan2021sbirl}
A.~J. Chan and M.~van~der Schaar.
\newblock Scalable bayesian inverse reinforcement learning.
\newblock {\em arXiv preprint arXiv:2102.06483}, 2021.

\bibitem{chang2015learning_dependency}
K.-W. Chang, H.~He, H.~Daum{\'e}~III, and J.~Langford.
\newblock Learning to search for dependencies.
\newblock {\em arXiv preprint arXiv:1503.05615}, 2015.

\bibitem{ChenJinglin2019ICiB}
J.~Chen and N.~Jiang.
\newblock Information-theoretic considerations in batch reinforcement learning.
\newblock In {\em Proceedings of the 36th International Conference on Machine
  Learning}, volume~97, pages 1042--1051, 2019.

\bibitem{Chowdhury2019}
S.~R. Chowdhury and A.~Gopalan.
\newblock Online learning in kernelized markov decision processes.
\newblock In K.~Chaudhuri and M.~Sugiyama, editors, {\em Proceedings of the
  Twenty-Second International Conference on Artificial Intelligence and
  Statistics}, volume~89 of {\em Proceedings of Machine Learning Research},
  pages 3197--3205. PMLR, 16--18 Apr 2019.

\bibitem{deisenroth2011pilco}
M.~Deisenroth and C.~E. Rasmussen.
\newblock {PILCO}: A model-based and data-efficient approach to policy search.
\newblock In {\em International Conference on Machine Learning}, pages
  465--472, 2011.

\bibitem{du2021bilinear}
S.~S. Du, S.~M. Kakade, J.~D. Lee, S.~Lovett, G.~Mahajan, W.~Sun, and R.~Wang.
\newblock Bilinear classes: A structural framework for provable generalization
  in rl.
\newblock {\em ICML}, 2021.

\bibitem{DuanYaqi2020MOEw}
Y.~Duan, Z.~Jia, and M.~Wang.
\newblock Minimax-optimal off-policy evaluation with linear function
  approximation.
\newblock In {\em Proceedings of the 37th International Conference on Machine
  Learning}, volume 119 of {\em Proceedings of Machine Learning Research},
  pages 2701--2709, 2020.

\bibitem{DuanYaqi2021RBaR}
Y.~Duan, C.~Jin, and Z.~Li.
\newblock Risk bounds and rademacher complexity in batch reinforcement
  learning.
\newblock {\em arXiv preprint arXiv:2103.13883}, 2021.

\bibitem{ernst2005tree}
D.~Ernst, P.~Geurts, and L.~Wehenkel.
\newblock Tree-based batch mode reinforcement learning.
\newblock {\em Journal of Machine Learning Research}, 6:503--556, 2005.

\bibitem{FakoorRasool2021CDCB}
R.~Fakoor, J.~Mueller, P.~Chaudhari, and A.~J. Smola.
\newblock Continuous doubly constrained batch reinforcement learning.
\newblock {\em arXiv preprint arXiv:2102.09225}, 2021.

\bibitem{FanJianqing2019ATAo}
J.~Fan, Z.~Wang, Y.~Xie, and Z.~Yang.
\newblock A theoretical analysis of deep q-learning.
\newblock In {\em Proceedings of the 2nd Conference on Learning for Dynamics
  and Control}, volume 120 of {\em Proceedings of Machine Learning Research},
  pages 486--489, 2020.

\bibitem{fisac2018general}
J.~F. Fisac, A.~K. Akametalu, M.~N. Zeilinger, S.~Kaynama, J.~Gillula, and
  C.~J. Tomlin.
\newblock A general safety framework for learning-based control in uncertain
  robotic systems.
\newblock {\em IEEE Transactions on Automatic Control}, 64(7):2737--2752, 2018.

\bibitem{pmlr-v97-fujimoto19a}
S.~Fujimoto, D.~Meger, and D.~Precup.
\newblock Off-policy deep reinforcement learning without exploration.
\newblock In {\em Proceedings of the 36th International Conference on Machine
  Learning}, volume~97 of {\em Proceedings of Machine Learning Research}, pages
  2052--2062. PMLR, 09--15 Jun 2019.

\bibitem{ho2016generative}
J.~Ho and S.~Ermon.
\newblock Generative adversarial imitation learning.
\newblock In {\em NIPS}, 2016.

\bibitem{Chiappa2020}
D.~Janz, D.~Burt, and J.~Gonzalez.
\newblock Bandit optimisation of functions in the matérn kernel rkhs.
\newblock In {\em Proceedings of the Twenty Third International Conference on
  Artificial Intelligence and Statistics}, volume 108 of {\em Proceedings of
  Machine Learning Research}, pages 2486--2495, 2020.

\bibitem{jarrett2020sbil}
D.~Jarrett, I.~Bica, and M.~van~der Schaar.
\newblock Strictly batch imitation learning by energy-based distribution
  matching.
\newblock In {\em Advances in Neural Information Processing Systems},
  volume~33, pages 7354--7365, 2020.

\bibitem{Jiang2020_note}
N.~Jiang.
\newblock Notes on tabular methods, 2020.

\bibitem{JinYing2020IPPE}
Y.~Jin, Z.~Yang, and Z.~Wang.
\newblock Is pessimism provably efficient for offline rl?
\newblock {\em arXiv preprint arXiv:2012.15085}, 2020.

\bibitem{Kakade2020}
S.~Kakade, A.~Krishnamurthy, K.~Lowrey, M.~Ohnishi, and W.~Sun.
\newblock Information theoretic regret bounds for online nonlinear control.
\newblock In {\em Advances in Neural Information Processing Systems},
  volume~33, pages 15312--15325, 2020.

\bibitem{kidambi2021optimism}
R.~Kidambi, J.~Chang, and W.~Sun.
\newblock Optimism is all you need: Model-based imitation learning from
  observation alone.
\newblock {\em arXiv preprint arXiv:2102.10769}, 2021.

\bibitem{Kidambi2020}
R.~Kidambi, A.~Rajeswaran, P.~Netrapalli, and T.~Joachims.
\newblock Morel: Model-based offline reinforcement learning.
\newblock In {\em Advances in Neural Information Processing Systems},
  volume~33, pages 21810--21823. Curran Associates, Inc., 2020.

\bibitem{klein2013csi}
E.~Klein, B.~Piot, M.~Geist, and O.~Pietquin.
\newblock A cascaded supervised learning approach to inverse reinforcement
  learning.
\newblock In {\em ECML/PKDD}, 2013.

\bibitem{ko2007gaussian}
J.~Ko, D.~J. Klein, D.~Fox, and D.~Haehnel.
\newblock Gaussian processes and reinforcement learning for identification and
  control of an autonomous blimp.
\newblock In {\em Proceedings 2007 ieee international conference on robotics
  and automation}, pages 742--747. IEEE, 2007.

\bibitem{kostrikov2019dac}
I.~Kostrikov, K.~K. Agrawal, D.~Dwibedi, S.~Levine, and J.~Tompson.
\newblock Discriminator-actor-critic: Addressing sample inefficiency and reward
  bias in adversarial imitation learning.
\newblock In {\em ICLR}. OpenReview.net, 2019.

\bibitem{KostrikovIlya2019ILvO}
I.~Kostrikov, O.~Nachum, and J.~Tompson.
\newblock Imitation learning via off-policy distribution matching.
\newblock {\em ICLR}, 2019.

\bibitem{kumar2020conservative}
A.~Kumar, A.~Zhou, G.~Tucker, and S.~Levine.
\newblock Conservative q-learning for offline reinforcement learning.
\newblock {\em arXiv preprint arXiv:2006.04779}, 2020.

\bibitem{lee2017dart}
M.~Laskey, J.~Lee, W.~Y. Hsieh, R.~Liaw, J.~Mahler, R.~Fox, and K.~Goldberg.
\newblock Iterative noise injection for scalable imitation learning.
\newblock {\em CoRR}, abs/1703.09327, 2017.

\bibitem{LeGratietLoic2015Aaot}
L.~Le~Gratiet, L.~Le~Gratiet, J.~Garnier, and J.~Garnier.
\newblock Asymptotic analysis of the learning curve for gaussian process
  regression.
\newblock {\em Machine learning}, 98(3):407--433, 2015.

\bibitem{LiaoPeng2020BPLi}
P.~Liao, Z.~Qi, and S.~Murphy.
\newblock Batch policy learning in average reward markov decision processes.
\newblock {\em arXiv preprint arXiv:2007.11771}, 2020.

\bibitem{Liu2020}
Y.~Liu, A.~Swaminathan, A.~Agarwal, and E.~Brunskill.
\newblock Provably good batch off-policy reinforcement learning without great
  exploration.
\newblock In {\em Advances in Neural Information Processing Systems},
  volume~33, pages 1264--1274, 2020.

\bibitem{MatsushimaTatsuya2020DRLv}
T.~Matsushima, H.~Furuta, Y.~Matsuo, O.~Nachum, and S.~Gu.
\newblock Deployment-efficient reinforcement learning via model-based offline
  optimization.
\newblock {\em ICLR}, 2020.

\bibitem{munos2008finite}
R.~Munos and C.~Szepesv{\'a}ri.
\newblock Finite-time bounds for fitted value iteration.
\newblock {\em Journal of Machine Learning Research}, 9(May):815--857, 2008.

\bibitem{ChowYinlam2019DBEo}
O.~Nachum, Y.~Chow, B.~Dai, and L.~Li.
\newblock Dualdice: Behavior-agnostic estimation of discounted stationary
  distribution corrections.
\newblock {\em Advances in Neural Information Processing Systems 2019}, 2019.

\bibitem{Osband2018}
I.~Osband, J.~Aslanides, and A.~Cassirer.
\newblock Randomized prior functions for deep reinforcement learning.
\newblock In {\em Advances in Neural Information Processing Systems},
  volume~31, 2018.

\bibitem{Pathak2019}
D.~Pathak, D.~Gandhi, and A.~Gupta.
\newblock Self-supervised exploration via disagreement.
\newblock In K.~Chaudhuri and R.~Salakhutdinov, editors, {\em Proceedings of
  the 36th International Conference on Machine Learning}, volume~97 of {\em
  Proceedings of Machine Learning Research}, pages 5062--5071. PMLR, 09--15 Jun
  2019.

\bibitem{pom89}
D.~A. Pomerlau.
\newblock {ALVINN}: An autonomous land vehicle in a neural network.
\newblock In D.~S. Touretzky, editor, {\em Advances in Neural Information
  Processing Systems}, pages 323--331, San Mateo, CA, 1989. Morgan Kaufmann
  Publishers inc.

\bibitem{rafailov2021vmail}
R.~Rafailov, T.~Yu, A.~Rajeswaran, and C.~Finn.
\newblock Visual adversarial imitation learning using variational models.
\newblock {\em CoRR}, abs/2107.08829, 2021.

\bibitem{rff}
A.~Rahimi and B.~Recht.
\newblock Random features for large-scale kernel machines.
\newblock In J.~Platt, D.~Koller, Y.~Singer, and S.~Roweis, editors, {\em
  Advances in Neural Information Processing Systems}, volume~20. Curran
  Associates, Inc., 2008.

\bibitem{RajaramanNived2020TtFL}
N.~Rajaraman, L.~F. Yang, J.~Jiao, and K.~Ramachandran.
\newblock Toward the fundamental limits of imitation learning.
\newblock {\em arXiv preprint arXiv:2009.05990}, 2020.

\bibitem{Rajeswaran-RSS-18}
A.~Rajeswaran, V.~Kumar, A.~Gupta, G.~Vezzani, J.~Schulman, E.~Todorov, and
  S.~Levine.
\newblock {Learning Complex Dexterous Manipulation with Deep Reinforcement
  Learning and Demonstrations}.
\newblock In {\em Proceedings of Robotics: Science and Systems (RSS)}, 2018.

\bibitem{RashidinejadParia2021BORL}
P.~Rashidinejad, B.~Zhu, C.~Ma, J.~Jiao, and S.~Russell.
\newblock Bridging offline reinforcement learning and imitation learning: A
  tale of pessimism.
\newblock {\em arXiv preprint arXiv:2103.12021}, 2021.

\bibitem{Rasmussen2005}
C.~E. Rasmussen and C.~K.~I. Williams.
\newblock {\em Gaussian Processes for Machine Learning (Adaptive Computation
  and Machine Learning)}.
\newblock The MIT Press, 2005.

\bibitem{reddy2020sqil}
S.~Reddy, A.~D. Dragan, and S.~Levine.
\newblock Sqil: Imitation learning via reinforcement learning with sparse
  rewards.
\newblock In {\em ICLR}, 2020.

\bibitem{ross2010efficient}
S.~Ross and J.~A. Bagnell.
\newblock Efficient reductions for imitation learning.
\newblock In {\em AISTATS}, pages 661--668, 2010.

\bibitem{ross2014reinforcement}
S.~Ross and J.~A. Bagnell.
\newblock Reinforcement and imitation learning via interactive no-regret
  learning.
\newblock {\em arXiv preprint arXiv:1406.5979}, 2014.

\bibitem{Ross2011_AISTATS}
S.~Ross, G.~J. Gordon, and J.~Bagnell.
\newblock A reduction of imitation learning and structured prediction to
  no-regret online learning.
\newblock In {\em AISTATS}, 2011.

\bibitem{seeger2008}
M.~W. Seeger, S.~M. Kakade, and D.~P. Foster.
\newblock Information consistency of nonparametric gaussian process methods.
\newblock {\em IEEE Transactions on Information Theory}, 54(5):2376--2382,
  2008.

\bibitem{SollichPeter2002LCfG}
P.~Sollich and A.~Halees.
\newblock Learning curves for gaussian process regression: Approximations and
  bounds.
\newblock {\em Neural computation}, 14(6):1393--1428, 2002.

\bibitem{SpencerJonathan2021FiIL}
J.~Spencer, S.~Choudhury, A.~Venkatraman, B.~Ziebart, and J.~A. Bagnell.
\newblock Feedback in imitation learning: The three regimes of covariate shift.
\newblock {\em arXiv preprint arXiv:2102.02872}, 2021.

\bibitem{Srinivas2010}
N.~Srinivas, A.~Krause, S.~Kakade, and M.~Seeger.
\newblock Gaussian process optimization in the bandit setting: No regret and
  experimental design.
\newblock In {\em Proceedings of the 27th International Conference on
  International Conference on Machine Learning}, ICML'10, page 1015–1022,
  2010.

\bibitem{Sun2019_model}
W.~Sun, N.~Jiang, A.~Krishnamurthy, A.~Agarwal, and J.~Langford.
\newblock Model-based rl in contextual decision processes: Pac bounds and
  exponential improvements over model-free approaches.
\newblock In {\em Proceedings of the Thirty-Second Conference on Learning
  Theory}, volume~99 of {\em Proceedings of Machine Learning Research}, pages
  2898--2933, 2019.

\bibitem{sun2019provably}
W.~Sun, A.~Vemula, B.~Boots, and D.~Bagnell.
\newblock Provably efficient imitation learning from observation alone.
\newblock In {\em International Conference on Machine Learning}, pages
  6036--6045. PMLR, 2019.

\bibitem{sun2017deeply}
W.~Sun, A.~Venkatraman, G.~J. Gordon, B.~Boots, and J.~A. Bagnell.
\newblock Deeply aggrevated: Differentiable imitation learning for sequential
  prediction.
\newblock {\em arXiv preprint arXiv:1703.01030}, 2017.

\bibitem{todorov2012mujoco}
E.~Todorov, T.~Erez, and Y.~Tassa.
\newblock Mujoco: A physics engine for model-based control.
\newblock In {\em IROS}, pages 5026--5033. IEEE, 2012.

\bibitem{TouatiAhmed2020SPOv}
A.~Touati, A.~Zhang, J.~Pineau, and P.~Vincent.
\newblock Stable policy optimization via off-policy divergence regularization.
\newblock {\em arXiv preprint arXiv:2003.04108}, 2020.

\bibitem{UeharaMasatoshi2019MWaQ}
M.~Uehara, J.~Huang, and N.~Jiang.
\newblock Minimax weight and q-function learning for off-policy evaluation.
\newblock In {\em Proceedings of the 37th International Conference on Machine
  Learning}, pages 9659--9668, 2020.

\bibitem{UeharaMasatoshi2021FSAo}
M.~Uehara, M.~Imaizumi, N.~Jiang, N.~Kallus, W.~Sun, and T.~Xie.
\newblock Finite sample analysis of minimax offline reinforcement learning:
  Completeness, fast rates and first-order efficiency.
\newblock {\em arXiv preprint arXiv:2102.02981}, 2021.

\bibitem{umlauft2018uncertainty}
J.~Umlauft, L.~P{\"o}hler, and S.~Hirche.
\newblock An uncertainty-based control lyapunov approach for control-affine
  systems modeled by gaussian process.
\newblock {\em IEEE Control Systems Letters}, 2(3):483--488, 2018.

\bibitem{Valko2013}
M.~Valko, N.~Korda, R.~Munos, I.~Flaounas, and N.~Cristianini.
\newblock Finite-time analysis of kernelised contextual bandits.
\newblock In {\em Proceedings of the Twenty-Ninth Conference on Uncertainty in
  Artificial Intelligence}, UAI'13, page 654–663, Arlington, Virginia, USA,
  2013. AUAI Press.

\bibitem{WainwrightMartinJ2019HS:A}
M.~J. Wainwright.
\newblock {\em High-Dimensional Statistics : A Non-Asymptotic Viewpoint}.
\newblock Cambridge University Press, New York, 2019.

\bibitem{Wang2020}
R.~Wang, D.~P. Foster, and S.~M. Kakade.
\newblock What are the statistical limits of offline rl with linear function
  approximation?.
\newblock {\em arXiv preprint arXiv:2010.11895}, 2020.

\bibitem{WilliamsChristopherK.I2000UaLB}
C.~K. Williams and F.~Vivarelli.
\newblock Upper and lower bounds on the learning curve for gaussian processes.
\newblock {\em Machine learning}, 40(1):77--102, 2000.

\bibitem{WuYifan2019BROR}
Y.~Wu, G.~Tucker, and O.~Nachum.
\newblock Behavior regularized offline reinforcement learning.
\newblock {\em arXiv preprint arXiv:1911.11361}, 2019.

\bibitem{XieTengyang2020QASf}
T.~Xie and N.~Jiang.
\newblock Q* approximation schemes for batch reinforcement learning: A
  theoretical comparison.
\newblock {\em UAI2020}, 2020.

\bibitem{2020Yang}
Z.~Yang, C.~Jin, Z.~Wang, M.~Wang, and M.~Jordan.
\newblock Provably efficient reinforcement learning with kernel and neural
  function approximations.
\newblock In {\em Advances in Neural Information Processing Systems},
  volume~33, pages 13903--13916, 2020.

\bibitem{YinMing2021NORL}
M.~Yin, Y.~Bai, and Y.-X. Wang.
\newblock Near-optimal offline reinforcement learning via double variance
  reduction.
\newblock {\em arXiv preprint arXiv:2102.01748}, 2021.

\bibitem{YinMing2020AEOE}
M.~Yin and Y.-X. Wang.
\newblock Asymptotically efficient off-policy evaluation for tabular
  reinforcement learning.
\newblock In {\em Proceedings of the Twenty Third International Conference on
  Artificial Intelligence and Statistics}, pages 3948--3958, 2020.

\bibitem{YuTianhe2021CCOM}
T.~Yu, A.~Kumar, R.~Rafailov, A.~Rajeswaran, S.~Levine, and C.~Finn.
\newblock Combo: Conservative offline model-based policy optimization.
\newblock {\em arXiv preprint arXiv:2102.08363}, 2021.

\bibitem{Yu2020}
T.~Yu, G.~Thomas, L.~Yu, S.~Ermon, J.~Y. Zou, S.~Levine, C.~Finn, and T.~Ma.
\newblock Mopo: Model-based offline policy optimization.
\newblock In {\em Advances in Neural Information Processing Systems},
  volume~33, pages 14129--14142, 2020.

\bibitem{ZanetteAndrea2020ELBf}
A.~Zanette.
\newblock Exponential lower bounds for batch reinforcement learning: Batch rl
  can be exponentially harder than online rl.
\newblock {\em arXiv preprint arXiv:2012.08005}, 2020.

\bibitem{zhang2019gendice}
R.~Zhang, B.~Dai, L.~Li, and D.~Schuurmans.
\newblock Gendice: Generalized offline estimation of stationary values.
\newblock In {\em International Conference on Learning Representations}, 2020.

\bibitem{ZhangTong2005LBfK}
T.~Zhang.
\newblock Learning bounds for kernel regression using effective data
  dimensionality.
\newblock {\em Neural computation}, 17(9):2077--2098, 2005.

\bibitem{ziebart2008maximum}
B.~D. Ziebart, A.~L. Maas, J.~A. Bagnell, and A.~K. Dey.
\newblock Maximum entropy inverse reinforcement learning.
\newblock In {\em AAAI}, 2008.

\end{thebibliography}
\end{document}